\documentclass{article}
\usepackage[a4paper, margin=1in]{geometry}
\usepackage{adjustbox}
\usepackage[numbers]{natbib}

\usepackage{microtype}
\usepackage{graphicx}
\usepackage{subfigure}
\usepackage{booktabs}
\usepackage{multirow}
\usepackage{tabularx}
\usepackage{subcaption}

\usepackage{amsmath}

\usepackage{mathtools}
\DeclarePairedDelimiter\ceil{\lceil}{\rceil}

\usepackage{hyperref}

\newcommand{\mycomment}[1]{}

\newcommand{\Exp}[1]{\mathbb{E}\left[{#1}\right]}
\newcommand{\pars}[1]{\left({#1}\right)}
\newcommand{\Prob}[1]{\mathbb{P}\left({#1}\right)}
\newcommand{\Ae}{A_{\epsilon, i}(\theta)}
\newcommand{\Aej}{A_{\epsilon, j}(\theta)}
\newcommand{\Sn}{S_{\epsilon, n}(\theta)}
\newcommand{\Rn}{R_{\epsilon, n}(\theta)}
\newcommand{\model}{\pi_{\theta}(y|x)}

\newcommand{\PP}[1]{\mathbb{P}\left({#1}\right)}
\newcommand{\EE}[1]{\mathbb{E}\left[{#1}\right]}
\newcommand{\cov}[1]{\mathrm{cov}\left({#1}\right)}

\newcommand{\heta}{\hat{\eta}_{\epsilon,n}(\theta)}
\usepackage{setspace}
\usepackage{algorithm}
\usepackage{algorithmic}

\usepackage{amsmath}
\usepackage{amssymb}
\usepackage{float}
\usepackage{placeins}
\usepackage{amsthm}
\usepackage[capitalize,noabbrev]{cleveref}

\theoremstyle{plain}
\newtheorem{theorem}{Theorem}[section]
\newtheorem{proposition}[theorem]{Proposition}
\newtheorem{lemma}[theorem]{Lemma}
\newtheorem{corollary}[theorem]{Corollary}
\theoremstyle{definition}

\theoremstyle{remark}

\newtheorem{assumption}{Assumption}

\title{Conformal Risk Minimization with Variance Reduction}

\author{
\begin{tabular}{ccc}
Sima Noorani & Orlando Romero & Nicolo Dal Fabbro \\
University of Pennsylvania & University of Pennsylvania & University of Pennsylvania \\
\texttt{nooranis@seas.upenn.edu} & \texttt{oromero@seas.upenn.edu} & \texttt{ndf96@seas.upenn.edu}
\end{tabular}
\\[3em] 
\begin{tabular}{cc}
Hamed Hassani & George Pappas \\
University of Pennsylvania & University of Pennsylvania \\
\texttt{hassani@seas.upenn.edu} & \texttt{pappasg@seas.upenn.edu}
\end{tabular}
}

\date{}

\begin{document}

\maketitle

\vskip 0.3in




\begin{abstract}
Conformal prediction {(CP)} is a distribution-free framework for achieving probabilistic guarantees on black-box models. {CP} is generally applied to a model post-training. Recent research efforts, on the other hand, have focused on optimizing CP efficiency \emph{during training}. We formalize this concept as the problem of \emph{conformal risk minimization} (CRM). In this direction, \emph{conformal training} (\texttt{ConfTr}) by~\citet{stutz2022learning} is a CRM technique that seeks to minimize the expected prediction set size of a model by simulating {CP} in-between training updates. 
In this paper, we provide a novel analysis for the \texttt{ConfTr} gradient estimation method, revealing a strong source of sample inefficiency 
that introduces training instability and limits its practical use. To address this challenge, we propose \emph{variance-reduced conformal training} \texttt{(VR-ConfTr)}, a CRM method that carefully incorporates a novel variance reduction technique in the gradient estimation of the \texttt{ConfTr} objective function. Through extensive experiments \footnote{We release our code at \href{https://github.com/nooranisima/conformal-risk-minimization-w-variance-reduction-code}{https://github.com/nooranisima/conformal-risk-minimization-w-variance-reduction-code}} on various benchmark datasets, we demonstrate that \texttt{VR-ConfTr} consistently achieves faster convergence and smaller prediction sets compared to baselines.
\end{abstract}

\section{Introduction}
Consider a classification task with input (features) $X\in\mathcal{X}$ and corresponding labels $Y\in \mathcal{Y}=\{1,\ldots,K\}$. Let~$\pi_\theta:\mathcal{X} \rightarrow \mathbb{R}^K$ be a parameterized predictor which, for every input $x$ and label $y$, approximates the posterior probability $\pi(y|x)=\PP{Y=y\,|\,X=x}$. Using $\pi_{\theta}$, we can estimate the label corresponding to an input $x$ as $\delta_\theta(x)=\arg\max_{y\in\mathcal{Y}}\pi_{\theta}(y|x)$. Usually, the performance of the predictor $\pi_\theta$ is assessed via the \emph{accuracy}, which is the portion of testing samples whose predicted label matches the true label. While the accuracy is a key performance metric, in safety-critical applications 
it is crucial not only to predict accurately but also to quantify the uncertainty associated with a prediction. To address this, conformal prediction (CP) \cite{vovk2005algorithmic,shafer2008tutorial, angelopoulos2023conformal}, uses the (pre-trained) model $\pi_\theta$
to construct, for an input $X$, a prediction set $C(X) \subseteq \mathcal{Y}$ that contains the true label with high probability, satisfying the desired \textit{coverage} guarantee. For example, the set $C(X)$ satisfies \textit{marginal coverage} with miscoverage rate $\alpha \in (0,1)$ if
$\PP{Y \in C(X)} \geq 1 - \alpha$. 

One way to evaluate the usefulness of prediction sets is the \emph{length efficiency} \citep{fontana2023conformal}, which represents a measure of their size. 
For instance, while it is possible to trivially guarantee any desired coverage by including the entire label space in $C(x)$, a prediction set constructed in this way is non-informative and useless. Thus, an efficient $C(x)$ is as small as possible while maintaining the desired coverage guarantee. Addressing the efficiency challenge by refining the CP set construction technique applied post-training, though effective, is \emph{inherently constrained} by the performance of the pre-trained model $\pi_\theta$. 
On the other hand, by integrating CP into the training process, recent research efforts~\cite{ dheur2024probabilistic, cherian2024large,einbinder2022traininguncertaintyawareclassifiersconformalized,stutz2022learning,bellotti2021optimized} guide the training of a model $\pi_\theta$ via CP-induced metrics, optimizing the model parameters $\theta$ to improve its \emph{inherent CP efficiency}. Here, we formalize this emerging optimization setting as the problem of \emph{conformal risk minimization} (CRM).
Among the existing CRM methods, 
the \emph{conformal training} (\texttt{ConfTr}) approach first introduced by \citet{stutz2022learning} is an intuitive technique - based on simulating CP during training to construct differentiable approximations of prediction sets - which is recently gaining momentum~\cite{cherian2024large,yan2024provably,wang2025enhancingtrustworthinessgraphneural}. 
%
%
However, despite its potential, \texttt{ConfTr} suffers from training instability and struggles to
converge \cite{stutz2022learning, liu2024cadapteradaptingdeepclassifiers, correia2024informationtheoreticperspectiveconformal}.
In this paper, we focus on the following questions:
\begin{center}
\emph{What is the source of training instability in \texttt{ConfTr}? \\How can we address this limitation?}
\end{center}
We provide answers to both questions. First, we theoretically show that 
%
%
\texttt{ConfTr} is intrinsically \emph{sample-inefficient}. 
%
%
Second, to address this limitation, we introduce \emph{variance-reduced conformal training} \texttt{VR-ConfTr}, a provably sample efficient CRM algorithm that radically improves gradient estimation and enhances training stability. 

\subsection{Contributions}
Our contributions can be summarized as follows:

\textbf{Analysis of the \texttt{ConfTr} algorithm.} Focusing on CRM for length efficiency optimization, we provide a novel analysis for the~\texttt{ConfTr}~\citep{stutz2022learning} method, which reveals a strong source of sample inefficiency in its gradient estimation technique. In particular, we show that the~\texttt{ConfTr} gradient variance \textit{is not reduced with the batch size}, and we show that this is related to the need for improved estimators of the quantile gradients.

\textbf{A ``plug-in" algorithm.} We introduce the pipeline of \emph{variance-reduced conformal training} (\texttt{VR-ConfTr}), our proposed algorithm to overcome this challenge, which (i) decouples the estimation of the population quantile and of its gradient, and (ii) leverages a ``plug-in" step to incorporate improved estimates of quantiles' gradients in the training.

\textbf{Novel variance reduction technique.} Building on a fundamental result, which characterizes the gradient of the population quantile as a conditional expectation, we propose a novel estimator for quantiles' gradients, which can be seamlessly integrated into \texttt{VR-ConfTr}. We show that, under reasonable assumptions, this integration makes \texttt{VR-ConfTr} provably \textit{sample-efficient}: unlike~\texttt{ConfTr}, our approach effectively reduces the variance of the resulting estimated gradients with the training batch size. 

\textbf{Empirical validations.} We extensively validate our method on various benchmark and real-world datasets, including MNIST, FMNIST, KMNIST, OrganAMNIST, and CIFAR10. Our results demonstrate that \texttt{VR-ConfTr} consistently and significantly improves the efficiency and stability of CRM for length efficiency optimization.

\textbf{Broad applicability.} Our approach and variance reduction technique can be integrated into any CRM method that requires quantile gradient estimation, at essentially no additional computational cost, extending its utility to a large class of CP frameworks and learning models.

\subsection{Related Works}
Conformal prediction (CP) is a distribution-free, principled framework that provides formal probabilistic guarantees for black-box models~\cite{vovk2005algorithmic,shafer2008tutorial, angelopoulos2023conformal}, with exemplar applications in computer vision \cite{angelopoulos2020uncertainty}, large language models \cite{mohri2024language,kumar2023conformal} and path {planning} \cite{lindemann2023safe}.
To improve CP efficiency, many research efforts have focused on approaches that apply CP post-training to black-box models. In particular, recent algorithmic developments address improving length efficiency through better \textbf{conformity score} design~\cite{romano2020classificationvalidadaptivecoverage, yang2024selectionaggregationconformalprediction, amoukou2023adaptiveconformalpredictionreweighting, deutschmann2024adaptive, luo2024weightedaggregationconformityscores}, or on designing better \textbf{calibration procedures}~\cite{kiyani2024lengthoptimizationconformalprediction,bai2022efficientdifferentiableconformalprediction,yang2021finite, colombo2020trainingconformalpredictors}.
These efforts do not fall under the CRM framework because they focus on learning low-dimensional hyper-parameters for pre-trained models as opposed to fully guiding the training of a $\theta$-parameterized model {$\model$}.

\textbf{Conformal risk minimization.} 
There is a growing body of work~\citep{einbinder2022traininguncertaintyawareclassifiersconformalized,cherian2024large,stutz2022learning,bellotti2021optimized, yan2024provably} integrating ideas from conformal prediction in order to directly train a model for improved CP. Among these, \texttt{ConfTr} proposed by \citet{stutz2022learning} has gained significant attention. This approach addresses length efficiency optimization by defining a loss function obtained by simulating conformal prediction during training. We will extensively describe and provide a novel analysis for this approach in the next section. Earlier work by \citet{bellotti2021optimized} considered an approach analogous to \texttt{ConfTr} in that the authors simulate conformal prediction during training. However, the algorithm provided by~\citet{bellotti2021optimized} treats the quantile-threshold as fixed and not as a function of the model parameters. It has been extensively shown by~\citet{stutz2022learning} that the approach by~\citet{bellotti2021optimized}  provides inferior performance with respect to \texttt{ConfTr}. Moreover, \citet{yan2024provably} use a similar training pipeline to \citet{stutz2022learning} in order to minimize the inefficiency of their proposed conformal predictor.  \citet{cherian2024large} train a score function, rather than a point predictor, subject to conditional coverage constraints \cite{gibbs2021adaptive}. \citet{einbinder2022traininguncertaintyawareclassifiersconformalized} utilizes conformal prediction insights in order to mitigate overconfidence in multi-class classifiers by minimizing a carefully designed loss function. 
Among these approaches, \texttt{ConfTr} is the method that has been applied the most. For example, \citet{ZHAO2025128704} investigate the use of the \texttt{ConfTr} loss function in the context of neural network pruning. Additionally, \citet{wang2025enhancingtrustworthinessgraphneural} have recently attempted to apply \texttt{ConfTr} to train Graph Neural Networks (GNNs).

\section{Problem Formulation}\label{sec:conformal}
Let us consider a (parameterized) model of logits $f_\theta: \mathcal{X} \to \mathbb{R}^K$ and let $\pi_\theta(x) = \mathrm{softmax}(f_\theta(x))$ denote the corresponding predicted probabilities. A central objective in conformal prediction is to use a given black box model $f_\theta$ to construct a \emph{set} predictor $C_\theta:\mathcal{X} \to 2^\mathcal{Y}$ in such a way that $C_\theta$ satisfies some form of probabilistic \emph{coverage} guarantee. In particular, $C_\theta(X)$ satisfies \emph{marginal} coverage if
\begin{equation}
    \PP{Y \in C_\theta(X)} \geq 1 -\alpha,
    \label{eq:marginalcoverage}
\end{equation}
for a user-specified miscoverage rate $\alpha \in (0,1)$. 

One common approach to achieve marginal coverage is via a \emph{thresholding} (\texttt{THR}) set predictor~\cite{vovk2005algorithmic}, $C_\theta(x;\tau) = \{y\in\mathcal{Y}: E_\theta(x, y) \geq \tau\}$ for some well-chosen threshold $\tau \in \mathbb{R}$ and \emph{conformity score} $E_\theta(x, y)$, which can be any heuristic notion of uncertainty regarding label $y$ upon input $x$ for the predictor $f_\theta(\cdot)$. Some choices for the conformity score include \emph{(i)} the predicted probabilities $E_\theta(x,y) = \pi_\theta(y|x) = [\pi_\theta(x)]_y$, \emph{(ii)} the logits $E_\theta(x,y) = [f_\theta(x)]_y$, and \emph{(iii)} the predicted log-probabilities $E_\theta(x,y) = \log\pi_\theta(y|x)$. 

If the distribution of  $Z = (X,Y)$ were known, we could achieve marginal coverage by setting 
$\tau = \tau(\theta)$ as the $\alpha$-quantile of the scalar random variable $E_\theta(Z)$, i.e. $\tau(\theta) = \inf\{t\in\mathbb{R}: \PP{E_\theta(Z) \leq t} \geq \alpha\}$. However since $Z$ is generally unknown in practice, we estimate $\tau(\theta)$ from samples $Z_1,\ldots,Z_n$ and use the empirical quantile $\hat{\tau}_n(\theta)$. If the data are \emph{exchangeable}, i.e., their joint distribution is invariant to permutations, then choosing $\hat{\tau}_n(\theta)$ as the empirical $\alpha$-quantile ensures that $C_\theta(x) := C_\theta(x,\hat{\tau}_n(\theta))$ satisfies the marginal coverage guarantee~\eqref{eq:marginalcoverage}. Specifically, \begin{equation}
\hat{\tau}_n(\theta) = E_{(\ceil{\alpha n})}(\theta),
\label{eq:empirical-quantile}
\end{equation}
where $E_{(1)}(\theta) \leq \ldots \leq E_{(n)}(\theta)$ denote the order statistics for $E_\theta(Z_1),\ldots,E_\theta(Z_n)$.

\subsection{Conformal Risk Minimization}
As we outlined in the introduction, recent research efforts have attempted to
combine training and CP into one, as opposed to using CP only as a post-training method. Here, we formalize this by borrowing terminology from statistical supervised learning and introducing the problem of \emph{conformal risk minimization} (CRM).
CRM can be understood as a framework for training a parameterized predictor that learns according to some CP efficiency metric. CRM can be formulated as follows:
\begin{equation*}
    \min_{\theta\in\Theta} \left\{L(\theta) := \EE{\ell(C_\theta(X), Y)}\right\}
\tag{CRM}
\label{eq:CRM}
\end{equation*}
for some \emph{conformal loss} $\ell$, where $C_\theta(x)$ is a \emph{conformalized} predictor. This problem is closely related to the \emph{conformal risk control} explored by~\citet{angelopoulos2022conformalriskcontrol}. More concretely, we will consider the threshold-based set predictor $C_\theta(x) := C_\theta(x; \tau(\theta)) = \{y\in\mathcal{Y}: E_\theta(x,y) \geq \tau(\theta)\}$.

One difficult aspect in solving~\eqref{eq:CRM} is its non-differentiability. We can remedy this by introducing a smoothed approximation of the problem. To do this, we can adopt measures similar to~\citet{stutz2022learning}. More precisely, we can first rewrite $\ell(C, y)$ as $\tilde{\ell}(\mathbf{C}, y)$, where $\mathbf{C} \in \{0,1\}^K$ is a vector of binary decision variables with $[\mathbf{C}]_k = 1_{k\in C}$. Then, assuming that $\ell(\mathbf{C},y)$ is well defined for every $\mathbf{C} \in [0,1]^K$, we can consider a smooth approximation of~\eqref{eq:CRM} by replacing $\ell(C_\theta(x;\tau(\theta)), y)$ with $\tilde{\ell}(\mathbf{C}_\theta(x;\tau(\theta)), y)$, where $[\mathbf{C}_\theta(x;\tau)]_k = 1_{E_\theta(x,y) - \tau \geq 0}$ is replaced with $[\mathbf{C}_\theta(x;\tau)]_k = \mathrm{sigmoid}\left(\frac{E_\theta(x,y) - \tau}{T}\right)$ for some temperature parameter $T>0$.

With this, we will focus on the following smoothed version of problem~\eqref{eq:CRM}:
\begin{equation}
    \min_{\theta\in\Theta} \left\{L(\theta):= h(\EE{\ell(\theta, \tau(\theta), X, Y)}) + R(\theta)\right\},
\tag{ConfTr-risk}
\label{eq:conftr-risk-2}
\end{equation}
for some monotone function $h(\cdot)$, loss $\ell(\theta,\tau,x,y)$, and regularizer $R(\theta)$. Recall also that $\tau(\theta) = \inf\{t\in\mathbb{R}: \PP{E_\theta(X,Y) \leq t} \geq \alpha\}$ for a given conformity score function $E_\theta(x,y)$.

Unlike the case of risk minimization problems, which can be understood as traditional stochastic optimization problems of the form $\theta\mapsto \EE{\ell(\theta, X, Y)}$, the problem~\eqref{eq:conftr-risk-2} does not lend itself to a trivially unbiased estimator of its gradient with variance decaying as $\mathcal{O}(1/n)$ when given $n$ i.i.d. samples from $(X,Y)$. The reason, aside from the monotone transform $h(\cdot)$, lies in the presence of $\tau(\theta)$. Unlike the traditional stochastic optimization problem, $\frac{\partial}{\partial\theta} \ell(\theta,\tau(\theta),X,Y)$ cannot be evaluated from a single realization of $(X,Y)$ since $\tau(\theta)$ and $\frac{\partial\tau}{\partial\theta}(\theta)$ are unknown due to the underlying distribution of $(X,Y)$ being unknown. However, these quantities can be estimated from data.

The quality of any such estimator directly affects the estimation $\widehat{\frac{\partial L}{\partial\theta}}(\theta)$ of $\frac{\partial L}{\partial\theta}(\theta)$ and, consequently, the performance of any gradient-based optimization algorithm used to approximately solve~\eqref{eq:conftr-risk-2}. Motivated by this, we aim to answer the following question:
\begin{center}
\emph{Can we design a gradient estimator for $L(\theta)$ that achieves arbitrarily small bias and (co)variance with sufficient samples? }
\end{center}

\section{Analysis of \texttt{ConfTr}~\cite{stutz2022learning}}
\label{sec:conftr-analysis}
\citet{stutz2022learning} introduced \textit{conformal training}~(\texttt{ConfTr}), which we categorize as a CRM approach for length efficiency optimization. In particular, \texttt{ConfTr} focuses on reducing \emph{inefficiency} of calibrated classifiers, quantified by the \emph{target size} of predicted sets. This can be understood as the problem in~\eqref{eq:CRM} with $\ell(C, y) = \max(0, |C| - \kappa)$ for some \emph{target size} $\kappa$ (intended to discourage no predictions at all) and with a $\log$ transform $h$ for numerical stability reasons. In this regard, it is worth noting that the earlier work of~\citet{sadinle2019least} was the first to study the closely related problem of \emph{least ambiguous} set-valued classifiers, which corresponds to $l(C,y) = |C|$. 

The underlying assumption, just as in any supervised learning task, is that the marginal distribution of $(X,Y)$ is unknown but that instead we can collect some i.i.d. training data $\mathcal{D} = \{(X_1,Y_1), \ldots, (X_n,Y_n)\}$. With this, an issue presents itself in that, unlike a typical loss function, we cannot evaluate $\frac{\partial}{\partial\theta}[\ell(\theta,\tau(\theta),X_i, Y_i)]$ from realizations of $X_i, Y_i$ alone, because $\tau(\theta) = \mathrm{quantile}_\alpha(E_\theta(X,Y))$ and $\frac{\partial\tau}{\partial\theta}(\theta)$ are functions of the \emph{distribution} of $(X,Y)$ and not a mere transformation. To resolve this issue,~\citet{stutz2022learning} propose their \texttt{ConfTr} algorithm, which randomly splits a given batch $B$ into two parts, which they refer to as \emph{calibration} batch $B_\mathrm{cal}$ and \emph{prediction} batch $B_\mathrm{pred}$. With this, the authors advocate for employing any smooth (differentiable) quantile estimator algorithm for $\tau(\theta)$ using the calibration batch. Then, they propose using this estimator to compute a sampled approximation of~\eqref{eq:conftr-risk-2}, replacing expectations by sample means constructed using the prediction batch. Let $\hat{L}(\theta)$ denote the end-to-end empirical approximation of $L(\theta)$. Once $\hat{L}(\theta)$ is constructed, the authors advocate for a (naive) risk minimization procedure where $\frac{\partial\hat{L}}{\partial\theta}(\theta)$ is computed and passed to an optimizer of choice.

\subsection{Sample-Inefficiency of \texttt{ConfTr}}
\label{sec:varAnaysisConfTr}

In this subsection, we will show that the approach used in~\cite{stutz2022learning} leads to an asymptotically unbiased gradient estimator, but its variance does not vanish. 






We will use $E_\theta(x,y)$ and $E(\theta,x,y)$ interchangeably and often write $(x,y)$ as~$z$.
To distinguish between partial and total differentiation, $\frac{\partial}{\partial\theta} \ell(\theta,\tau(\theta),x,y)$ denotes the Jacobian of $\theta \mapsto \ell(\theta,\tau(\theta),x,y)$ evaluated at $\theta$, whereas $\frac{\partial \ell}{\partial\theta}(\theta,\tau(\theta),x,y)$ denotes the Jacobian of $\ell(\cdot,\tau(\theta),x,y)$ evaluated at $\theta$. In particular, we have $\frac{\partial \ell}{\partial\theta}(\theta,\tau(\theta),x,y) = \frac{\partial}{\partial\theta'} \ell(\theta',\tau(\theta),x,y) \big\vert_{\theta'=\theta}$. Recall that $Z_i = (X_i,Y_i)$ are i.i.d. copies of $Z = (X,Y)$ and $E_{(1)}(\theta) \leq \ldots E_{(n)}(\theta)$ denote the order statistics for $E_\theta(Z_1),\ldots, E_\theta(Z_n)$. We make the following regularity assumptions:



\begin{assumption}
$E(\theta,x,y)$ is continuously differentiable and $M$-Lipschitz in $\theta$.
\label{ass:M-Lipschitz}
\end{assumption}

\begin{assumption}
$E_\theta(Z)$ and $\frac{\partial E}{\partial\theta}(\theta,Z)$ have a continuous joint probability density function.
\label{ass:continuous-joint-pdf}
\end{assumption}

\citet{stutz2022learning} consider \emph{smooth} quantile estimators based on smooth sorting. However, in the next proposition we argue that the empirical quantile does not need to be smoothed.

\begin{proposition}
$E_{(1)}(\theta), \ldots, E_{(n)}(\theta)$ are almost surely (a.s.) everywhere differentiable in $\theta$. In particular, the empirical quantile $\hat{\tau}_n(\theta) = E_{(\ceil{\alpha n})}(\theta)$ is a.s. everywhere differentiable.
\label{prop:a.s.-diff}
\end{proposition}

The proof of this result relies on noting that while the sort function is not differentiable everywhere, it is nonetheless differentiable almost everywhere. More precisely, the sort function is piecewise linear, with the non-differentiable points corresponding to ties in variables that are to be sorted. Due to the assumed continuous distribution of the scores $E_\theta(Z_1),\ldots,E_\theta(Z_n)$, it follows that there are almost surely (a.s.) no ties, and therefore the order statistics are indeed differentiable (a.s.). Subsequently, the empirical quantile $\hat{\tau}_n(\theta)$ is a.s. differentiable (everywhere in $\theta$). Further, the smooth quantile $\hat{\tau}_n(\theta;\varepsilon)$ based on the smooth sorting method from~\citet{blondel2020fast} (which~\citet{stutz2022learning} re-implement in their own codebase) actually satisfies $\hat{\tau}_n(\theta;\varepsilon) = \hat{\tau}_n(\theta) = E_{(\ceil{\alpha n})}(\theta)$ as long as $\varepsilon > 0$ is small enough, to the order of $\min_{i\neq j} |E_\theta(Z_i) - E_\theta(Z_j)|$, as can be seen from Lemma~3 in~\cite{blondel2020fast}. 

We can now analyze the quality of the estimate $\widehat{\frac{\partial L}{\partial\theta}}(\theta)$ of $\frac{\partial L}{\partial\theta}(\theta)$ using~\texttt{ConfTr}, assuming for simplicity that the ``smooth quantile'' exactly coincides with the empirical sample quantile. More precisely,~\citet{stutz2022learning} propose to use a batch of $2n$ i.i.d. samples, which get split in two parts of equal size $n$: one part for ``calibration'' and the other for ``prediction.'' The calibration samples, as the name suggest, are used to estimate $\tau(\theta)$ as $\hat{\tau}_n(\theta)$ and nothing else, whereas the prediction samples are used to finish the estimation of $\frac{\partial L}{\partial\theta}(\theta)$ by merely replacing $\tau(\theta)$ with $\hat{\tau}_n(\theta)$ and replacing the expectations with sample means. This is formalized in Algorithm~\ref{alg:vr_conftr} with $\hat{\tau}(\theta) := E_{(\ceil{\alpha n})}(\theta)$ and $\widehat{\frac{\partial\tau}{\partial\theta}}(\theta) := \frac{\partial\hat{\tau}}{\partial\theta}(\theta) = \frac{\partial}{\partial\theta} E_{(\ceil{\alpha n})}(\theta)$.

To proceed with our analysis, let us first characterize the asymptotic behavior of $\frac{\partial\hat{\tau}_n}{\partial\theta}(\theta)$.


\begin{proposition}
Let $\hat{\tau}_n(\theta) = E_{(\ceil{\alpha n})}(\theta)$. Then,
\begin{equation}
    \frac{\partial \hat{\tau}_n}{\partial\theta}(\theta) \overset{\mathrm{dist}}{\longrightarrow} \frac{\partial E}{\partial\theta}(\theta,Z) \,\big|_{E_\theta(Z)=\tau(\theta)}
\end{equation}
as $n\to\infty$.
\label{thm:conftr-weakconvergence}
\end{proposition}

Since, $Z = (X,Y)$ will typically be a high-dimensional random vector (due to $X$), $\frac{\partial E}{\partial\theta}(\theta,Z) \,\big|_{E_\theta(Z)=\tau(\theta)}$ will not be a constant. In particular, $\frac{\partial \hat{\tau}_n}{\partial\theta}(\theta)$ will fail to be a consistent estimator of $\frac{\partial\tau}{\partial\theta}(\theta)$. We can more specifically characterize the behavior of $\frac{\partial \hat{\tau}_n}{\partial\theta}(\theta)$ by first noting the following helpful result.



\begin{proposition}
\label{th:fundLemma}
Suppose that $X$ is absolutely continuous and $E_\theta(x,y)$ is continuously differentiable in $\theta$ and $x$. Then, for every $\theta \in \Theta$,
\begin{equation}\label{eq:fundEq}
    \frac{\partial\tau}{\partial\theta}(\theta) = \mathbb{E}\Bigg[\frac{\partial E}{\partial\theta}(\theta,X,Y) \,\Big\vert\, E_\theta(X,Y) = \tau(\theta) \Bigg].
\end{equation}
\end{proposition}

In Appendix \ref{appendix:proofs}, we provide a rigorous proof for the above proposition, which was carried out independently from that of the equivalent result by~\citet{hong2009estimating} (i.e., Theorem 2). Using this proposition and leveraging Assumption~\ref{ass:M-Lipschitz}, we can establish that $\frac{\partial\hat{\tau}_n}{\partial\theta}(\theta)$ is asymptotically unbiased but its covariance matrix does not vanish as $n\to\infty$. The formal proof of this can be found in Corollary~\ref{app:corVar} in the appendix.





{


From Proposition~\ref{prop:a.s.-diff}, we have that, almost surely, $\theta \mapsto \ell(\theta,\hat{\tau}_n(\theta),Z)$ is differentiable. Therefore, it follows from the chain rule that
\begin{align}\label{eq:eq8}
    \frac{\partial}{\partial \theta} [\ell(\theta,\hat{\tau}_n(\theta), Z)] &= 
    \frac{\partial \ell}{\partial \theta}(\theta, \hat{\tau}_n(\theta), x, y) \\
    &\quad + \frac{\partial \ell}{\partial\tau}(\theta, \hat{\tau}_n(\theta), Z) 
    \frac{\partial\hat{\tau}_n}{\partial \theta}(\theta) \nonumber
\end{align}
holds almost surely. By inspection, we can see that the covariance of $\frac{\partial\hat{\tau}_n}{\partial \theta}(\theta)$ will bottleneck the covariance of~\eqref{eq:eq8}, and thus also the covariance of $\widehat{\frac{\partial L}{\partial\theta}}(\theta)$. We can formalize this in the following theorem.



\begin{theorem}
Suppose that $\ell(\cdot,\cdot,,z)$ is continuously differentiable. For the~\texttt{ConfTr} method, the estimator $\widehat{\frac{\partial L}{\partial\theta}}(\theta)$ converges weakly to a random vector that is not constant. If, in addition, $\ell(\cdot,\cdot,z)$ is $M$-Lipschitz and bounded by~$M$, then $\widehat{\frac{\partial L}{\partial\theta}}(\theta)$ is asymptotically unbiased but its covariance matrix does not vanish as $n\to\infty$.

\label{thm:bias-var-conftr}
\end{theorem}

The first part of the theorem follows from the continuous mapping theorem and by noting that $\hat{\tau}_n(\theta) \overset{\mathrm{a.s.}}{\longrightarrow} \tau(\theta)$. The second part follows by noting that all the terms in line 7 of Algorithm~\ref{alg:vr_conftr} are uniformly integrable.

The takeaways of the analysis in this section are that $\frac{\partial\hat{\tau}_n}{\partial\theta}(\theta)$ is a poor estimator of $\frac{\partial\tau}{\partial\theta}(\theta)$, \textit{even though $\hat{\tau}(\theta)$ is an effective estimator of $\tau(\theta)$}. This in turn causes \textit{the overall gradient's variance to be} $\Omega(1)$, because, as we show in Theorem~\ref{thm:bias-var-conftr}, the bias and covariance of $\frac{\partial\hat{\tau}_n}{\partial\theta}(\theta)$ get inherited by $\widehat{\frac{\partial L}{\partial\theta}}(\theta)$.
This motivates our proposed solution addressing the gradient estimation issue of \texttt{ConfTr}, which we present next.

}

\begin{figure*}[ht]
\centering
\setlength{\textfloatsep}{0pt}     
\includegraphics[width=1\linewidth]{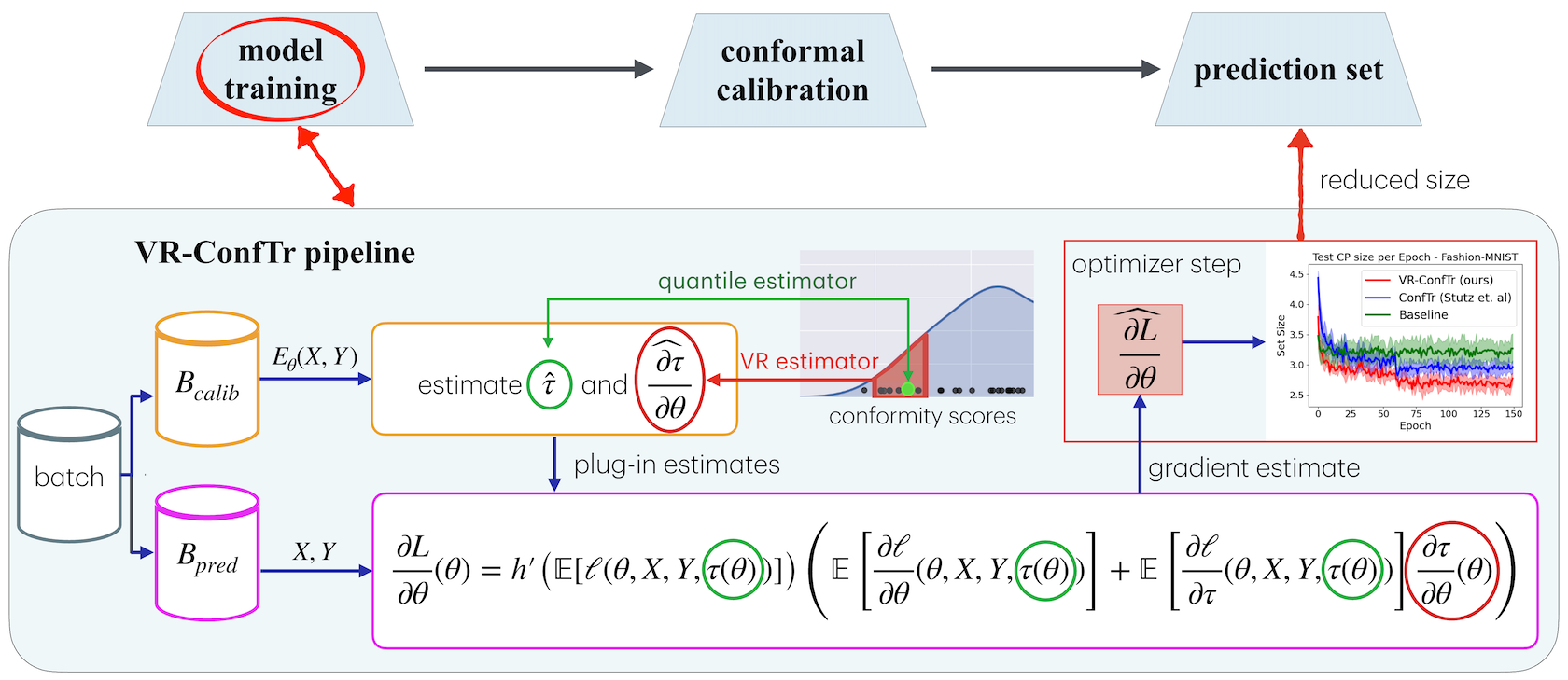} 
\caption{In this figure, we illustrate the \texttt{VR-ConfTr} pipeline and position it with respect to a typical CP procedure.}
\label{fig:conftr-pipeline}
\end{figure*}

\section{Variance-Reduced Conformal Training}
\label{sec:vr-conftr}

In order to surpass the shortcoming of \texttt{ConfTr} described in the previous section, let us first note that the gradient of the conformal risk~\eqref{eq:conftr-risk-2} can be written as
\begin{equation}\label{eq:gradConfTr}
    \frac{\partial L}{\partial\theta}(\theta) = h'\left(\mathbb{E}\big[\ell(\theta,\tau(\theta),Z)\big]\right)
    \times \Bigg(
    \mathbb{E}\left[\frac{\partial\ell}{\partial\theta}(\theta,\tau(\theta),Z)\right] + \mathbb{E}\left[\frac{\partial\ell}{\partial\tau}(\theta,\tau(\theta),Z)\right]
    \frac{\partial\tau}{\partial\theta}(\theta)\Bigg),
\end{equation}
where $h'$ denotes the derivative of $h$, $Z = (X,Y)$. Note that we dropped the regularizer for simplicity. 

\textbf{Discussion:} Inspecting equation~\eqref{eq:gradConfTr}, we note that, in order to estimate the gradient $\frac{\partial L}{\partial\theta}(\theta)$, the estimator $\widehat{\frac{\partial\tau}{\partial\theta}}(\theta)$ of the population quantile gradient $\frac{\partial\tau}{\partial\theta}(\theta)$ \textit{does not necessarily need to be the gradient of the sample quantile} $\hat{\tau}(\theta)$, which was the strategy adopted by the \texttt{ConfTr} method. As we show in section~\ref{sec:varAnaysisConfTr}, this naive choice, which sets $\widehat{\frac{\partial\tau}{\partial\theta}}(\theta) = \frac{\partial\hat{\tau}}{\partial\theta}(\theta)$, results in a \textit{highly sample-inefficient gradient estimator}. With this in mind, we will now leverage the structure of the population quantile gradient established in~\eqref{eq:fundEq} to design a novel estimator for $\frac{\partial\tau}{\partial\theta}(\theta)$. 
Then, in section~\ref{sec:algorithm}, we illustrate the pipeline of our proposed \texttt{VR-ConfTr} algorithm, showing how we can integrate the new estimator in computing the estimate $\widehat{\frac{\partial L}{\partial\theta}}$ of $\frac{\partial L}{\partial\theta}(\theta)$. 

\subsection{Quantile Gradient Estimation}\label{sec:estimation}
We will now use the relationship established in Proposition~\ref{th:fundLemma} to design a novel estimator of the quantile gradient. The idea is as follows: let us denote
\begin{equation}
\begin{aligned}
\eta_\varepsilon(\theta) &:= \EE{\frac{\partial E}{\partial \theta}(\theta, X,Y)\,\big\vert\, A_\varepsilon(\theta)}, \\ \Sigma_\varepsilon(\theta) &:= \cov{\frac{\partial E}{\partial \theta}(\theta, X,Y)\,\big\vert\, A_\varepsilon(\theta)},
\end{aligned}\label{eq:eta_Sigma_eps}
\end{equation}
for $\varepsilon>0$, 
where $A_\varepsilon(\theta) := \{|E_\theta(X,Y) - \tau(\theta)| \leq \varepsilon\}$. Note that the term $\eta_\varepsilon(\theta)$ in~\eqref{eq:eta_Sigma_eps} is approximately equal to the population quantile gradient $\frac{\partial\tau}{\partial\theta}(\theta)$ if $\varepsilon \approx 0$. 
Based on this, assuming that we have a good estimate $\hat\tau(\theta)$ of the population quantile $\tau(\theta)$ available, we can estimate $\eta(\theta) = \frac{\partial\tau}{\partial\theta}(\theta)$, 
using $\eta_\varepsilon(\theta)$ via the following $\varepsilon$-threshold strategy:
\begin{equation}
    \hat{\eta}(\theta) := \frac{1}{\sum_{i=1}^n 1_{\hat{A}_{\varepsilon, i}(\theta)}}\sum_{i=1}^n 1_{\hat{A}_{\varepsilon, i}(\theta)}\frac{\partial E}{\partial\theta}(\theta,X_i, Y_i),
    \label{eq:eta_hat_thresh}
\end{equation}
using i.i.d. samples $(X_1,Y_1),\ldots,(X_n,Y_n)$, where $\hat{A}_{\varepsilon,i}(\theta) = \{|E_\theta(X_i,Y_i) - \hat{\tau}(\theta)| \leq \varepsilon\}$. In other words, given a batch of $n$ samples, the estimator $\hat{\eta}(\theta)$ in~\eqref{eq:eta_hat_thresh} is computed as the average of $m = \sum_{i=1}^n 1_{\hat{A}_{\varepsilon,i}(\theta)}$ conformity scores' gradients, corresponding to the samples whose conformity scores are $\varepsilon$-close to the estimated population quantile $\hat\tau(\theta)$.

\textbf{Tuning $\varepsilon$ with $m$-ranking.} In practice, a good value of $\varepsilon$ may depend on batch and parameter $\theta$, and it could significantly change for each iteration of the optimization process. Furthermore, fine-tuning a time-varying value of a scalar quantity $\varepsilon>0$ can be challenging, and we therefore need some other heuristic way to tune $\varepsilon$ adaptively
. 
One intuitive way to do this, that we adopt in our experiments, is to fix an integer $m$, sort the samples based on the distances $\{\big|E_\theta(X_i,Y_i) - \hat{\tau}(\theta)\big|\}_{i=1,\ldots,n}$, and then average the ``top'' $m$ samples (smallest distances) to estimate the quantile gradient. This strategy is equivalent to selecting $\varepsilon = \inf\left\{\varepsilon' > 0: \sum_{i=1}^n 1_{\hat{A}_{\varepsilon',i}(\theta)} \geq m\right\}$, and it is very practical since it only requires to fine-tune an integer $m$.
\subsection{Proposed Algorithm: \texttt{VR-confTr}}\label{sec:algorithm}

Suppose that the variance-reduced estimator for $\frac{\partial\tau}{\partial\theta}(\theta)$ has been already designed. Then, the new estimate for $\tau(\theta)$ and $\frac{\partial\tau}{\partial\theta}(\theta)$ \textit{can be plugged into expression}~\eqref{eq:gradConfTr} for the gradient of the conformal training risk function, before approximating the expectation by sample means, leading to the \emph{plug-in} estimator for $\frac{\partial L}{\partial\theta}(\theta)$. Naturally, the plug-in gradient estimator is then passed through an optimizer in order to approximately solve~\eqref{eq:CRM}. Our proposed pipeline, which we call \emph{variance-reduced conformal training} (\texttt{VR-ConfTr}) algorithm, paired with our novel estimator in~\eqref{eq:eta_hat_thresh}, constitutes our main contribution and proposed solution to improve the sample inefficiency of~\texttt{ConfTr}. The critical step of constructing the plug-in estimator is summarized in Algorithm~\ref{alg:vr_conftr} and the entire pipeline is illustrated in Figure~\ref{fig:conftr-pipeline}. Similar to Theorem~\ref{thm:bias-var-conftr}, we can establish that the bounds on the bias and covariance of $\widehat{\frac{\partial\tau}{\partial\theta}}(\theta)$ get inherited by $\widehat{\frac{\partial L}{\partial\theta}}(\theta)$, assuming that $\hat{\tau}(\theta) \overset{\mathrm{a.s.}}{\longrightarrow} \tau(\theta)$.

\begin{algorithm}[t]
\caption{Gradient estimator of \texttt{VR-ConfTr}}\label{alg:vr_conftr}
\begin{algorithmic}[1]
\REQUIRE \,\\
batch $B = \{(X_1,Y_1),\ldots, (X_{2n}, Y_{2n})\}$ of i.i.d. samples from $(X,Y)$, \\
score function $E(\theta,x,y):\Theta\times\mathcal{X}\times\mathcal{Y}\to\mathbb{R}$,\\
conformal loss $\ell(\theta,x,y,\tau):\Theta\times\mathcal{X}\times\mathcal{Y}\times\mathbb{R} \to\mathbb{R}$,\\
monotone transformation $\mathcal{F}:\mathbb{R}\to\mathbb{R}$,\\
estimator $\hat{\tau}(\cdot)$ for $\tau(\theta) = Q_\alpha(E_\theta(X,Y))$,\\
estimator $\widehat{\frac{\partial{\tau}}{\partial \theta}}(\cdot)$ for $\frac{\partial \tau}{\partial \theta}(\theta)$.\\
\ENSURE output an estimate $\widehat{\frac{\partial L}{\partial\theta}}$ of the gradient $\frac{\partial L}{\partial \theta}(\theta)$ of the conformal training risk~\eqref{eq:conftr-risk-2}   
\STATE partition $B$ into $\{B_{\mathrm{cal}}, B_\mathrm{pred}\}$, with $|B_{\mathrm{cal}}| = |B_{\mathrm{pred}}| = n$.
\STATE $\hat{\tau} \gets \hat{\tau}(B_\mathrm{cal})$ $\hfill$// estimate $\tau(\theta)$ using { $B_{\mathrm{cal}}$}
\STATE $\widehat{\frac{\partial\tau}{\partial\theta}} \gets\widehat{\frac{\partial\tau}{\partial\theta}}(B_\mathrm{cal})$ $\hfill$// estimate $\frac{\partial\tau}{\partial\theta}(\theta)$ using {$B_{\mathrm{cal}}$}\vspace{0.1cm}
\STATE $\hat{\ell} \gets \frac{1}{|B_\mathrm{pred}|}\sum_{(x,y)\in B_\mathrm{pred}}\ell(\theta, x, y, \hat{\tau})$
\STATE $\widehat{\frac{\partial\ell}{\partial\theta}} \gets \frac{1}{|B_\mathrm{pred}|}\sum_{(x,y)\in B_\mathrm{pred}}\frac{\partial\ell}{\partial\theta}(\theta, x, y, \hat{\tau})$
\STATE $\widehat{\frac{\partial\ell}{\partial\tau}} \gets \frac{1}{|B_\mathrm{pred}|}\sum_{(x,y)\in B_\mathrm{pred}}\frac{\partial\ell}{\partial\tau}(\theta, x, y, \hat{\tau})$
\STATE $\widehat{\frac{\partial L}{\partial\theta}} \gets {{h}'(\hat{\ell})}\left(\widehat{\frac{\partial\ell}{\partial\theta}} + \widehat{\frac{\partial\ell}{\partial\tau}}\widehat{\frac{\partial\tau}{\partial\theta}} \right) + \frac{\partial R}{\partial\theta}(\theta)$ $\hfill$// ``\textbf{plug-in}'' 
\end{algorithmic}
\end{algorithm}

\begin{table}[htp]
\centering
\begin{tabular}{@{}lllccc@{}}
\toprule
\textbf{Dataset} & \textbf{Model} & \textbf{Algorithm} & \textbf{Accuracy (Avg ± Std)} & \textbf{Size (Avg ± Std) (\%)} \\ 
\midrule
\multirow{3}{*}{MNIST} & \multirow{3}{*}{Linear} & Baseline & $0.887 \pm 0.004$ & $4.122 \pm 0.127$ (+12\%) \\ 
& & ConfTr~\cite{stutz2022learning}  & $0.842 \pm 0.141$ & $3.990 \pm 0.730$ (+8\%) \\ 
& & VR-ConfTr \textbf{(ours)} &  $0.886 \pm 0.071$ & $\mathbf{3.688 \pm 0.350}$ \\ 
\cmidrule{1-5}
\multirow{3}{*}{Fashion-MNIST} & \multirow{3}{*}{MLP} & Baseline & $0.845 \pm 0.002$ & $3.218 \pm 0.048$ (+15\%) \\ 
& & ConfTr~\cite{stutz2022learning}   & $0.799 \pm 0.065$ & $3.048 \pm 0.201$ (+9\%) \\ 
& & VR-ConfTr \textbf{(ours)}  & $0.839 \pm 0.043$ & $\mathbf{2.795 \pm 0.154}$ \\ 
\cmidrule{1-5}
\multirow{3}{*}{Kuzushiji-MNIST} & \multirow{3}{*}{MLP} & Baseline & $0.872 \pm 0.046$ & $4.982 \pm 0.530$ (+6\%) \\ 
& & ConfTr~\cite{stutz2022learning}   & $0.783 \pm 0.125$ & $4.762 \pm 0.226$ (+2\%) \\ 
& & VR-ConfTr \textbf{(ours)}  & $0.835 \pm 0.098$ & $\mathbf{4.657 \pm 0.680}$ \\ 
\cmidrule{1-5}
\multirow{3}{*}{OrganA-MNIST} &  \multirow{3}{*}{ResNet-18} & Baseline & $0.552 \pm 0.017$ & $4.823 \pm 0.748$ (+2\%) \\ 
& & ConfTr~\cite{stutz2022learning}   & $0.526 \pm 0.047$ & $6.362 \pm 0.857$ (+33\%) \\ 
& & VR-ConfTr \textbf{(ours)}  &  $0.547 \pm 0.021$ & $\mathbf{4.776 \pm 1.178}$ \\ 
\cmidrule{1-5}
\multirow{3}{*}{CIFAR-10} & \multirow{3}{*}{ResNet-20} & Baseline & $0.743 \pm 0.003$ & $2.881 \pm 0.038$ (+12\%) \\ 
& & ConfTr~\cite{stutz2022learning}   & $0.733 \pm 0.051$ & $2.806 \pm 0.389$ (+10\%) \\ 
& & VR-ConfTr \textbf{(ours)} & $0.742 \pm 0.023$ & $\mathbf{2.508 \pm 0.039}$ \\ 
\bottomrule
\end{tabular}
\caption{Summary of evaluation results. For \texttt{VR-ConfTr}, we show in percentage the average set size (\textbf{Size (Avg ± Std) (\%)}) improvement against \texttt{ConfTr} by~\cite{stutz2022learning}. The third column presents the average accuracy and its standard deviation (\textbf{Accuracy (Avg ± Std)}). The confidence level $\alpha$ used for conformal prediction is fixed at 0.01 for all datasets except CIFAR-10, where it is set to 0.1.}
\label{tab:results_summary}
\end{table}

\subsection{Sample-Efficiency of \texttt{VR-ConfTr}}

We now provide an analysis for the $\varepsilon$-estimator $\hat{\eta}(\theta)$~\eqref{eq:eta_hat_thresh} of the population quantile gradient, establishing its bias-variance trade-off in the following theorem.

\begin{theorem}[]\label{th:VR_theorem} Fix $\theta$ and $\varepsilon>0$. Let $\hat{\eta}(\theta)$ be the gradient estimator defined in~\eqref{eq:eta_hat_thresh}. Then, the bias and variance of the estimator can be characterized as follows:
\begin{alignat*}{2}
(i)&\hspace{0.7cm}\EE{\hat{\eta}(\theta)}&&= (1-\left[q_{\varepsilon}(\theta)\right]^n)\eta_{\varepsilon}(\theta)\\
(ii)&\hspace{0.5cm}\cov{\hat{\eta}(\theta)}&&\preceq \frac{2\Sigma_{\varepsilon}(\theta)}{p_{\varepsilon}(\theta) n} + \left[q_{\varepsilon}(\theta)\right]^n\eta_{\varepsilon}(\theta)\eta_{\varepsilon}^\mathsf{T}(\theta)
,
\end{alignat*}    
where $p_{\varepsilon}(\theta) = \mathbb{P}(A_{\varepsilon, i}(\theta))$ and $q_{\varepsilon}(\theta) = 1-p_{\varepsilon}(\theta)$.
\end{theorem}

The main takeaway of result $(i)$ is that $\hat{\eta}(\theta)$ is an \emph{asymptotically unbiased} estimator of $\eta_{\varepsilon}(\theta)$, but not $\eta(\theta)$. However, by definition we also have $\eta_\varepsilon(\theta) \approx \eta(\theta)$ for $\varepsilon \approx 0$. The second result $(ii)$, instead, shows that \emph{variance reduction} is obtained by the proposed estimator, when compared to the naive estimator $\frac{\partial\hat{\tau}}{\partial\theta}(\theta)$. For large $n$, the variance reduction is proportional to $p_\varepsilon(\theta)n$, which is equal to the (expected) proportion of samples that are used in the estimator. More precisely, the variance of the estimator is $\mathcal{O}\left(\frac{1}{p_\varepsilon(\theta)n}\right)$ as $\varepsilon\to 0$ or $n\to\infty$. A key takeaway of $(i)$ and $(ii)$ is the explicit characterization of the \emph{bias-variance trade-off} as a function of the threshold $\varepsilon>0$ and of the batch size $n$: for a given batch size $n$, a larger $\varepsilon$ increases the expected amount of samples used by the estimator, reducing its variance. However, larger $\varepsilon$ also increases the bias of the estimator
towards the unconditional expectation $\EE{\frac{\partial E}{\partial\theta}(\theta,X,Y)}$. 


\section{Experiments}
We evaluate \texttt{VR-ConfTr} against (i) a baseline model trained with cross-entropy loss (\texttt{Baseline}), and (ii) \texttt{ConfTr}. We perform experiments across benchmark datasets - MNIST~\cite{deng2012mnist}, Fashion-MNIST~\cite{FMNIST}, Kuzushiji-MNIST~\cite{clanuwat2018deep}, CIFAR10 ~\cite{krizhevsky2009learning}, and a healthcare dataset comprising abdominal computed tomography scans, OrganAMNIST~\cite{medmnistv2}. One of the main performance metrics that we consider is the \emph{length-efficiency} of the conformal prediction sets produced by applying a standard CP procedure to the trained model. Other relevant metrics are the convergence speed of the algorithm, as well as the final accuracy of the model. To tune the \texttt{VR-ConfTr} $\varepsilon$-estimator for the quantile gradient 
in~\eqref{eq:eta_hat_thresh}, we employ the $m$-ranking method ~\ref{sec:estimation}, and investigate multiple choices for $m$, detailed in Appendix~\ref{appendix:m_tuning_experiments}. We provide extensive details about the training settings, the adopted model architectures and hyper-parameters in Appendix~\ref{appendix:experiments}.
Next, we summarize results from evaluating the trained model, reporting average \textit{accuracy} and CP \textit{length efficiency} over multiple runs. In Section~\ref{sec:speed} we compare \texttt{Vr-ConfTr} and \texttt{ConfTr} illustrating curves of relevant evaluation metrics, highlighting the improved training performance and variance reduction.

\subsection{Summary of Evaluation Results}
Table \ref{tab:results_summary} presents the CP set size resulting from CP procedure applied post-training, and the accuracy of the trained model for each dataset.

The metrics in Table~\ref{tab:results_summary} are averaged over 5-10 training trials,
with details on trial variations in Appendix~\ref{appendix:experiments}.
%
For a fair comparison, we used the same model architecture for all methods (\texttt{ConfTr}, \texttt{VR-ConfTr}, \texttt{Baseline}) and identical hyper-parameters for \texttt{ConfTr} and \texttt{VR-ConfTr}. 
For post-training CP, we use the standard \texttt{THR} method with the corresponding $\alpha$. Average set sizes are reported over 10 calibration-test splits. The main takeaway from Table~\ref{tab:results_summary} is that \texttt{VR-ConfTr} improves over all considered metrics compared to \texttt{ConfTr}.
\texttt{VR-ConfTr} consistently achieves smaller prediction set sizes than \texttt{ConfTr} and \texttt{Baseline}. Important to note that our focus is not to optimize \texttt{ConfTr} but to demonstrate that \texttt{VR-ConfTr} improves performance and stability regardless of \texttt{ConfTr}'s performance or hyperparameters. As noted by~\citet{stutz2022learning}, \texttt{Baseline} sometimes achieves slightly higher accuracy than \texttt{ConfTr} and \texttt{VR-ConfTr}, though \texttt{VR-ConfTr} consistently outperforms \texttt{ConfTr}. However, the focus of conformal training is to improve CP efficiency by reducing prediction set sizes while maintaining comparable accuracy to non-CRM methods, not surpassing \texttt{Baseline} in accuracy.
\subsection{On the Training Performance of \texttt{VR-ConfTr}}\label{sec:speed}
Here, we focus on the training performance of \texttt{VR-ConfTr}, with special attention to the speed in minimizing the conformal training loss described in section~\ref{sec:conformal}, and in minimizing the CP set sizes on test data. The results illustrate the evolution of the different metrics across epochs, validate the beneficial effect of the variance reduction technique and the superior performance of \texttt{VR-ConfTr} when compared to the competing \texttt{ConfTr} by~\citet{stutz2022learning}. 
\begin{figure}[htbp]
\centering
    \includegraphics[width=0.24\textwidth]{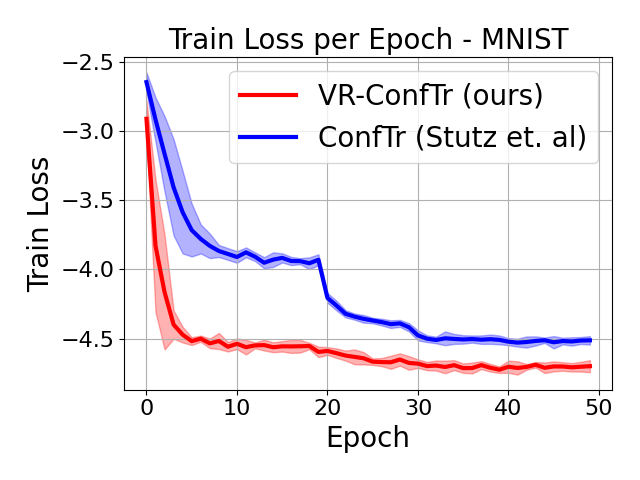}
    \includegraphics[width=0.24\textwidth]{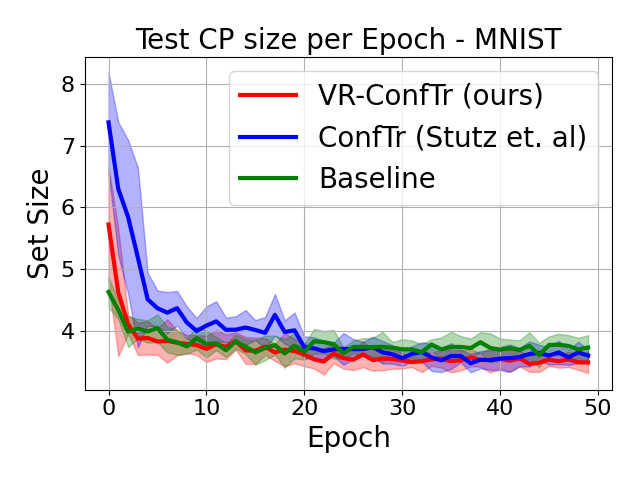}
    \includegraphics[width=0.24\textwidth]{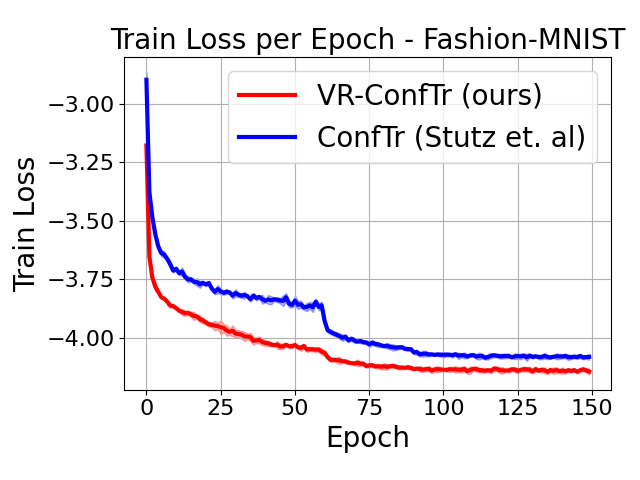}
    \includegraphics[width=0.24\textwidth]{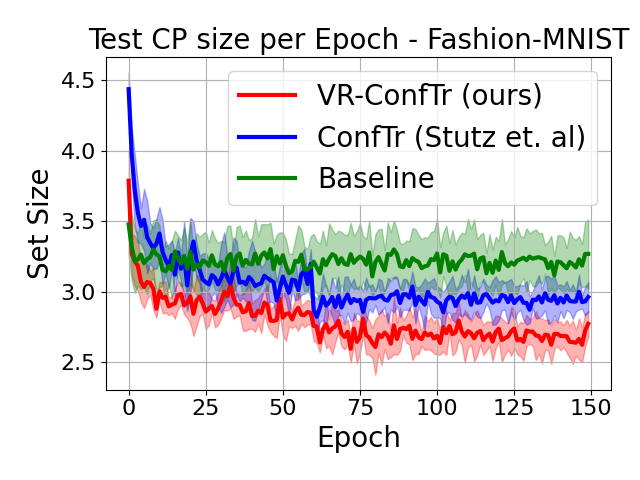}
    \includegraphics[width=0.24\textwidth]{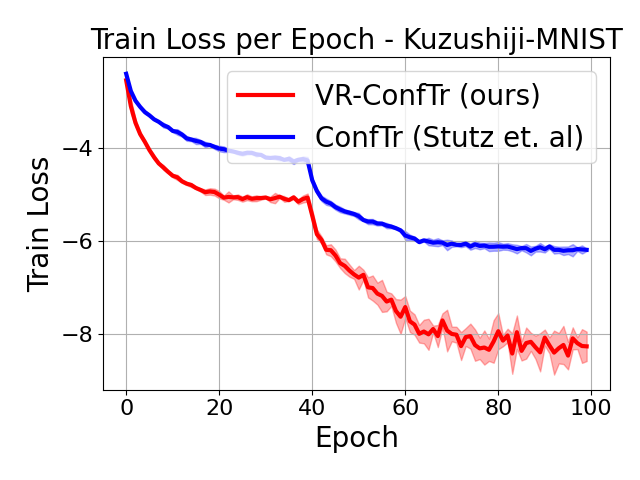}
    \includegraphics[width=0.24\textwidth]{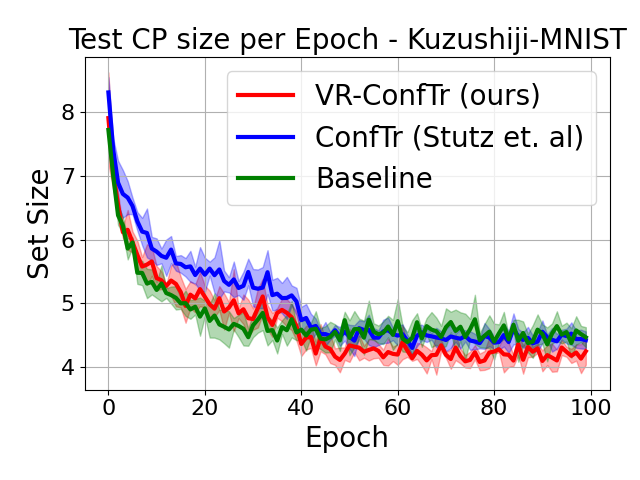}
    \includegraphics[width=0.24\textwidth]{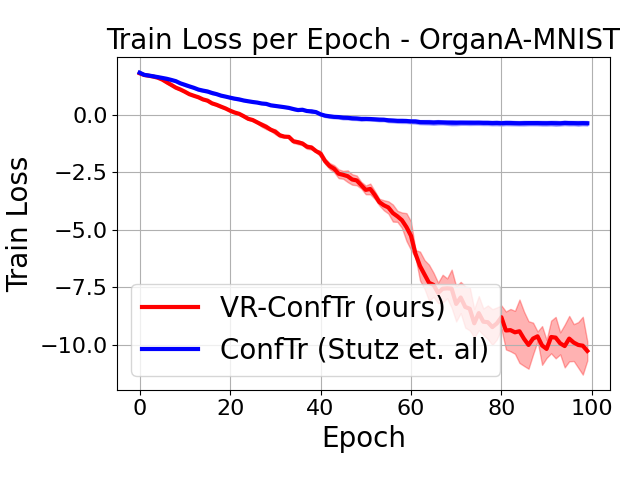}
    \includegraphics[width=0.24\textwidth]{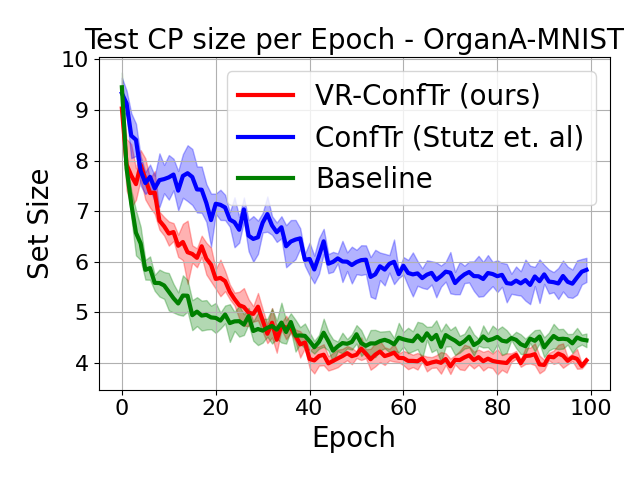}
\caption{Learning curves for MNIST, Fashion-MNIST, Kuzushiji-MNIST, and OrganAMNIST. Each row shows training loss (left) and test CP set sizes (right) for the corresponding dataset, evaluated using the \texttt{THR} conformal predictor.}
\label{fig:train_trajectories}
\end{figure}

In Figure~\ref{fig:train_trajectories}, we show the training performance for four datasets (MNIST, FMNIST, KMNIST and OrganAMNIST) illustrating two key metrics: (i) the evolution of the conformal training loss defined in section~\ref{sec:conformal} and (ii) the test CP size across epochs. 
In all the plots, we see that \texttt{VR-ConfTr}  reaches smaller values of the loss and in significantly fewer epochs as compared to~\texttt{ConfTr}. In the case of MNIST, for example, \texttt{VR-ConfTr} reaches a lower value of the loss in 10 times fewer epochs as compared to \texttt{ConfTr}. Similarly, for FMNIST \texttt{VR-ConfTr} achieves a smaller size in one third of epochs compared to \texttt{ConfTr}. For both Kuzushiji-MNIST and OrganA-MNIST, we notice that not only \texttt{VR-ConfTr} is faster, but it also gets to significantly smaller values of the loss. For the more challenging OrganA-MNIST dataset, this difference appears even more accentuated, not only in the training loss but also in the test CP set sizes. Notice that for all the three methods (\texttt{VR-ConfTr}, \texttt{ConfTr} and \texttt{Baseline}) we performed hyper-parameter tuning. Notably, in the case of the OrganA-MNIST dataset, we were not able to obtain an improvement with \texttt{ConfTr} in the final set size with respect to \texttt{Baseline}, which stresses the need for a method with improved gradient estimation, as the one we propose in this paper.

\textbf{Fine-tuning Cifar-10:} Figure~\ref{fig:cifar10-finetuning-results} presents the training performance for CIFAR-10, where we fine-tune the last linear layer of a pretrained ResNet20 trained with cross-entropy loss. Unlike the other datasets, where we train all model parameters using the \texttt{ConfTr} or \texttt{VR-ConfTr} loss, we train CIFAR-10 fine-tuning the final layer only. Fine-tuning with \texttt{ConfTr} has gained traction due to its reduced computational overhead~\cite{yan2024provably}. Despite modifying only the final layer, \texttt{Vr-ConfTr} consistently yields smaller CP sets compared to \texttt{ConfTr}, while also achieving higher accuracy within the first epoch.
\begin{figure}[h]
    \centering
    \includegraphics[width=0.45\linewidth]{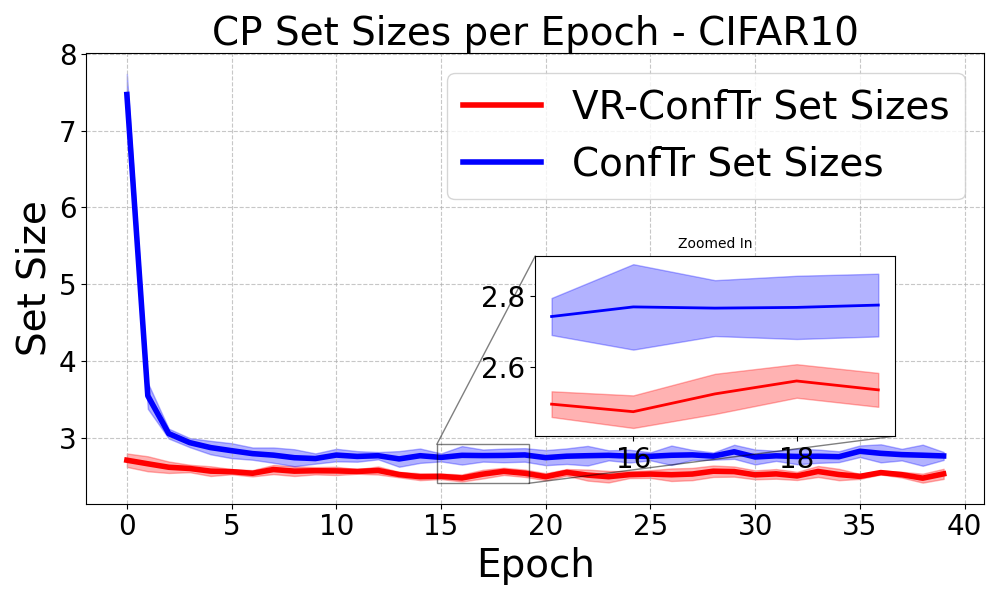}
    \includegraphics[width=0.45\linewidth]{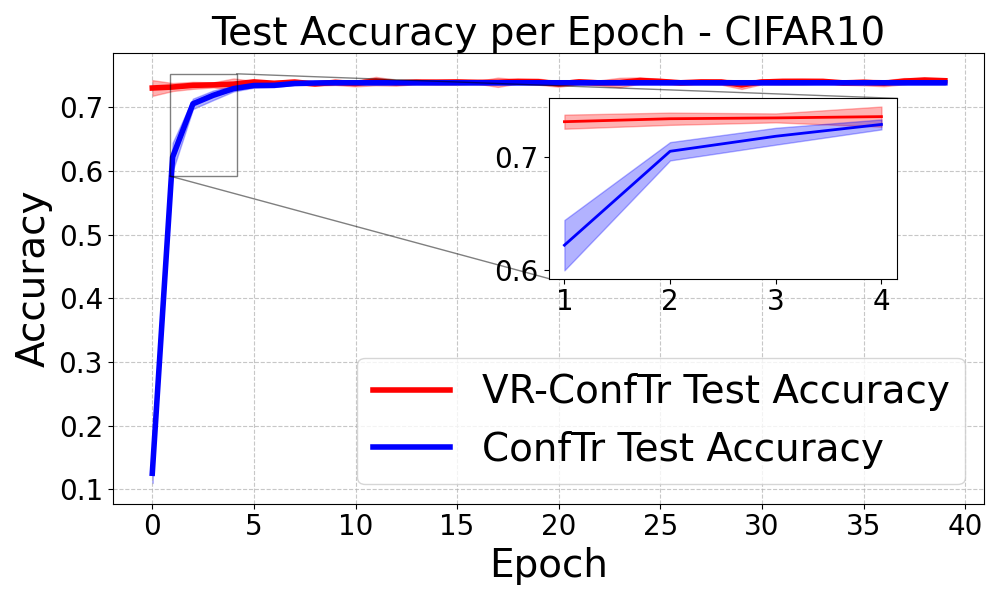}
    \caption{Learning curves for CIFAR-10 illustrating the fine-tuning process of a linear layer on a pretrained ResNet20 model using \texttt{ConfTr} and \texttt{VR-ConfTr}. Test CP set sizes are evaluated using the THR conformal predictor, consistent with the other datasets.}
    \label{fig:cifar10-finetuning-results}
\end{figure}

\section{Concluding Remarks and Future Directions}\label{sec:conclusion}

%
We theoretically justified the sample inefficiency in the \texttt{ConfTr} method proposed by~\citet{stutz2022learning}, which is a \emph{conformal risk minimization} (CRM) method for length efficiency optimization.
We have shown that the source of sample inefficiency lies in the estimation of the gradient of the population quantile. To address this issue, we introduced a novel
technique that improves the gradient estimation of the population quantile of the conformity scores by provably reducing its variance. We show that, by incorporating this estimation technique in our proposed \texttt{VR-ConfTr} algorithm, the training becomes more stable and the post-training conformal predictor is often more efficient as well. Our work also opens up possibilities for future research in the area of~CRM. Indeed, further methods for quantile gradient estimation could be developed and readily integrated with our ``plug-in'' algorithm, for which we can expect improved training performance.
\newpage
\section*{Acknowledgments}
The authors would like to thank Matthew Cleaveland for the insightful conversations and for pointing us to relevant literature.
\bibliographystyle{plainnat} 
\bibliography{references} 

\appendix
\onecolumn



\section{Proofs for Section~\ref{sec:conftr-analysis} and \ref{sec:vr-conftr}}\label{appendix:proofs}
In this appendix, we provide the proofs of all the theoretical results presented in section~\ref{sec:conftr-analysis}.
\begin{proposition}(\textbf{Proposition \ref{prop:a.s.-diff}})
$E_{(1)}(\theta), \ldots, E_{(n)}(\theta)$ are almost surely (a.s.) everywhere differentiable in $\theta$. In particular, the empirical quantile $\hat{\tau}_n(\theta) = E_{(\ceil{\alpha n})}(\theta)$ is a.s. everywhere differentiable.
\begin{proof}
We will formally show that the ordered statistics $E_{(j)}(\theta)$, with $j = 1, ..., n$, are differentiable for any $\theta$ with probability 1. We first recall some notation. Let $E_{(1)}(\theta) \leq \ldots \leq E_{(n)}(\theta)$ denote the order statistics corresponding to the scalar random variables $E(\theta, X_1,Y_1),\ldots,E(\theta, X_n,Y_n)$.
\newline
Let us also denote by $\omega(\theta): [n] \to [n]$ the permutation of indices $[n] := \{1,\ldots,n\}$ that correspond to the order statistics, i.e., $\omega(\theta) = (\omega_1(\theta), \ldots, \omega_n(\theta))$, and 
\[
(E_{(1)}(\theta), \ldots, E_{(n)}(\theta)) = (E(\theta, X_{\omega_1(\theta)}, Y_{\omega_1(\theta)}), 
\ldots, E(\theta, X_{\omega_n(\theta)}, Y_{\omega_n(\theta)})).
\]
Now define the set $A_n$ as follows:
\begin{equation}
A_n = \{(E_1, ..., E_n): E_i = E_j \text{ for some } i\neq j\}.
\end{equation}
Now note that, by definition, the conformity score function $E(\theta, X,Y)$ is continuous and differentiable in $\theta$. Now fix some $\bar{\theta}$. Consider the event in which the ordered statistics are such that $E_{(1)}(\bar{\theta})<\ldots<E_{(n)}(\bar{\theta})$, hence 
\begin{equation}
(E(\bar{\theta}, X_{\omega_1(\bar{\theta})},Y_{\omega_1(\bar{\theta})}), \ldots, E(\bar{\theta}, X_{\omega_n(\bar{\theta})},Y_{\omega_n(\bar{\theta})})) = (E_{(1)}(\bar{\theta}), \ldots, E_{(n)}(\bar{\theta})) \notin A_n,
\end{equation} 
which means that $\omega(\bar{\theta})$ is the unique ordered statistics permutation for $\{E(\theta, X_i,Y_i)\}_{i = 1}^n$ and note that this happens \textit{almost surely} (with probability 1), because $E(\theta, X, Y)$ is an absolutely continuous random variable. The key step is now to note that by continuity of $E(\theta, X_i, Y_i)$ in $\theta$, there exists $\delta > 0$ such that, for $\theta \in\{\theta' : \|\theta'-\bar{\theta}\|\leq \delta\}$, we have $\omega(\theta) = \omega(\bar{\theta})$, which means that, if $\|\theta-\bar{\theta}\|\leq \delta$,
\begin{equation}\label{eq:constant}
\begin{aligned}
(E_{(1)}({\theta}), \ldots, E_{(n)}({\theta})) &= (E({\theta}, X_{\omega_1({\theta})},Y_{\omega_1({\theta})}), \ldots, E({\theta}, X_{\omega_n({\theta})},Y_{\omega_n({\theta})}))
\\&
= (E({\theta}, X_{\omega_1({\bar{\theta}})},Y_{\omega_1({\bar{\theta}})}), \ldots, E({\theta}, X_{\omega_n({\bar{\theta}})},Y_{\omega_n({\bar{\theta}})}))
\end{aligned}
\end{equation}
At this point, let $j \in \{1, ..., n\}$, and let us denote $E_{(j)}({\theta}) = E({\theta}, X_{\omega_j({{\theta}})},Y_{\omega_j({{\theta}})}) = E({\theta},\omega_j({{\theta}}))$, and, for any $\theta\in\{\theta' : \|\theta'-\bar{\theta}\|\leq \delta\}$ the derivative of $E_{(j)}({\theta})$ is 
\begin{equation}
\frac{\partial}{\partial \theta}E_{(j)}({\theta}) = \frac{\partial}{\partial \theta}E({\theta},\omega_j({{\theta}})) = \frac{\partial}{\partial \theta}E({\theta},\omega_j({{\bar{\theta}}})) = \frac{\partial E}{\partial \theta}({\theta},\omega_j({{\bar{\theta}}})),
\end{equation}
which is true because, as we show in~\eqref{eq:constant} above, for $\theta\in\{\theta' : \|\theta'-\bar{\theta}\|\leq \delta\}$, the function $\omega_j({\theta})$ is a constant equal to $\omega_j({\bar{\theta}})$. Note that, as we do in the main paper, we here denote by $\frac{\partial E}{\partial \theta}({\theta},\omega_j({{\bar{\theta}}}))$ the partial derivative with respect to $\theta$. Note that, given that the choice of $\bar{\theta}$ is arbitrary, we have shown that the function $\theta\mapsto E(\theta, X_{\omega_{j}({\theta})}, Y_{\omega_{j}({\theta})})$ is indeed differentiable with probability 1 for all $j=1, \ldots, n$.  \newline
To be absolutely convinced that~\eqref{eq:constant} is true, note that we can show it by continuity of $\theta\mapsto E(\theta, X, Y)$, as follows: let's fix $\bar{\theta}$ and let us denote again $E_{(j)}({\theta}) = E({\theta}, X_{\omega_j({{\theta}})},Y_{\omega_j({{\theta}})}) = E({\theta},\omega_j({{\theta}}))$. We want to show that there exists $\delta >0$ such that $\omega(\bar{\theta}) = \omega(\theta)$ for any $\theta\in \{\theta' : \|\theta'-\bar{\theta}\|\leq \delta\}$. To do so, it is sufficient to show that, for any $\theta\in \{\theta' : \|\theta'-\bar{\theta}\|\leq \delta\}$, 
\begin{equation}
E(\theta,\omega_{i+1}(\bar{\theta})) > E(\theta,\omega_{i}(\bar{\theta})), \text{ for } i = 1, ..., n-1.     
\end{equation}
Let us define
\begin{equation}\label{eq:minEps}
\varepsilon = \min_{i = 1, ..., n-1}\{E({\bar{\theta}},\omega_{i+1}({{\bar{\theta}}})) - E({\bar{\theta}},\omega_{i}({{\bar{\theta}}}))\}.
\end{equation}
From continuity of $\theta\mapsto E(\theta, X, Y)$, there exists $\delta>0$ such that if $\theta\in\{\theta' : \|\theta'-\bar{\theta}\|\leq \delta\}$, we have 
\begin{equation}\label{eq:contEps2}
|E(\theta,\omega_{i}(\bar{\theta})) - E({\bar{\theta}},\omega_{i}(\bar{\theta}))|< \frac{\epsilon}{2}.
\end{equation}
Note that, from~\eqref{eq:minEps} and~\eqref{eq:contEps2}, we have for all $i = 1, \ldots, n$,
\begin{equation}\label{eq:epsilon3}
    E({\bar{\theta}},\omega_{i+1}({{\bar{\theta}}})) \geq E({\bar{\theta}},\omega_{i}({{\bar{\theta}}})) + \epsilon,
\end{equation}
and
\begin{equation}
    E({{\theta}},\omega_{i}(\bar{\theta}))>E(\bar{\theta},\omega_{i}(\bar{\theta}))-\frac{\epsilon}{2}.
\end{equation}
Hence, note that, starting from this last inequality, and then using~\eqref{eq:epsilon3}
\begin{equation}
\begin{aligned}
    E({{\theta}},\omega_{i+1}(\bar{\theta}))&> E(\bar{\theta},\omega_{i+1}(\bar{\theta})) -\frac{\epsilon}{2}
    \\&\geq E(\bar{\theta},\omega_{i}(\bar{\theta})) + \frac{\epsilon}{2}.
\end{aligned}
\end{equation}
Now, we can use again~\eqref{eq:contEps2} (continuity) to show that
\begin{equation}
    E({\bar{\theta}},\omega_{i}(\bar{\theta}))>E({\theta},\omega_{i}(\bar{\theta}))-\frac{\epsilon}{2},
\end{equation}
and thus observe that, for all $i = 1, \ldots, n$,
\begin{equation}
    E({{\theta}},\omega_{i+1}(\bar{\theta}))>E(\bar{\theta},\omega_{i}(\bar{\theta})) + \frac{\epsilon}{2}>E({\theta},\omega_{i}(\bar{\theta})),
\end{equation}
from which we can confirm that, for $\theta\in\{\theta' : \|\theta'-\bar{\theta}\|\leq \delta\}$, $\omega(\theta) = \omega(\bar{\theta})$, and we can conclude.
\end{proof}
\end{proposition}



\begin{lemma}
Given any continuous, bounded function $h$, the function $g(t,\theta) = \EE{h(\frac{\partial E}{\partial\theta}(\theta,Z) \,\vert\,E_\theta(Z) = t}$ is continuous in $t$.
\label{lemma:conditional-expectation-continuous}
\end{lemma}

\begin{proof}
Given an absolutely continuous random vector~$V$, we will denote its probability density function (PDF) as $f_V(v)$. Let $\mathcal{E}_\theta$ and $\mathcal{V}_\theta$ denote the support of $E_\theta(Z)$ and $\frac{\partial E}{\partial\theta}(\theta, Z)$ respectively, i.e. the set of points where the PDF is strictly positive. Suppose that $t\in\mathcal{E}_\theta$. Then,
\begin{align*}
    g(t,\theta) &= \EE{h\left(\frac{\partial E}{\partial\theta}(\theta,Z)\right) \,\vert\,E_\theta(Z) = t}\\
    &= \int_{\mathcal{V}_\theta} f_{\frac{\partial E}{\partial\theta}(\theta, Z)|E_\theta(Z)=t}(v)\,h\left(\frac{\partial E}{\partial\theta}(\theta,v)\right)\,\mathrm{d}v\\
    &= \int_{\mathcal{V}_\theta} \frac{f_{\frac{\partial E}{\partial\theta}(\theta, Z),E_\theta(Z)}(v, t)}{f_{E_\theta(Z)}(t)}h\left(\frac{\partial E}{\partial\theta}(\theta,v)\right)\,\mathrm{d}v.
\end{align*}
Since $f_{\frac{\partial E}{\partial\theta}(\theta, Z),E_\theta(Z)}(v, t)$ is continuous in $t$ (from Assumption~\ref{ass:continuous-joint-pdf}), then so is $f_{E_\theta(Z)}(t)$. Subsequently, their ratio is continuous, and the result follows.

\end{proof}


\begin{proposition}(\textbf{Proposition \ref{thm:conftr-weakconvergence}})
Let $\hat{\tau}_n(\theta) = E_{(\ceil{\alpha n})}(\theta)$. Then,
\begin{equation}
    \frac{\partial \hat{\tau}_n}{\partial\theta}(\theta) \overset{\mathrm{dist}}{\longrightarrow} \frac{\partial E}{\partial\theta}(\theta,Z) \,\big|_{E_\theta(Z)=\tau(\theta)}
\end{equation}
as $n\to\infty$.
\end{proposition}
\begin{proof}

Let $Z_i = (X_i, Y_i)$ for $i\in\{1,\ldots,n\}$ and $Z = (X,Y)$. 
Let $h$ be a bounded, continuous function. Then,
\begin{align*}
    &\EE{h\left(\frac{\partial\hat{\tau}_n}{\partial\theta}(\theta)\right)} = \EE{h\left(\frac{\partial E}{\partial\theta}(\theta, Z_{\omega_{\ceil{\alpha n}}(\theta)})\right)}\\
    &= \mathbb{E}_{t\sim \mathrm{dist}(\hat{\tau}_n(\theta))}\left[\EE{ h\left(\frac{\partial E}{\partial\theta}(\theta, Z_{\omega_{\ceil{\alpha n}}(\theta)})\right) \,\big|\,\hat{\tau}_n(\theta) = t}\right].
\end{align*}
Now, focusing on the inner expectation, we can see that $\hat{\tau}_n(\theta)=t$ occurs almost surely (assuming no ties in the scores $E_\theta(Z_1),\ldots,E_\theta(Z_n)$) if and only if exactly one of the scores is equal to $t$ and exactly $\ceil{\alpha n} - 1$ of the remaining being strictly smaller than $t$. Due to exchangeability of the scores, every ordering is equally likely. Therefore,
\begin{equation*}
\begin{split}
    &\EE{ h\left(\frac{\partial E}{\partial\theta}(\theta, Z_{\omega_{\ceil{\alpha n}}(\theta)})\right) \,\big|\,\hat{\tau}_n(\theta) = t}\\
    &= \mathbb{E}\Big[ h\left(\frac{\partial E}{\partial\theta}(\theta, Z_1)\right) \,\big|\, E_\theta(Z_1) = t,\\
    &\qquad \sum_{i=2}^n 1_{E_\theta(Z_i) < t} = \ceil{\alpha n} - 1\Bigg].
\end{split}
\end{equation*}
Now, due to the assumed independence of $Z_1,\ldots,Z_n$, we can drop $\{\sum_{i=2}^n 1_{E(\theta,Z_i) < t} = \ceil{\alpha n} - 1\}$ from the conditional, because $h\left(\frac{\partial E}{\partial\theta}(\theta, Z_1)\right)$ does not depend on $Z_2,\ldots,Z_n$. Subsequently,
\begin{align*}
    &\EE{h\left(\frac{\partial\hat{\tau}_n}{\partial\theta}(\theta)\right)}\\
    &= \mathbb{E}_{t\sim \mathrm{dist}(\hat{\tau}_n(\theta))}\Bigg[\EE{h\left(\frac{\partial E}{\partial\theta}(\theta, Z_1)\right) \,\big|\, E_\theta(Z_1) = t}\Bigg]\\
    &= \mathbb{E}_{t\sim \mathrm{dist}(\hat{\tau}_n(\theta))}\Bigg[\EE{h\left(\frac{\partial E}{\partial\theta}(\theta, Z)\right) \,\big|\, E_\theta(Z) = t}\Bigg]\\
    &= \mathbb{E}_{t\sim \mathrm{dist}(\hat{\tau}_n(\theta))}[g(t,\theta)],
\end{align*}
where $g(t,\theta) := \EE{h\left(\frac{\partial E}{\partial\theta}(\theta, Z)\right) \,\big|\, E_\theta(Z) = t}$. Therefore, 
\begin{equation*}
    \EE{h\left(\frac{\partial\hat{\tau}_n}{\partial\theta}(\theta)\right)} = \EE{g(\hat{\tau}_n(\theta), \theta)}.
\end{equation*}
Noting that $\hat{\tau}_n(\theta) \to \tau(\theta)$ in distribution, it follows from Lemma~\ref{lemma:conditional-expectation-continuous} and by the continuous mapping theorem that $g(\hat{\tau}_n(\theta), \theta) \overset{\mathrm{dist}}{\longrightarrow} g(\tau(\theta),\theta)$. We can rewrite this as
\begin{equation*}
\begin{split}
    &\lim_{n\to\infty}\EE{h\left(\frac{\partial E}{\partial\theta}(\theta, Z_{\omega_{\ceil{\alpha n}}})\right)}\\
    &= \EE{h\left(\frac{\partial E}{\partial\theta}(\theta, Z)\right)\,\big\vert\, E_\theta(Z) = \tau(\theta)}.
\end{split}
\end{equation*}
Since $h$ was arbitrary, it follows from the Portmanteau lemma that
\begin{equation}
\begin{split}
    \frac{\partial E}{\partial\theta}(\theta, Z_{\omega_{\ceil{\alpha n}}}) \overset{\mathrm{dist}}{\longrightarrow} \frac{\partial E}{\partial\theta}(\theta, Z)\,\big\vert_{E_\theta(Z) = \tau(\theta)}
\end{split}
\end{equation}
as $n\to\infty$. 

\end{proof}

\begin{corollary}\label{app:corVar}
Under the same conditions as Proposition~\ref{thm:conftr-weakconvergence}, we have
\begin{equation}
    \EE{\frac{\partial \hat{\tau}_n}{\partial\theta}(\theta)} \to \EE{\frac{\partial E}{\partial\theta}(\theta,Z) \,\Big|\, E_\theta(Z) = \tau(\theta)}
\end{equation}
and
\begin{equation}
    \cov{\frac{\partial \hat{\tau}_n}{\partial\theta}(\theta)} \to \cov{\frac{\partial E}{\partial\theta}(\theta,Z) \,\Big|\, E_\theta(Z) = \tau(\theta)}
\end{equation}
as $n\to\infty$.
\end{corollary}
\begin{proof}
For the corollary, simply note that we can apply the Portmanteau lemma again, choosing some continuous, bounded $h$ such that $h(\theta) = \theta_l$ and $h(\theta) = \theta_l\theta_p$ (with $l,p$ indexing components of $\theta$) over $\{\|\theta\| \leq M\}$. This way, noting that $\|\frac{\partial E}{\partial\theta}(\theta,z)\| \leq M$, we can see that
\begin{equation*}
    \EE{\frac{\partial \hat{\tau}_n}{\partial\theta_l}(\theta)} \to \EE{\frac{\partial E}{\partial\theta_l}(\theta,Z) \,\Big|\, E_\theta(Z) = \tau(\theta)}
\end{equation*}
and
\begin{equation*}
\begin{split}
    &\EE{\frac{\partial \hat{\tau}_n}{\partial\theta_l}(\theta)\frac{\partial \hat{\tau}_n}{\partial\theta_p}(\theta)}\\
    &\to \EE{\frac{\partial E}{\partial\theta_l}(\theta,Z) \frac{\partial E}{\partial\theta_p}(\theta,Z) \,\Big|\, E_\theta(Z) = \tau(\theta)}.
\end{split}
\end{equation*}
Gathering all the indices $l,p$ completes the proof.
\end{proof}

\subsection*{Proof of Proposition~\ref{th:fundLemma}}

Let $H(s) = \begin{cases} 1, & s \geq 0 \\ 0, & s < 0 \end{cases}$ denote the Heaviside step function, and let $H_n(s) = \Phi\left( \frac{s}{\sigma_n} \right)$ denote a smooth approximation, where $\Phi$ denotes the cumulative distribution function (CDF) of the standard Gaussian distribution, and $\sigma_n > 0$ is a sequence such that $\sigma_n \to 0$ as $n \to \infty$. Note that $H_n(s) \to H(s)$ pointwise as $n \to \infty$, and that each $H_n$ is smooth.

By definition, we have
\begin{equation}
    \mathbb{P}\left[ E(\theta, X, Y) \leq \tau(\theta) \right] = \alpha.
\end{equation}

which we can rewrite as
\begin{equation}
    \mathbb{E}\left[ H\big( \tau(\theta) - E(\theta, X, Y) \big) \right] = \alpha.
    \label{eq:heaviside_expectation}
\end{equation}

Since $0 \leq H(s) \leq 1$ for all $s$, and $H_n(s) \to H(s)$ pointwise, by the Dominated Convergence Theorem, we have
\begin{align}
    \alpha &= \mathbb{E}\left[ H\big( \tau(\theta) - E(\theta, X, Y) \big) \right] \nonumber \\
    &= \mathbb{E}\left[ \lim_{n \to \infty} H_n\big( \tau(\theta) - E(\theta, X, Y) \big) \right] \nonumber \\
    &= \lim_{n \to \infty} \mathbb{E}\left[ H_n\big( \tau(\theta) - E(\theta, X, Y) \big) \right].
    \label{eq:limit_expectation}
\end{align}
Differentiating both sides with respect to $\theta$, we obtain
\begin{equation}
    0 = \frac{\partial}{\partial \theta} \lim_{n \to \infty} \mathbb{E}\left[ H_n\left( \tau(\theta) - E(\theta, X, Y) \right) \right], \label{eq:differentiation_step} 
\end{equation}
where $\alpha$ is a constant independent of $\theta$.
By interchanging the limit and differentiation, this becomes
\begin{equation}
    0 = \lim_{n \to \infty} \frac{\partial}{\partial \theta} \mathbb{E}\left[ H_n\left( \tau(\theta) - E(\theta, X, Y) \right) \right]. \label{eq:interchange_limit_derivative}
\end{equation}
To justify the interchange, we note that $f_n(\theta) = \mathbb{E}\left[ H_n\left( \tau(\theta) - E(\theta, X, Y) \right) \right]$ converges pointwise to $\alpha$, a constant. By uniform convergence of $f_n(\theta)$ and its derivative$\frac{\partial}{\partial \theta}$, we can exchange the limit and differentiation.
Using the Leibniz Integral Rule, we interchange differentiation and expectation:
\begin{equation}
    0 = \lim_{n \to \infty} \mathbb{E}\left[ \frac{\partial}{\partial \theta} H_n\left( \tau(\theta) - E(\theta, X, Y) \right) \right]. \\
    \label{eq:interchange_derivative_expectation}
\end{equation}
The interchange is valid because $H_n$ is infinitely differentiable, $\tau(\theta)$ and $E(\theta, X, Y)$ are continuously differentibale with respect to $\theta$, and the derivative $\frac{\partial}{\partial \theta} H_n(\tau(\theta) - E(\theta, X, Y))$ is continuous in $\theta$ and integrable.
Finally applying the chain rule to differentiate $H_n\left( \tau(\theta) - E(\theta, X, Y) \right)$ with respect to $\theta$, we get:
\begin{equation}
    0 = \lim_{n \to \infty} \mathbb{E}\left[ H_n'\left( \tau(\theta) - E(\theta, X, Y) \right) \left( \frac{\partial \tau}{\partial \theta}(\theta) - \frac{\partial E}{\partial \theta}(\theta, X, Y) \right) \right].
    \label{eq:chain_rule}
\end{equation}
Define
\begin{align*}
    \delta_n(s) &= H_n'(s) = \frac{1}{\sqrt{2\pi} \sigma_n} e^{-\frac{s^2}{2 \sigma_n^2}}, \\
    \gamma_n(\theta, x, y) &= \delta_n\left( \tau(\theta) - E(\theta, x, y) \right), \\
    \Delta(\theta, x, y) &= \frac{\partial \tau}{\partial \theta}(\theta) - \frac{\partial E}{\partial \theta}(\theta, x, y).
\end{align*}
Now, we can see that
\begin{equation}
    \frac{\partial}{\partial \theta} \mathbb{E}\left[ H_n\left( \tau(\theta) - E(\theta, X, Y) \right) \right] = \mathbb{E}\left[ \gamma_n(\theta, X, Y) \Delta(\theta, X, Y) \right].
\end{equation}
Let $a_n \asymp b_n$ denote asymptotic equivalence, meaning that $\lim_{n\to\infty} \frac{a_n}{b_n} = 1$, assuming $b_n \neq 0$ for finite $n$. From equation \eqref{eq:chain_rule}, it follows that
\begin{equation}
    \mathbb{E}\left[ \gamma_n(\theta, X, Y) \Delta(\theta, X, Y) \right] \asymp 0.
    \label{eq:pre-limit}
\end{equation}

\textbf{Analyzing the Expectation:}

Let $\varepsilon_n > 0$ be any sequence such that $\varepsilon_n = o(\sigma_n)$, meaning $\varepsilon_n / \sigma_n \to 0$ as $n \to \infty$. Define the set
\begin{equation}
    A_{\varepsilon_n}(\theta) = \left\{ \omega \in \Omega : -\varepsilon_n < E(\theta, X(\omega), Y(\omega)) - \tau(\theta) < \varepsilon_n \right\}.
\end{equation}

We can decompose the expectation in \eqref{eq:pre-limit} as
\begin{align}
    \mathbb{E}\left[ \gamma_n \Delta \right] &= \mathbb{P}\left( A_{\varepsilon_n}(\theta) \right) \mathbb{E}\left[ \gamma_n \Delta \,|\, A_{\varepsilon_n}(\theta) \right] + \mathbb{P}\left( A_{\varepsilon_n}^c(\theta) \right) \mathbb{E}\left[ \gamma_n \Delta \,|\, A_{\varepsilon_n}^c(\theta) \right].
    \label{eq:expectation_decomposition}
\end{align}

\textbf{Negligibility of the Second Term}: On the complement \( A_{\varepsilon_n}^c(\theta) \), the value \( \tau(\theta) - E(\theta, X, Y) \) is either greater than \( \varepsilon_n \) or less than \( -\varepsilon_n \). Therefore, for \(\ s \geq \varepsilon_n \) or \( s \leq -\varepsilon_n \), \( \delta_n(s) = H_n'(s) \) becomes very small. Particularly
\[
\delta_n(s) = \dfrac{1}{\sqrt{2\pi} \, \sigma_n} \exp\left( -\dfrac{s^2}{2\sigma_n^2} \right).
\] Since \( \varepsilon_n = o(\sigma_n) \) and \( \sigma_n \to 0 \), it follows that \( \dfrac{\varepsilon_n}{\sigma_n} \to \infty \). For \( |s| \geq \varepsilon_n \), we have:
$
\delta_n(s) \leq \dfrac{1}{\sqrt{2\pi} \, \sigma_n} \exp\left( -\dfrac{\varepsilon_n^2}{2\sigma_n^2} \right)
$. from \( \dfrac{\varepsilon_n}{\sigma_n} \to \infty \), it also follows that \( \exp\left( -\dfrac{\varepsilon_n^2}{2\sigma_n^2} \right) \to 0 \) and \[
\gamma_n(\theta, X, Y) \leq \dfrac{1}{\sqrt{2\pi} \, \sigma_n} \exp\left( -\dfrac{\varepsilon_n^2}{2\sigma_n^2} \right) \to 0.
\]
Therefore, \eqref{eq:expectation_decomposition} becomes
\begin{equation}
    \mathbb{E}\left[ \gamma_n \Delta \right] \asymp \mathbb{P}\left( A_{\varepsilon_n}(\theta) \right) \mathbb{E}\left[ \gamma_n \Delta \,|\, A_{\varepsilon_n}(\theta) \right].
    \label{eq:approximate_expectation}
\end{equation}
Plugging in \eqref{eq:approximate_expectation} into \eqref{eq:pre-limit} we get:
\begin{equation*}
    \mathbb{E}[\gamma_n(\theta,X,Y)\Delta(\theta,X,Y)\,|\,A_{\varepsilon_n}(\theta)] \asymp 0,
    \label{eq:simplified-expectation}
\end{equation*}

By noting that $\mathbb{P}(A_{\varepsilon_n}(\theta)) > 0$ for all $n$ due to continuity of $x\mapsto E(\theta,x,y)$. We can then rewrite as
\begin{equation}
    \mathbb{E}[\gamma_n(\theta,X,Y),|\,A_{\varepsilon_n}(\theta)]\frac{\partial\tau}{\partial\theta}(\theta) \asymp \mathbb{E}\Bigg[\gamma_n(\theta,X,Y)\frac{\partial E}{\partial\theta}(\theta,X,Y)\,|\,A_{\varepsilon_n}(\theta)\Bigg].
    \label{eq:pre-limit3}
\end{equation}
When conditioned on $A_{\varepsilon_n}(\theta)$, we have $\gamma_n(\theta,X,Y) \asymp \delta_n(0)$. Indeed, note that $\delta_n(\varepsilon_n) \leq \gamma_n(\theta,X,Y) \leq \delta_n(0)$, which we can rewrite as $\frac{\delta_n(\varepsilon_n)}{\delta_n(0)} \leq \frac{\gamma_n(\theta,X,Y)}{\delta_n(0)} \leq 1$. Noting that $\delta_n(\varepsilon_n) \asymp 1$, it indeed follows that $\gamma_n(\theta,X,Y) \asymp \delta_n(0)$. Therefore, \eqref{eq:pre-limit3} can be further simplified to
\begin{equation}
    \frac{\partial\tau}{\partial\theta}(\theta) \asymp \mathbb{E}\Bigg[\frac{\partial E}{\partial\theta}(\theta,X,Y) \,\big|\, A_{\varepsilon_n}(\theta) \Bigg],
    \label{eq:pre-limit4}
\end{equation}
which readily leads to
\begin{align*}
    \frac{\partial\tau}{\partial\theta}(\theta) &= \lim_{n\to\infty} \mathbb{E}\Bigg[\frac{\partial E}{\partial\theta}(\theta,X,Y) \,\big|\, A_{\varepsilon_n}(\theta) \Bigg]\\
    &= \lim_{\varepsilon\to 0} \mathbb{E}\Bigg[\frac{\partial E}{\partial\theta}(\theta,X,Y) \,\big|\, A_{\varepsilon}(\theta) \Bigg]\\
    &= \lim_{\varepsilon\to 0} \mathbb{E}\Bigg[\frac{\partial E}{\partial\theta}(\theta,X,Y) \,\big|\, -\varepsilon < E_\theta(X,Y) - \tau(\theta) < \varepsilon \Bigg]\\
    &= \mathbb{E}\Bigg[\frac{\partial E}{\partial\theta}(\theta,X,Y) \,\big|\, E_\theta(X,Y) = \tau(\theta)\Bigg].
\end{align*}
\subsection*{Proof of Theorem~\ref{th:VR_theorem}}
We start with two preliminary results that we will use in the proof.

\textbf{Some preliminaries.} First, we recall a well-known result. Let $k\sim \text{Binomial}(n, p)$ be a random variable sampled from a Binomial distribution with $n$ trials and with probability $p$ of success. The following holds:
\begin{equation}\label{eq:expBinom}
\Exp{\frac{1}{1+k}} = \frac{(1-(1-p)^{n+1})}{(n+1)p}.
\end{equation}
Note that this follows from the following simple steps:
\begin{equation*}
\begin{aligned}
\Exp{\frac{1}{1+k}} &= \sum_{k = 0}^n\frac{1}{1+k}\cdot\binom{n}{k}p^k\pars{1-p}^{n-k} \\&= \frac{1}{p(n+1)}\sum_{k = 0}^n\binom{n+1}{k+1}p^{k+1}(1-p)^{n-k}\\&
= \frac{1}{p(n+1)}\sum_{j = 1}^{n+1}\binom{n+1}{j}p^{j}(1-p)^{n+1-j}
\\&
= \frac{\pars{1-(1-p)^{n+1}}}{p(n+1)}.
\end{aligned}
\end{equation*}
Next, we state another well-known identity. Let us consider the following recursion:
\begin{equation*}
a_{n+1} = \rho\,a_n + b,
\end{equation*}
where $\rho>0$. Simply unrolling the recursion, we can obtain
\begin{equation}\label{eq:recursion}
a_n = \rho^na_0 + b\pars{\frac{1-\rho^n}{1-\rho}}.
\end{equation}
\textbf{Proof of the Theorem. }

Before we proceed, we introduce some notation. For simplicity, since $\hat{\tau}(\theta)$ converges \textit{almost surely} to ${\tau}(\theta)$, we assume $\hat{\tau}(\theta) = \tau(\theta)$ in the analysis. Let us then denote
\begin{equation}
\begin{aligned}
G_i(\theta)&=\frac{\partial E}{\partial \theta}(\theta, X_i,Y_i),\\
A_{\varepsilon, i}(\theta) &= \{\varepsilon\leq E(\theta, X_i,Y_i)-\tau(\theta)\leq \varepsilon\},
\\
R_{\varepsilon, n}(\theta)&=\sum_{i=1}^n\chi_{A_{\varepsilon, i}(\theta)},
\\
\Sn &= \{i\in[n]:  \chi_{\Ae} = 1\},
\label{eq:defEqs}
\end{aligned}
\end{equation}
where $\chi_{A_{\varepsilon, i}(\theta)}$ is an indicator function for the event $A_{\varepsilon, i}(\theta)$, i.e.,
\begin{equation}
\chi_{A_{\varepsilon, i}(\theta)}  = 
\begin{cases} 
    1, &\textnormal{if } |E_\theta(X_i,Y_i)-\tau(\theta)|\leq \varepsilon\\
    0, &\textnormal{if } |E_\theta(X_i,Y_i)-\tau(\theta)| > \varepsilon
\end{cases}.
\end{equation}
We are now ready to analyze the estimator $\hat{\eta}_{\varepsilon, n}(\theta)$ for ${\eta}_{\varepsilon}(\theta)$:
\begin{equation}\label{eq:estimator}
    \hat{\eta}_{\varepsilon, n}(\theta) =\begin{cases} 
        \frac{1}{R_{\varepsilon, n}(\theta)}\sum_{i = 1}^n\chi_{A_{\varepsilon, i}(\theta)}G_i(\theta), & \textrm{if}\hspace{0.1cm} R_{\varepsilon, n}(\theta)>0\\
        0& \textnormal{if } R_{\varepsilon, n}(\theta) = 0
    \end{cases},
\end{equation}
where $\varepsilon$ and $n$ are denoted explicitly to remove ambiguity.

Equipped with the basic results established earlier in this subsection, we can proceed first with proving assertion \emph{(i)}. Note that, by definition~\eqref{eq:estimator}, and because $\{X_i, Y_i\}_{i = 1}^n$ are sampled independently, we have
\begin{equation}
\Exp{\frac{1}{|\Sn|}\sum_{i\in \Sn}G_i|\bigcap_{i\in \Sn}\Ae} = \frac{1}{|\Sn|}\sum_{i\in \Sn}\Exp{\chi_{\Ae}G_i(\theta)|\Ae} = \eta_{\epsilon}.
\end{equation}
Also note that $\Sn = \emptyset$ is equivalent to $\Rn = 0$, and that 
\begin{equation}
\Prob{\Sn = \emptyset} = \Prob{\Rn = 0} = q_{\epsilon}(\theta)^n,    
\end{equation}
with $q_{\epsilon}(\theta) = 1-p_{\epsilon}(\theta)$ and $p_{\epsilon}(\theta) = \mathbb{P}(A_{\varepsilon, i}(\theta))$. For simplicity, we denote $p = p_{\varepsilon}(\theta)$ and $q = 1 - p$ for the remainder of the proof. 
Hence, we can get (i) as follows:
\begin{equation}
\Exp{\heta} = q^n\Exp{\heta|\Rn = 0} + (1-q^n)\eta_{\epsilon},
\end{equation}
where we used the fact that $\sum_{S\neq\emptyset}\Prob{\Sn = S} = 1-\Prob{\Sn = \emptyset} = 1-q^n$.

Now, we prove (ii). We start by analyzing $\Exp{\heta\heta^\top}$:
\begin{equation}
\begin{aligned}
\Exp{\heta\heta^\top} &= \sum_{S\subseteq[n]}\Prob{\Sn = S}\Exp{\heta\heta^\top|\Sn = S}
\\& = \Prob{\Sn = \emptyset}\Exp{\heta\heta^\top|\Sn = \Sn}
\\& +\sum_{S\neq\emptyset}\Prob{\Sn = S}\Exp{\pars{\frac{1}{|S|}\sum_{i\in S}G_i(\theta)}\pars{\frac{1}{|S|}\sum_{i\in S}G_i(\theta)}^\top|\bigcap_{i\in S}\Ae}
\\&
= \Prob{\Sn = \emptyset}\Exp{\heta\heta^\top|\Rn = 0}
\\&+\sum_{S\neq\emptyset}\Prob{\Sn = S}\frac{1}{|S|^2}\sum_{i\in S}\sum_{j\in S}\Exp{G_i(\theta)G_j(\theta)^\top|\Ae, \Aej}.
\end{aligned}
\end{equation}
Now note that
\begin{equation}
\begin{aligned}
\Exp{G_i(\theta)G_j(\theta)^\top|\Ae, \Aej} &= \delta_{i,j}\Exp{G_i(\theta)G_i(\theta)^\top|\Ae} \\&+(1-\delta_{ij})\Exp{G_i(\theta)|\Ae}\Exp{G_j(\theta)^\top|\Aej}
\\&=\Exp{G_i(\theta)|\Ae}\Exp{G_j(\theta)^\top|\Ae}
\\
&+ \delta_{ij}\left( \Exp{G_i(\theta)G_i(\theta)^\top|\Ae} 
    - \Exp{G_i(\theta)|\Ae}\Exp{G_i(\theta)^\top|\Ae} \right) \\
&=
\eta_{\epsilon}\eta_{\epsilon}^\top + \delta_{ij}\Sigma_{\epsilon},
\end{aligned}
\end{equation}
where we used the fact that $\{X_i, Y_i\}_{i = 1}^n$ are sampled i.i.d. and the definitions in~\eqref{eq:defEqs}. Now, we can proceed as follows:
\begin{equation}\label{eq:boundStat}
\begin{aligned}
\Exp{\heta\heta^\top} &= q^n\Exp{\heta\heta^\top|\Rn = 0}
\\& 
+ \sum_{S\neq\emptyset}\frac{\Prob{\Sn = S}}{|S|}\pars{|S|\eta_{\epsilon}\eta_{\epsilon}^\top +\Sigma_{\epsilon}}\\&
= q^n\Exp{\heta\heta^\top|\Rn = 0}\\&
+ (1-q^n)\eta_{\epsilon}\eta_{\epsilon}^\top + f_n\Sigma_{\epsilon},
\end{aligned}
\end{equation}
where we write
\begin{equation}
    f_n = \sum_{S\neq\emptyset}\frac{\Prob{\Sn = S}}{|S|}.
\end{equation}
Now, we will show that 
\begin{equation}
f_n \leq \frac{2-p}{p n}.
\end{equation}
First, let us define the following function
\begin{equation}
f(k) = \begin{cases}
    0, \ &\text{if } k = 0\\
    \frac{1}{k} &\text{if } k\geq1
\end{cases},
\end{equation}
and note that 
\begin{equation}
\begin{aligned}
f_n = \Exp{f(|\Sn|)} = \Exp{f(\Rn)}.
\end{aligned}
\end{equation}
Now note that
\begin{equation}
\begin{aligned}
f_{n+1} &= \Exp{f(R_{\epsilon, n+1}(\theta))}\\&
 = \Prob{\Ae^c}\Exp{f(\Rn)} + \Prob{\Ae}\Exp{f(1+\Rn)}
 \\&
 = q\Exp{f(\Rn)} + p\Exp{f(1+\Rn)}
 \\&
 = qf_n + p\Exp{\frac{1}{1+\Rn}}
 \\&
 = qf_n + \frac{1-q^{n+1}}{n+1}, 
\end{aligned}
\end{equation}
where, in the last equation, we used the fact shown in the preliminaries (see~\eqref{eq:expBinom}):
\begin{equation}
\Exp{\frac{1}{1+\Rn}} = \frac{1-q^{n+1}}{p(n+1)}.
\end{equation}
Now let $a_n = nf_n$. We can write
\begin{equation}
\begin{aligned}
(n+1)f_{n+1} = (n+1)\pars{qf_n + \frac{1-q^n}{n+1}},
\end{aligned}
\end{equation}
from which we obtain the following recursion:
\begin{equation}
\begin{aligned}
a_{n+1} &= qnf_n + qf_n + (1-q^{n+1})
\\&
= qa_n + qf_n + (1-q^{n+1})
\\&
\leq qa_n + 1+q,
\end{aligned}
\end{equation}
where we used the fact that $f_n \leq 1$ and that $1-q^n\leq 1$. With this recursion, we can now use the result illustrated in the preliminaries in \eqref{eq:recursion} and get, using $q = 1-p$,
\begin{equation}
    a_n\leq q^na_0 + \frac{1-q^n}{1-q}(1+q)\leq \frac{2-p}{p}.
\end{equation}
From the above inequality, we can conclude that 
\begin{equation}
0\leq f_n= \frac{a_n}{n}\leq \frac{2-p}{p n}.
\end{equation}
Plugging this last result in~\eqref{eq:boundStat}, we can get
\begin{equation}
\Exp{\heta\heta^\top}\preceq q^n\Exp{\heta\heta^\top|\Rn = 0} + (1-q^n)\eta_{\epsilon}\eta_{\epsilon}^\top + \frac{2-p}{p n}\Sigma_{\epsilon}.
\end{equation}
We are now in the position to write and bound $\text{cov}(\heta)$:
\begin{equation}
\begin{aligned}
\text{cov}\left({\hat{\eta}_{\epsilon,n}(\theta)}\right) &= \mathbb{E}\left[{\hat{\eta}_{\epsilon,n}(\theta)\hat{\eta}_{\epsilon,n}(\theta)^\top}\right] -
\left[{\hat{\eta}_{\epsilon,n}(\theta)}\right]\left[{\hat{\eta}_{\epsilon,n}(\theta)^\top}\right]
\\&
\preceq q^n\mathbb{E}\left[{\hat{\eta}_{\epsilon,n}(\theta)\hat{\eta}_{\epsilon,n}(\theta)^\top|R_{\epsilon, n}(\theta) = 0}\right] + (1-q^n)\eta_{\epsilon}\eta_{\epsilon}^\top + \frac{2-p}{p n}\Sigma_{\epsilon}
\\&-\left({q^n\mathbb{E}\left[{\hat{\eta}_{\epsilon,n}(\theta)|R_{\epsilon, n}(\theta) = 0}\right]+ (1-q^n)\eta_{\epsilon}}\right)
\\&\cdot\left({q^n\mathbb{E}\left[{\hat{\eta}_{\epsilon,n}(\theta)|R_{\epsilon, n}(\theta) = 0}\right] + (1-q^n)\eta_{\epsilon}}\right)^\top
\\&=\frac{2-p}{p n}\Sigma_{\epsilon} + (1-q^n)\eta_{\epsilon}\eta_{\epsilon}^\top - (1-q^n)^2\eta_{\epsilon}\eta_{\epsilon}^\top
\\&
= \frac{2-p}{p n}\Sigma_{\epsilon} + (1-q^n)(1  - (1 - q^n))\eta_{\epsilon}\eta_{\epsilon}^\top
\\&
= \frac{2-p}{p n}\Sigma_{\epsilon} + (1-q^n)q^n\eta_{\epsilon}\eta_{\epsilon}^\top
\\&\preceq\frac{2-p}{p n}\Sigma_{\epsilon} + q^n\eta_{\epsilon}\eta_{\epsilon}^\top,
\end{aligned}
\end{equation}
where we used (i), the fact that $1-q^n\leq1$ and the fact that $\Exp{\heta|\Rn = 0} = 0$, which follows by~\eqref{eq:estimator}.
\section{Useful Facts and Derivations}\label{app:useful}
In this appendix, we provide, for completeness, some useful facts and explicit derivations of properties that we use in the paper. In particular, we explicitly derive equation~\eqref{eq:eq8} using the generalize chain rule (GCR).

\subsection{Explicit derivation of equation~\eqref{eq:eq8}}
\label{app:GCRb2}
Please note that equation~\eqref{eq:eq8} follows from taking the derivative of a function of multiple variables and the chain rule. This is also called the \emph{generalized chain rule} in some textbooks \citep{herman2013calculus}(see Theorem 4.10). In the paper, when writing 
\begin{equation}
\frac{\partial}{\partial \theta}\ell(\theta, \hat{\tau}(\theta), X, Y),
\end{equation}
we mean the \textit{total} derivative of the function $\theta \mapsto l(\theta, \hat{\tau}(\theta), X, Y)$, evaluated at a dummy $\theta$.
On the other hand, when writing 
\begin{equation}
\frac{\partial \ell}{\partial \theta}(\theta, \hat{\tau}(\theta), x, y),
\end{equation}
we mean the \textit{partial} derivative of $\ell(\theta, q, x, y)$ with respect to $\theta$, evaluated at $(\theta, q, x, y) = (\theta, \hat{\tau}(\theta), X, Y)$. The difference is that, in the partial derivative, $\hat{\tau}(\theta)$ is treated as a constant, whereas for the total derivative we do not treat $\hat{\tau}(\theta)$ as a constant.
Now, the generalized chain rule (in vector form) can be written as follows: let $u(\theta)\in\mathbb{R}^n$ and $v(\theta)\in\mathbb{R}^m$ be two differentiable functions of $\theta$, and $f(u,v)$ a differentiable function of two vector variables $u$ and $v$. Then
\begin{equation}\label{eq:fuv}
\frac{\partial }{\partial \theta}f(u(\theta), v(\theta)) = \left(\frac{\partial u}{\partial \theta}(\theta)\right)^\top \frac{\partial f}{\partial u}(u(\theta), v(\theta)) + \left(\frac{\partial v}{\partial \theta}(\theta)\right)^\top \frac{\partial f}{\partial v}(u(\theta), v(\theta)),
\end{equation}
where $\frac{\partial u}{\partial \theta}(\theta)$ is the Jacobian of $u(\theta)$, i.e., the matrix with $\frac{\partial u_i}{\partial \theta_j}(\theta)$ in the $i$-th row and $j$-th column (equivalently, $\frac{\partial v}{\partial \theta}(\theta)$ is the Jacobian of $v(\theta)$). Note that in the case of $\ell(\theta, \hat{\tau}(\theta), x, y)$, $x$ and $y$ do not depend on $\theta$ so we can focus on $\ell$ as a function of the two functions $u(\theta) = \theta$ and $v(\theta) = \hat{\tau}(\theta)$. Replacing these $u(\theta)$ and $v(\theta)$ in equation~\eqref{eq:fuv}, and replacing $f(u(\theta), v(\theta))$ with $\ell(\theta, \hat{\tau}(\theta), x, y)$ we see that then
\begin{equation}
\frac{\partial}{\partial \theta}\ell(\theta, \hat{\tau}(\theta), x, y) = \frac{\partial \ell}{\partial \theta}(\theta, \hat{\tau}(\theta), x, y) + \frac{\partial \ell}{\partial \hat{\tau}}(\theta, \hat{\tau}(\theta), x, y)\frac{\partial \hat{\tau}}{\partial \theta}(\theta),
\end{equation}
which is precisely equation \eqref{eq:eq8} in the main paper, where we used the fact that $\left(\frac{\partial \theta}{\partial \theta}\right) = I_d$, where $I_d$ is a $d\times d$ identity matrix, with $d$ the dimension of $\theta$.
\newline
Given that usually in textbooks the generalized chain rule (GCR) is only shown for scalar multi-variable functions, we now report the derivation of equation \eqref{eq:eq8} using the scalar GCR as reported and proved in the statement of Theorem 4.10 in \citep{herman2013calculus}. Hence, we will now provide the derivation of \eqref{eq:eq8} at a more granular level. Consider a differentiable function $\ell$ of $k$ variables, $\ell:\mathbb{R}^k\rightarrow \mathbb{R}$. Now let $f_1, ..., f_k$ be differentiable functions, with $f_i:\mathbb{R}^d\rightarrow\mathbb{R}$, for $i = 1, ..., k$ and some $d\geq1$. Then, denoting a vector $[t_1, ..., t_d] \in \mathbb{R}^d$ and $w = \ell(f_1(t_1, ..., t_d), ..., f_k(t_1, ..., t_d))$ we have (GCR):
\begin{equation}\label{eq:genChainRule}
\frac{\partial w}{\partial t_j} = \sum_{i =1}^k \frac{\partial w}{\partial f_i}\frac{\partial f_i}{\partial t_j}.
\end{equation}
Now note that in the case of our paper, we have $w = \ell(\theta, \hat{\tau}(\theta), x, y)$. Note that $x$ and $y$ have no dependency on parameters in $\theta$ and hence their derivatives will be zero. We can then focus on $\theta$ and $\hat{\tau}(\theta)$. For convenience, note that we can write $\theta = \left[\theta_1, ..., \theta_d\right]$. Now note that the gradient of $w$ is
\begin{equation}
\frac{\partial}{\partial \theta}\left[w\right] = \left[\frac{\partial w}{\partial \theta_1}, ..., \frac{\partial w}{\partial \theta_d}\right]^\top.
\end{equation}
Now note that, for some $j \in \{1, ...,  d\}$, using the chain rule~\eqref{eq:genChainRule} above, 
\begin{equation}
\begin{aligned}
\frac{\partial w}{\partial \theta_j} &= \sum_{i =1}^d \frac{\partial w}{\partial \theta_i}\frac{\partial \theta_i}{\partial \theta_j} + \frac{\partial \ell}{\partial \hat{\tau}}(\theta, \hat{\tau}(\theta), x, y)\frac{\partial \hat{\tau}}{\partial \theta_j}(\theta)
\\&+ \frac{\partial w}{\partial x}\frac{\partial x}{\partial \theta_j}+ \frac{\partial w}{\partial y}\frac{\partial y}{\partial \theta_j}
\\&=
\frac{\partial \ell}{\partial \theta_j}(\theta, \hat{\tau}(\theta), x, y) + \frac{\partial \ell}{\partial \hat{\tau}}(\theta, \hat{\tau}(\theta), x, y)\frac{\partial \hat{\tau}}{\partial \theta_j}(\theta),  
\end{aligned}
\end{equation}
where we used the fact that $\frac{\partial \theta_i}{\partial \theta_j} = 0$ if $i\neq j$ and $\frac{\partial \theta_i}{\partial \theta_i} = 1$. We also explicitly used the fact that $\frac{\partial x}{\partial \theta_j} = 0$ and $\frac{\partial y}{\partial \theta_j} = 0$ because the samples do not depend on the parameter $\theta$. Stacking together $\frac{\partial w}{\partial \theta_j}$ we can see that we obtain precisely equation \eqref{eq:eq8} of the paper:
\begin{equation}
\begin{aligned}
\frac{\partial}{\partial \theta}\left[w\right] &= \frac{\partial}{\partial \theta}\left[\ell(\theta, \hat{\tau}(\theta), X, Y)\right]
\\& = \frac{\partial \ell}{\partial \theta}(\theta, \hat{\tau}(\theta), X, Y) + \frac{\partial \ell}{\partial \hat{\tau}}(\theta, \hat{\tau}(\theta), X, Y)\frac{\partial \hat{\tau}}{\partial \theta}(\theta).
\end{aligned}
\end{equation}
\newpage
\section{Additional experiments}
\label{appendix:m_tuning_experiments}
Here, we provide additional experimental results to complement the findings in the main paper.
\subsection{GMM}
As a warm-up, we validate the results of Theorem~\ref{th:VR_theorem} on a synthetic Gaussian Mixture Model (GMM) dataset. To tune the $\varepsilon$-estimator in~\eqref{eq:eta_hat_thresh}, we employ the $m$-ranking method described in section~\ref{sec:estimation}.
%
%
%
%
%
%
The results, as shown in Figure~\ref{fig:GMM}, illustrate that our estimator (\texttt{VR-ConfTr}) reduces variance effectively, while the naive one (\texttt{ConfTr}) is sample inefficient. 

%
%

\begin{figure}[!h]
    \centering    
\includegraphics[width=0.49\linewidth]{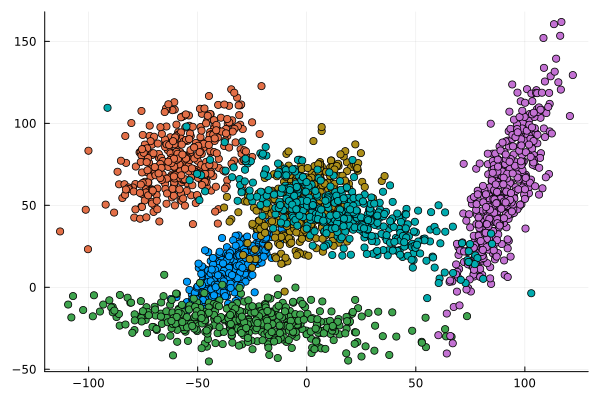}
\includegraphics[width=0.49\linewidth]{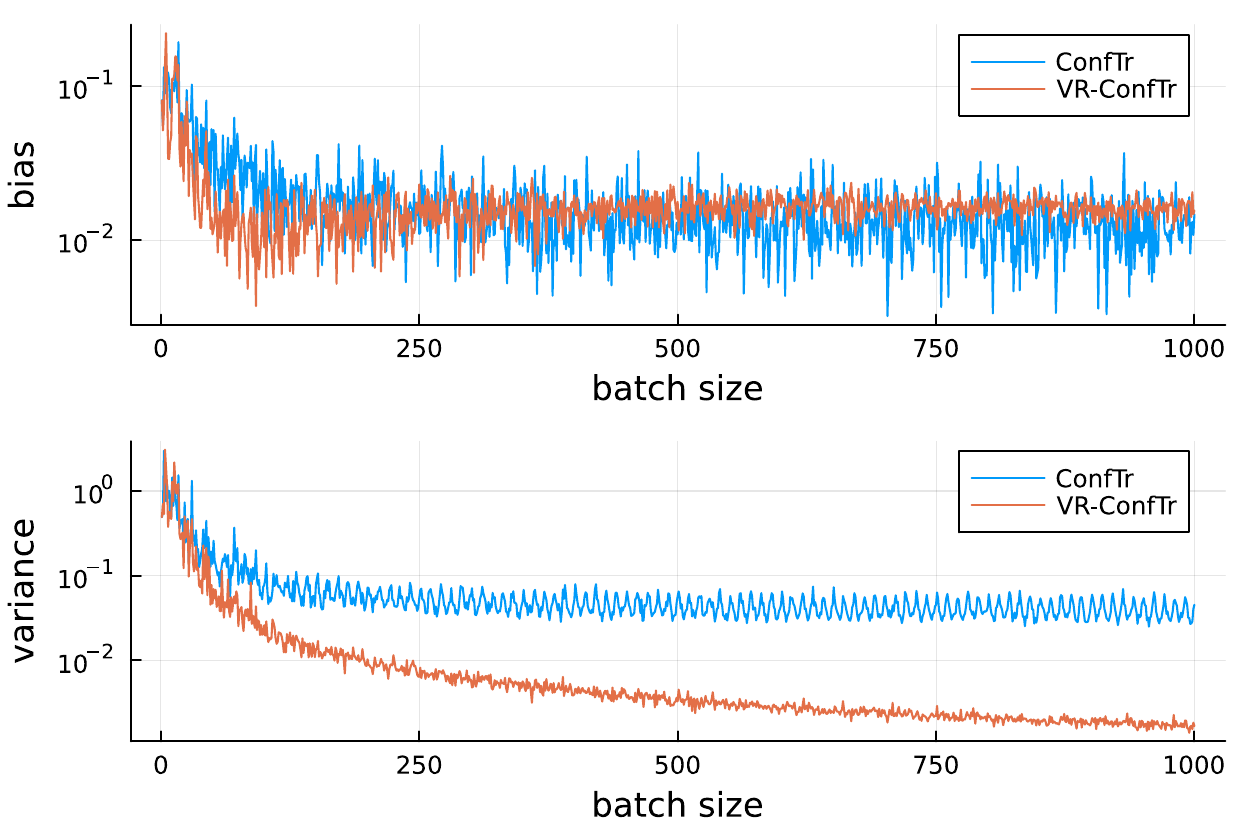}
    \caption{Sample batch from the GMM distribution (left); bias and variance for the quantile gradient estimates, comparing \texttt{ConfTr} and \texttt{VR-ConfTr} on the GMM dataset (right).}
    \label{fig:GMM}
\end{figure}

\subsection{Additional Training Curves}
We first present additional training curves, specifically the test loss and accuracy per epoch for MNIST, Fashion-MNIST, KMNIST, and OrganAMNIST. These plots highlights the performance throughout the training process, providing further insights into convergence behavior and generalization performance. It can be seen that the test loss exhibits a pattern similar to the training loss in \ref{fig:train_trajectories}. In terms of accuracy, \texttt{VR-ConfTr} achieves higher accuracy than \texttt{ConfTr}.
\begin{figure}[!htp]
\centering
    \includegraphics[width=0.24\linewidth]{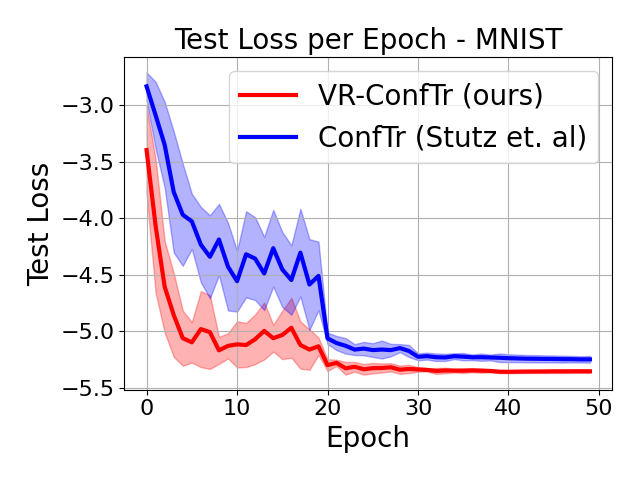}
    \includegraphics[width=0.24\linewidth]{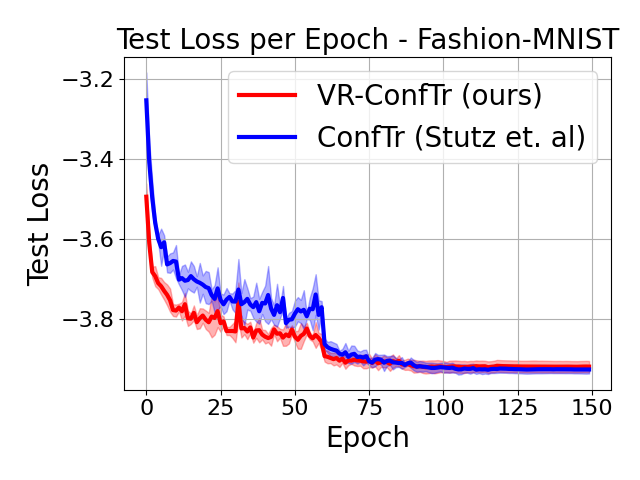}
    \includegraphics[width=0.24\linewidth]{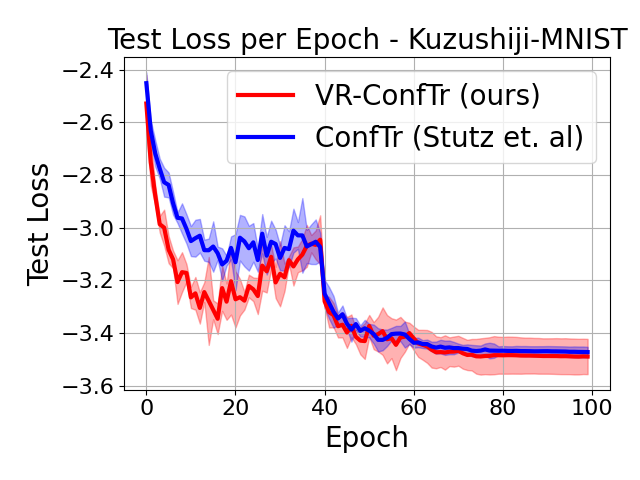}
    \includegraphics[width=0.24\linewidth]{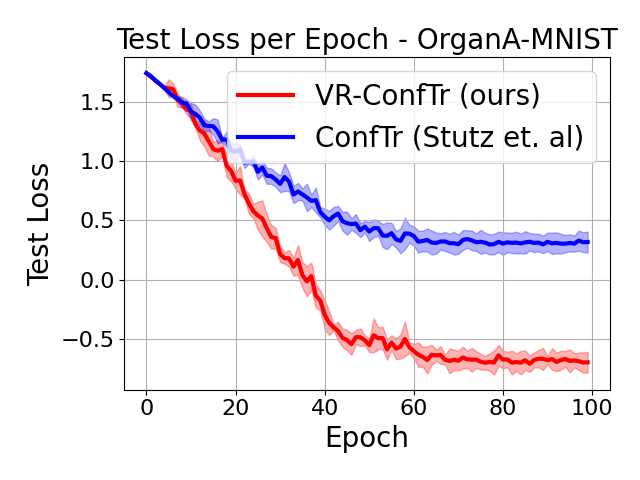}
    \includegraphics[width=0.24\linewidth]{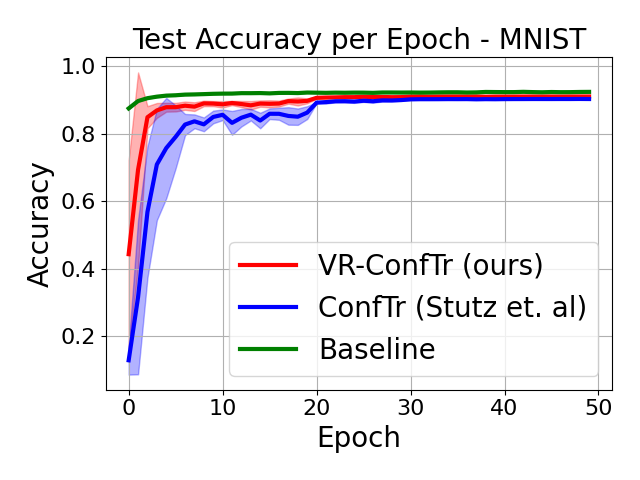}
    \includegraphics[width=0.24\linewidth]{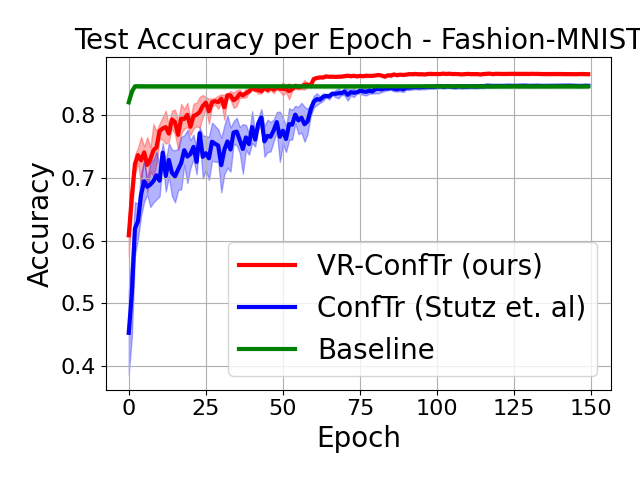}
    \includegraphics[width=0.24\linewidth]{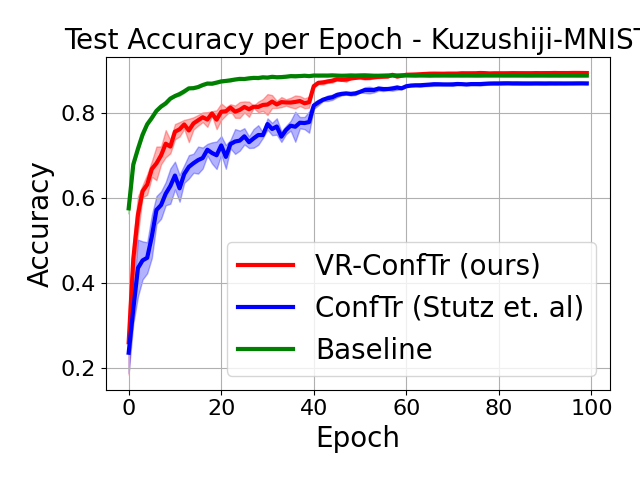}
    \includegraphics[width=0.24\linewidth]{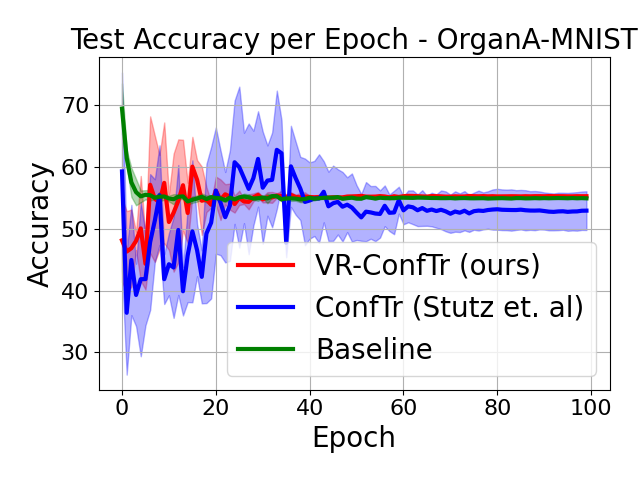}
    \caption{Learning curves for MNIST, Fashion-MNIST, Kuzushiji-MNIST, and OrganAMNIST. For each dataset, we show the test loss on the first row and tets accuracy on the bottom row at the end of each epoch.}
    \label{fig:additional_trajectories}
\end{figure}

\subsubsection{Variance of the gradients over the course of training}
\label{appendix:gradient-analysis}
In this section, we present visualization for the variance of the estimated quantile gradients during training for our proposed method \texttt{Vr-ConfTr}, compared to \texttt{ConfTr} in figure~\ref{fig:bias-variance-grad-fmnist}. We conduct this experiment on the MNIST dataset, using the $m$-ranking estimator with \texttt{Vr-ConfTr}, and evaluate performance across different batch sizes. This analysis aims to empirically substantiate our claim that \texttt{Vr-ConfTr} reduces variance of the estimated quantile gradients over the epochs, leading to more stable gradient updates and improved final performance. Furthermore, we demonstrate that with an appropriate choice of the hyperparameter $m$ for the $m$-ranking estimator, \texttt{Vr-ConfTr} not only reduces variance but also shows improvements in terms of the bias of the estimated quantile gradients during training. In order to compute the variance and bias for the estimated quantile gradient $\widehat{\frac{\partial \tau}{\partial \theta}}$, we estimate the population quantile $\tau(\theta)$ and its gradient $\frac{\partial \tau}{\partial \theta}$ at each model update utilizing the full training, calibration, and test datasets.

\begin{figure}[!htp]
    \centering
    \includegraphics[width=0.49\linewidth]{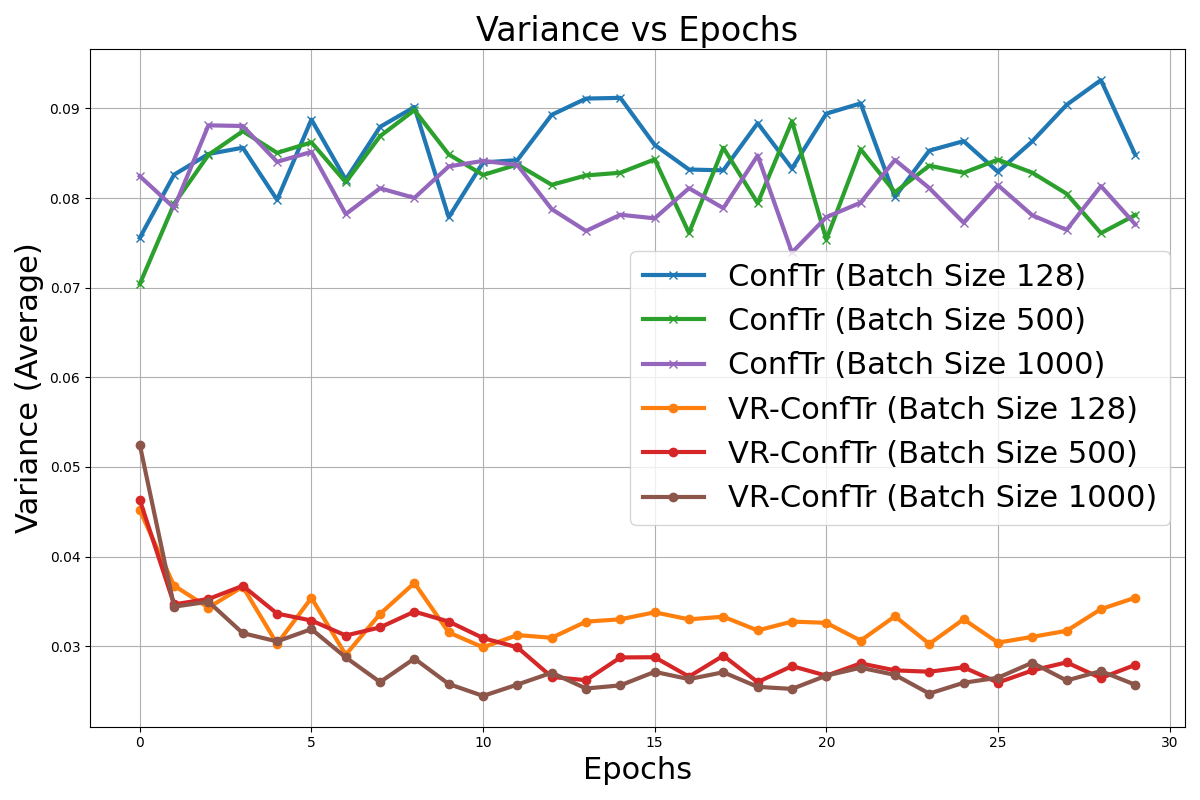}
    \includegraphics[width=0.49\linewidth]{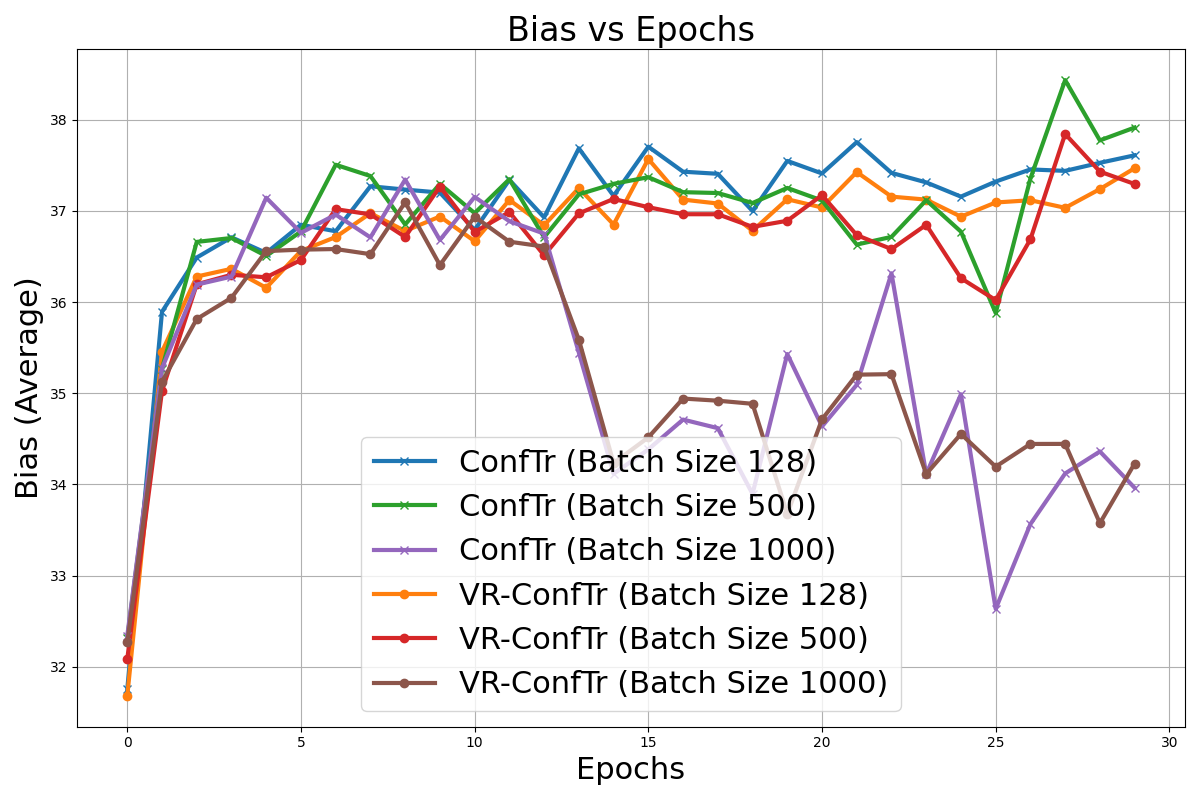}
    \caption{Variance and bias of the estimated quantile gradients during training for \texttt{ConfTr} and \texttt{Vr-ConfTr}, evaluated on the MNIST dataset across different batch sizes. The left figure shows the variance of the gradients over epochs. The right panel illustrates the bias of the estimated gradients, demonstrating that \texttt{Vr-ConfTr} maintains low bias while effectively reducing variance.}
    \label{fig:bias-variance-grad-fmnist}
\end{figure}

\subsection{Ablation study for $m$ and $\varepsilon$}
\label{appendix:ablation}
\subsubsection{$\varepsilon$-threshold tuning ablation study}
\label{appendix:epsilon-ablation-gmm}
This study evaluates the bias and variance of  the $\widehat{\frac{\partial{{\tau}}}{{\partial \theta}}}$ using the $\varepsilon$-threshold estimator and tuning the threshold $\varepsilon$, with \texttt{VR-ConfTr} for the GMM dataset depicted in figure~\ref{fig:GMM}. Figure~\ref{fig:eps-gmm-ablation} shows how varying $\varepsilon$ impacts the estimator's performance, highlighting the trade-offs between bias and variance of $\widehat{\frac{\partial{{\tau}}}{{\partial \theta}}}$ as  $\varepsilon$ changes.

\begin{figure}[!htp]
    \centering
\includegraphics[width=0.49\linewidth]{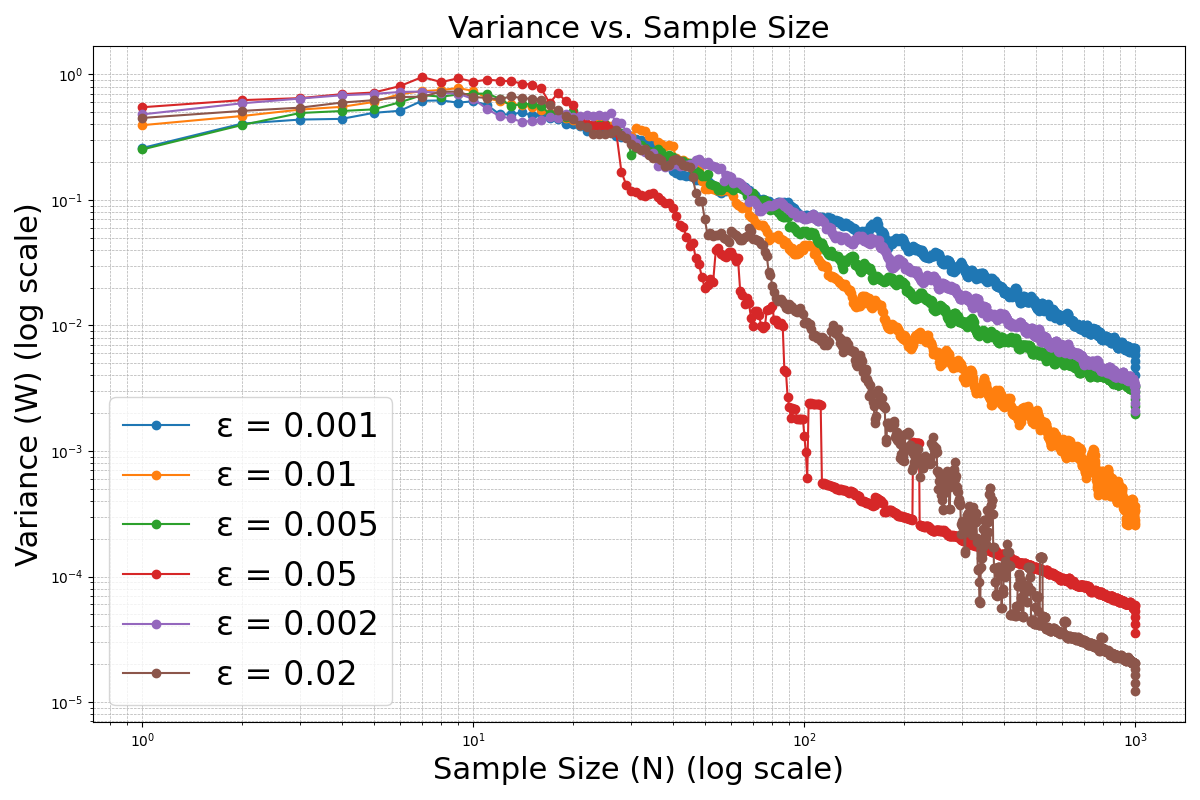}
\includegraphics[width=0.49\linewidth]{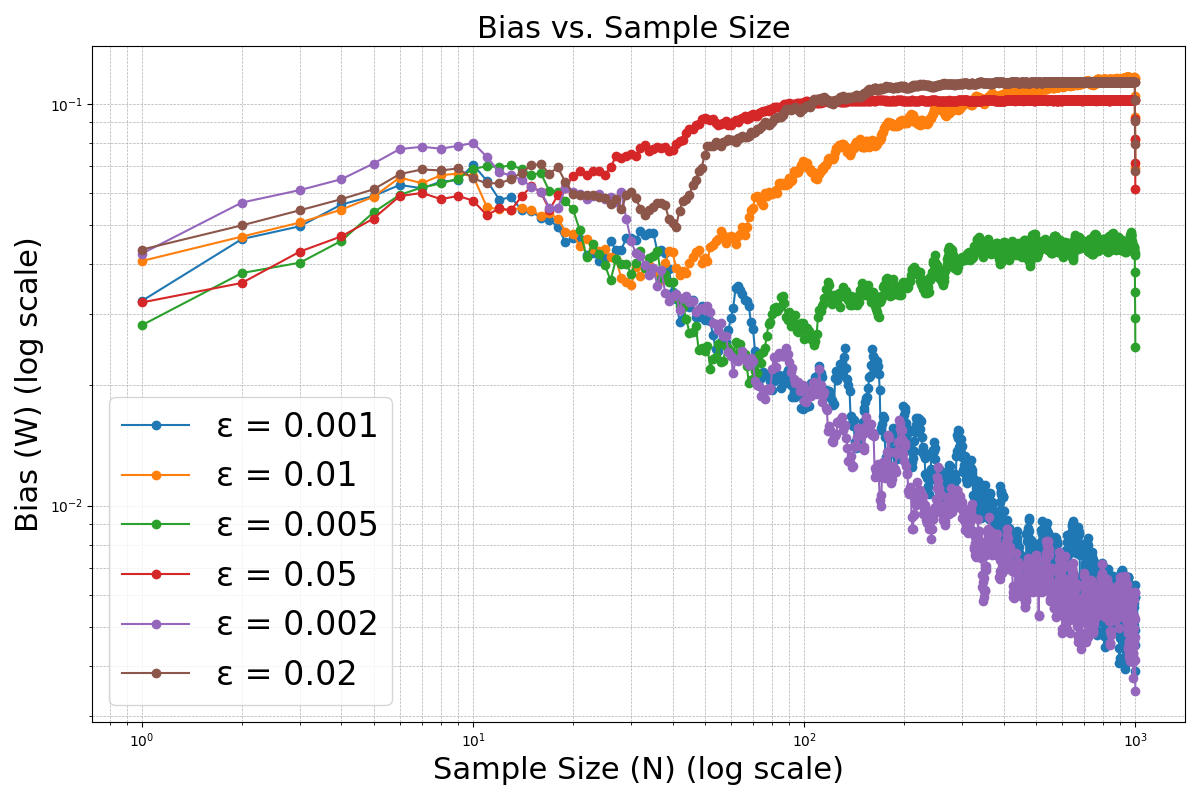}
    \caption{Bias and variance for the quantile gradient estimates tuning the $\varepsilon$ threshold in \texttt{VR-ConfTr} on the GMM dataset. The left panel shows the variance, and the right panel shows the bias for different $\varepsilon$ values.}
    \label{fig:eps-gmm-ablation}
\end{figure}

\subsubsection{$\varepsilon$-threshold tuning with $m$-ranking ablation study}
We evaluate the bias and variance of $\widehat{\frac{\partial{{\tau}}}{{\partial \theta}}}$ this time using the $m$-ranking strategy to finetune $\varepsilon$ with \texttt{VR-ConfTr} for the GMM dataset. 
Figure~\ref{fig:m-gmm-ablation} shows how varying $m$ impacts the estimator's performance, highlighting the trade-offs between bias and variance of $\widehat{\frac{\partial{{\tau}}}{{\partial \theta}}}$ as $m$ changes. Here $m$ explicitly depends on the desired miscoverage rate $\alpha$ and the sample size $n$.
\label{appendix:m-ablation-gmm}
\begin{figure}[!htp]
    \centering
    \includegraphics[width=0.49\linewidth]{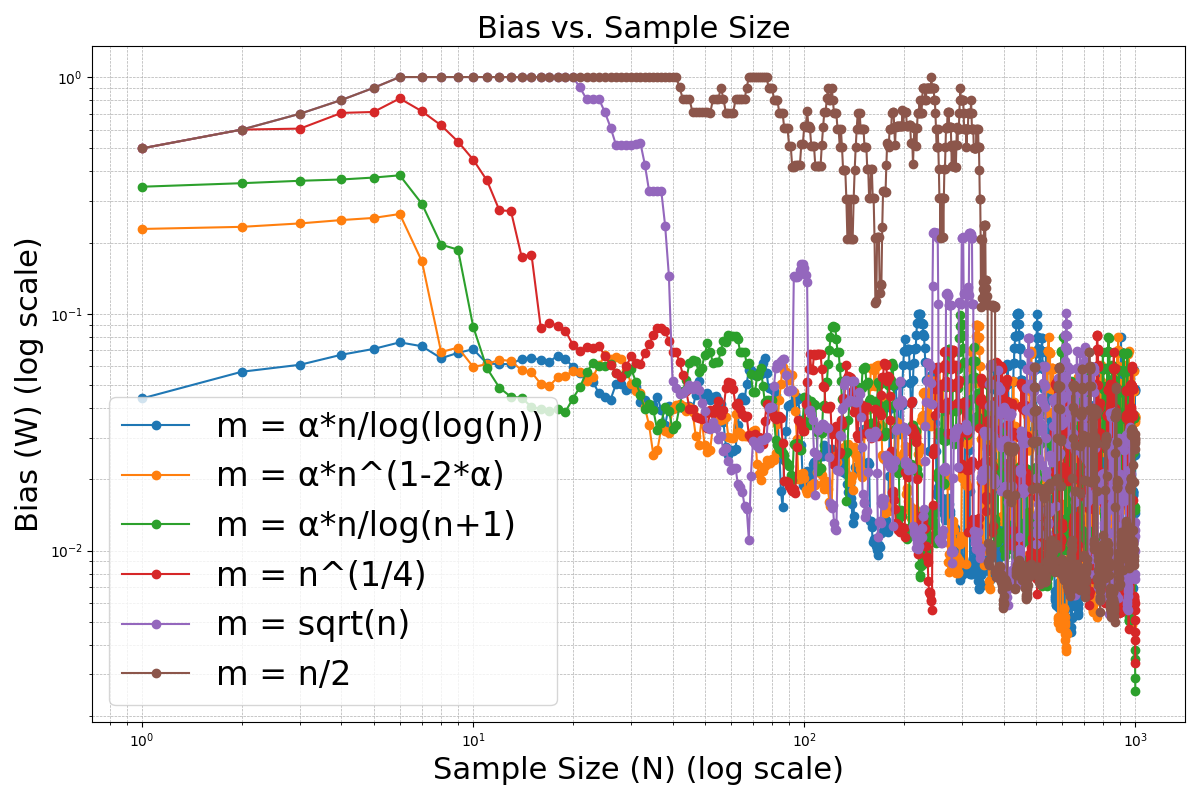}
    \includegraphics[width=0.49\linewidth]{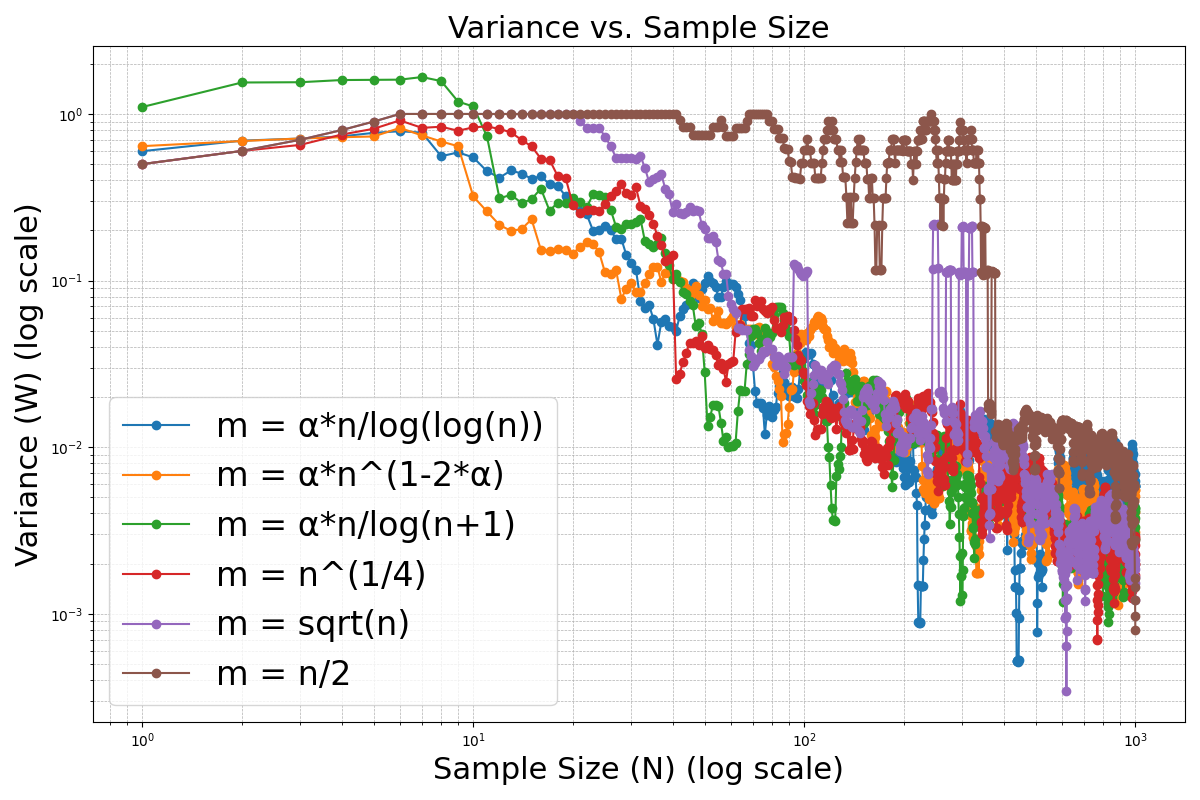}
    \caption{Bias and variance for the quantile gradient estimate using the $m$-ranking to adaptively tune $\varepsilon$ with \texttt{VR-ConfTr} on the GMM dataset. The left panel shows the bias, and the right panel shows the variance for different $m$ values}
    \label{fig:m-gmm-ablation}
\end{figure}
\subsubsection{On tuning the $\epsilon$-threshold with $m$-ranking}
\label{appendix:epsilon-tuning-fmnist}
%
In practice, using the $\varepsilon$-estimator, when training the models, we noticed that a ``good" value of $\varepsilon$ \textit{varies significantly} across iterations. Note that a good value of the threshold $\varepsilon$ not only depends on the specific batch $B_{\text{cal}}$ at a given iteration, but also on the model parameters $\theta$ at that iteration. Hence, hyper-parameter tuning with the $\varepsilon$-threshold estimator requires some heuristic to adapt the threshold to specific iterations. In this sense, the $m$-ranking estimator is a natural heuristic for a batch and parameter-dependent choice of the threshold $\varepsilon$. We noticed indeed that performing hyper-parameter tuning of the $m$-ranking estimator we were able to provide a good value of $m$ to be used \textit{across all iterations}, which from the point of view of hyper-parameter tuning is a great advantage. 
\\
To empirically illustrate this connection and validate the importance of dynamically tuning the $\varepsilon$-threshold estimator, figure~\ref{fig:optimal-eps-tuning-fmnist} presents the adaptive choice of the $\varepsilon$-threshold estimator on the Fashion-MNIST dataset when using $m$-ranking tuning strategy with $m=6$. 
\begin{figure}[!htp]
    \centering
    \includegraphics[width=0.5\linewidth]{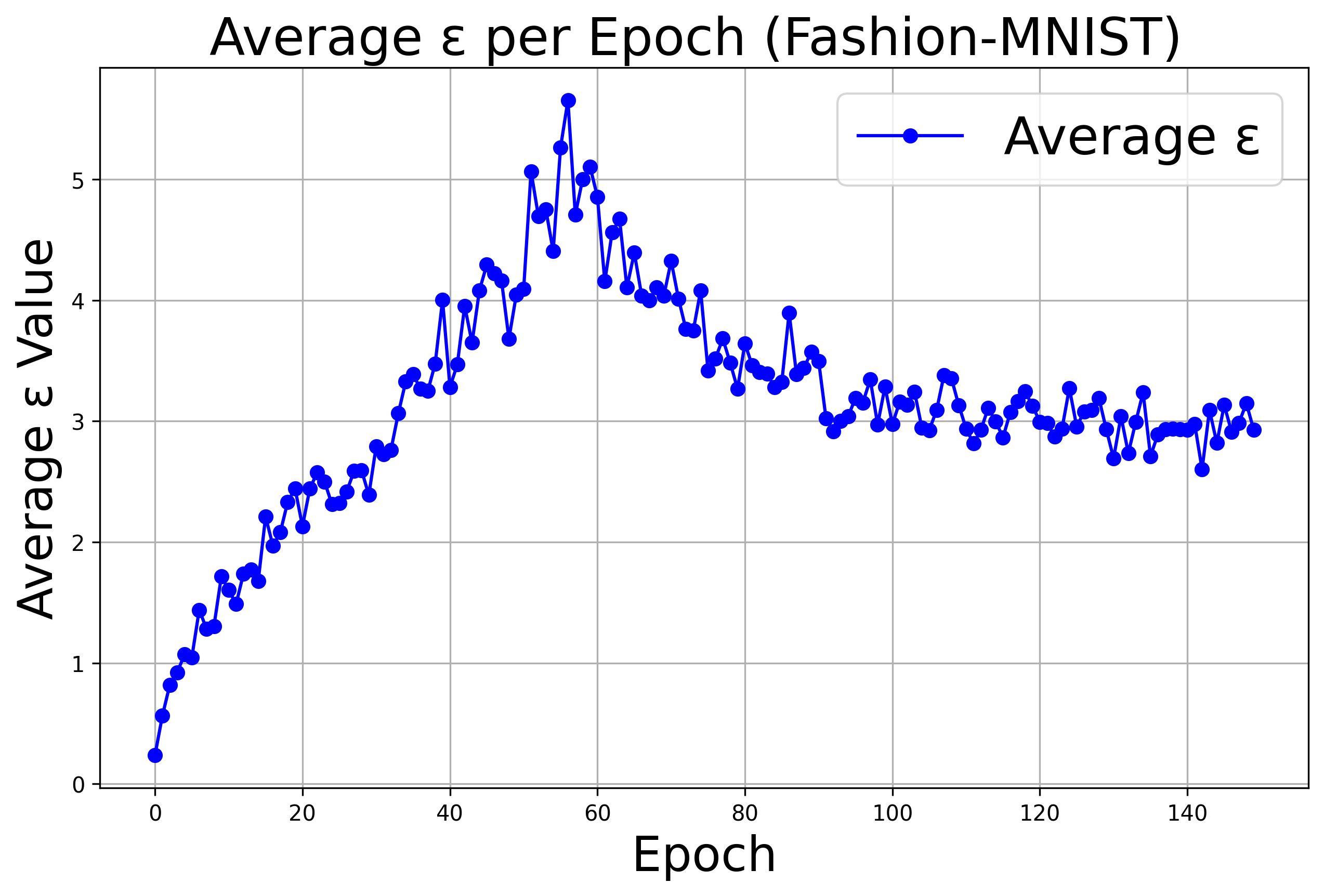}
    \caption{Adaptive tuning of the $\varepsilon$-threshold estimator on Fashion-MNIST using the $m$-ranking strategy. The variability in $\varepsilon$ underscores the necessity of dynamic adjustment in threshold-based approaches.}
    \label{fig:optimal-eps-tuning-fmnist}
\end{figure}
\subsection{Class-conditional coverage and set size}
\label{appendix:class-conditional}
We evaluated the trained models in terms of class-conditional coverage and set size, using the same CP-procedure applied post-training with the standard \texttt{THR} method and $\alpha = 0.01$. Figure~\ref{fig:combined_results} displays the class-conditional coverage and set sizes for each dataset. The results show the effectiveness of \texttt{Vr-ConfTr} in achieving reliable class-conditional coverage while outperforming \texttt{ConfTr} in terms of producing smaller class-conditional prediction set sizes. The results are taken as the average over all the training and testing trials.

\begin{figure*}[!htp]
    \centering
    
    \begin{minipage}{0.48\textwidth}
        \centering
        \includegraphics[width=\linewidth]{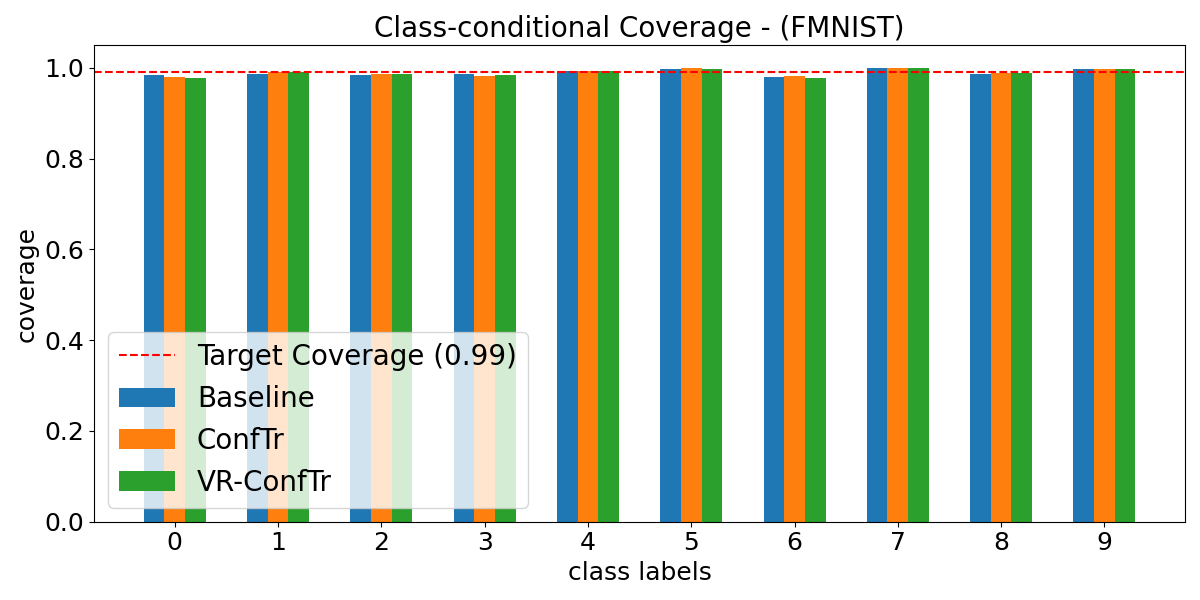}
        \caption{Class-Conditional Coverage (Fashion-MNIST)}
        \label{fig:fmnist_cov}
    \end{minipage}
    \hfill
    \begin{minipage}{0.48\textwidth}
        \centering
        \includegraphics[width=\linewidth]{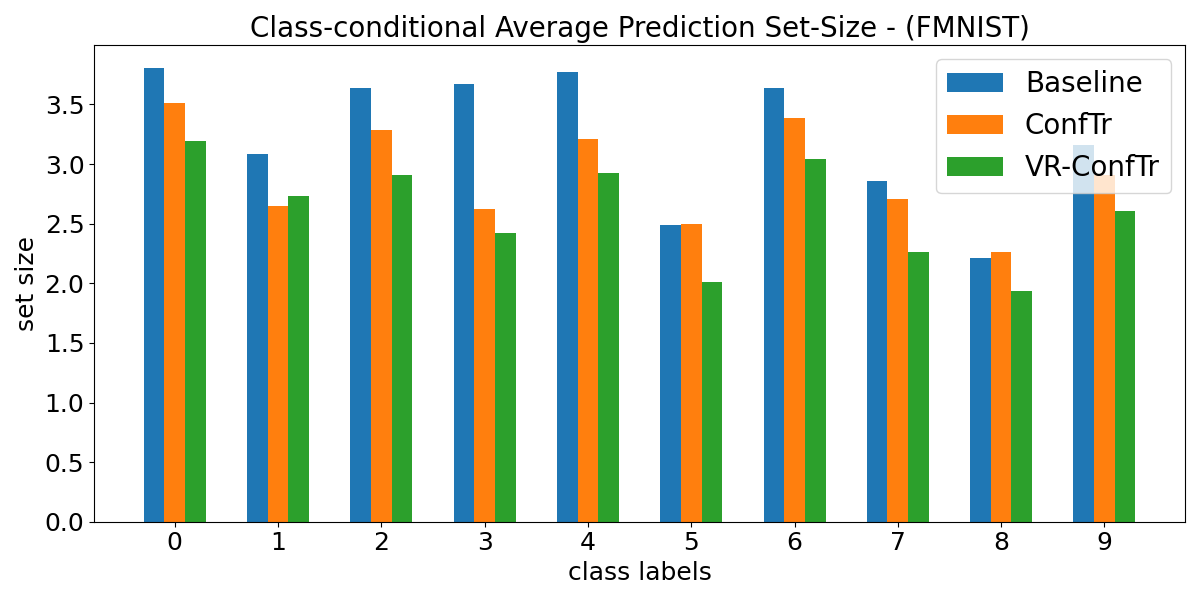}
        \caption{Class-Conditional Set Sizes (Fashion-MNIST)}
        \label{fig:fmnist_setsize}
    \end{minipage}
    
    \vspace{1em} 
    
    \begin{minipage}{0.48\textwidth}
        \centering
        \includegraphics[width=\linewidth]{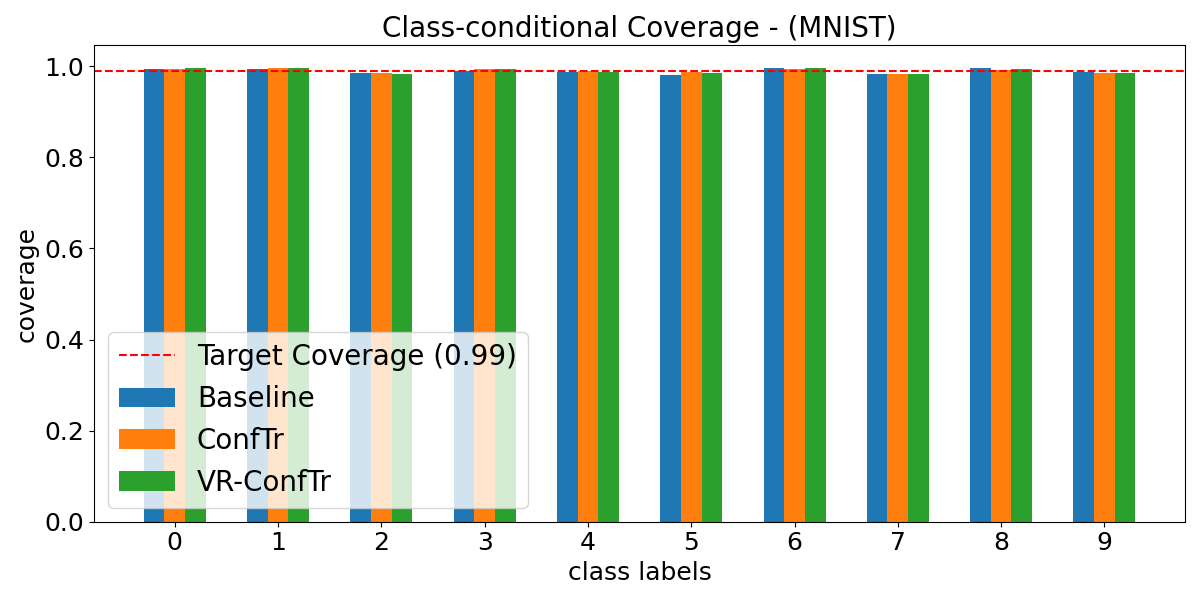}
        \caption{Class-Conditional Coverage (MNIST)}
        \label{fig:mnist_cov}
    \end{minipage}
    \hfill
    \begin{minipage}{0.48\textwidth}
        \centering
        \includegraphics[width=\linewidth]{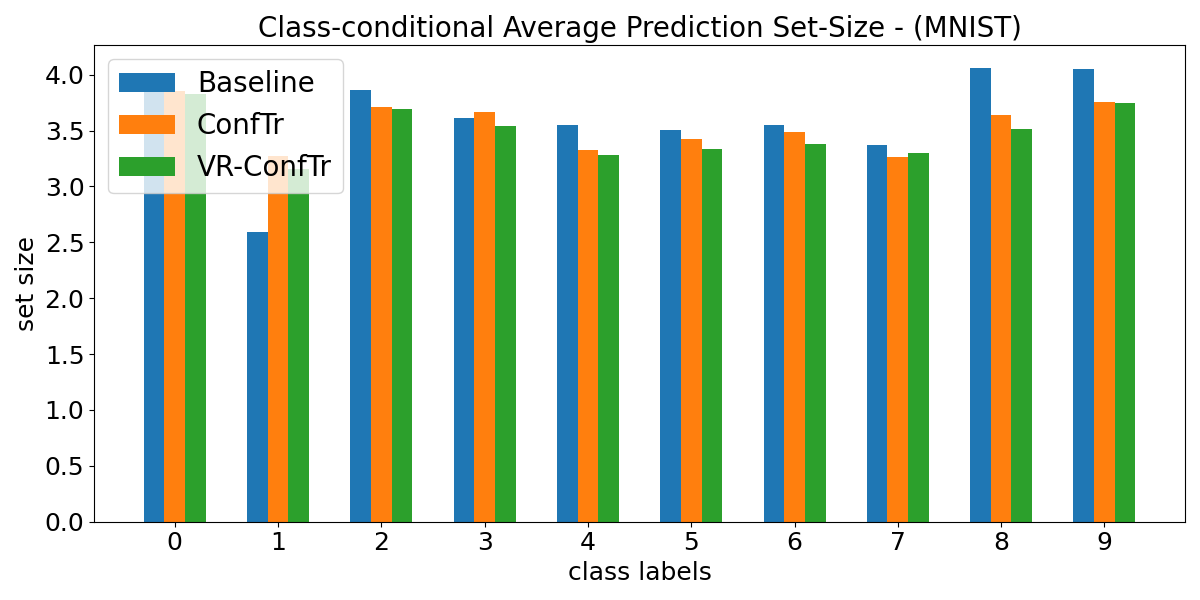}
        \caption{Class-Conditional Set Sizes (MNIST)}
        \label{fig:mnist_setsize}
    \end{minipage}
    
    \vspace{1em} 
    
    \begin{minipage}{0.48\textwidth}
        \centering
        \includegraphics[width=\linewidth]{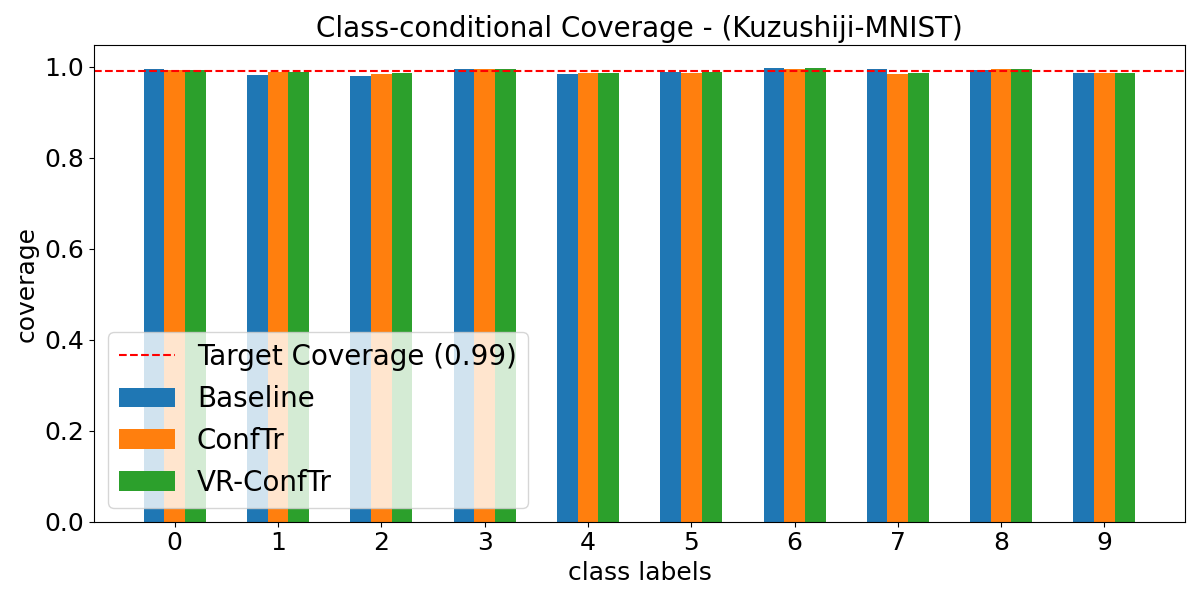}
        \caption{Class-Conditional Coverage (Kuzushiji-MNIST)}
        \label{fig:kmnist_cov}
    \end{minipage}
    \hfill
    \begin{minipage}{0.48\textwidth}
        \centering
        \includegraphics[width=\linewidth]{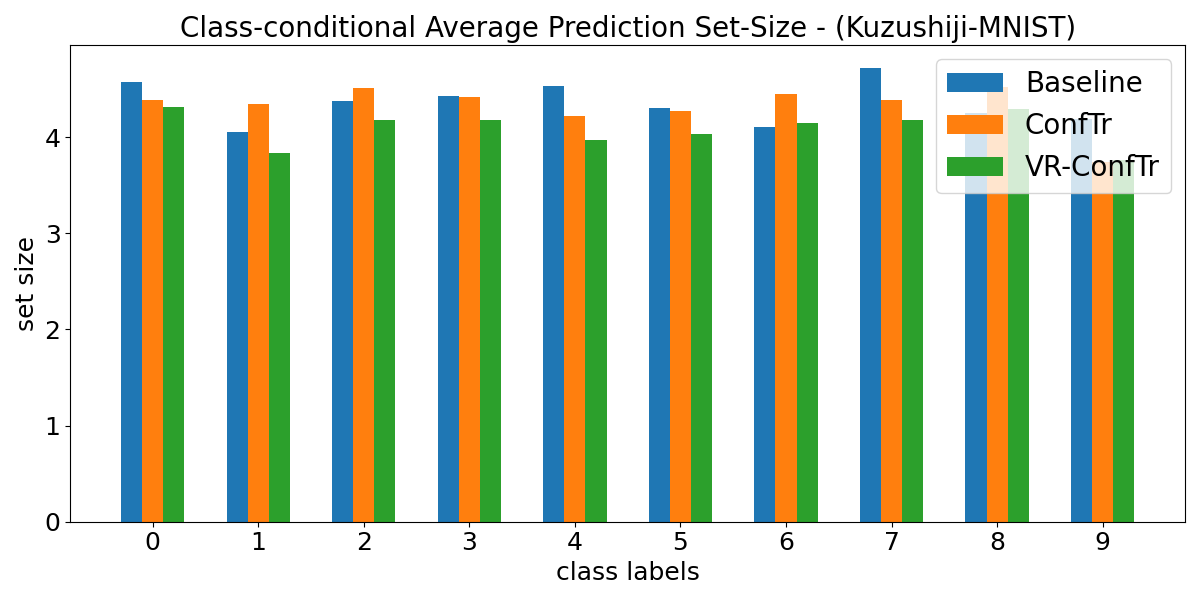}
        \caption{Class-Conditional Set Sizes (Kuzushiji-MNIST)}
        \label{fig:kmnist_setsize}
    \end{minipage}
    
    \vspace{1em} 
    
    \begin{minipage}{0.48\textwidth}
        \centering
        \includegraphics[width=\linewidth]{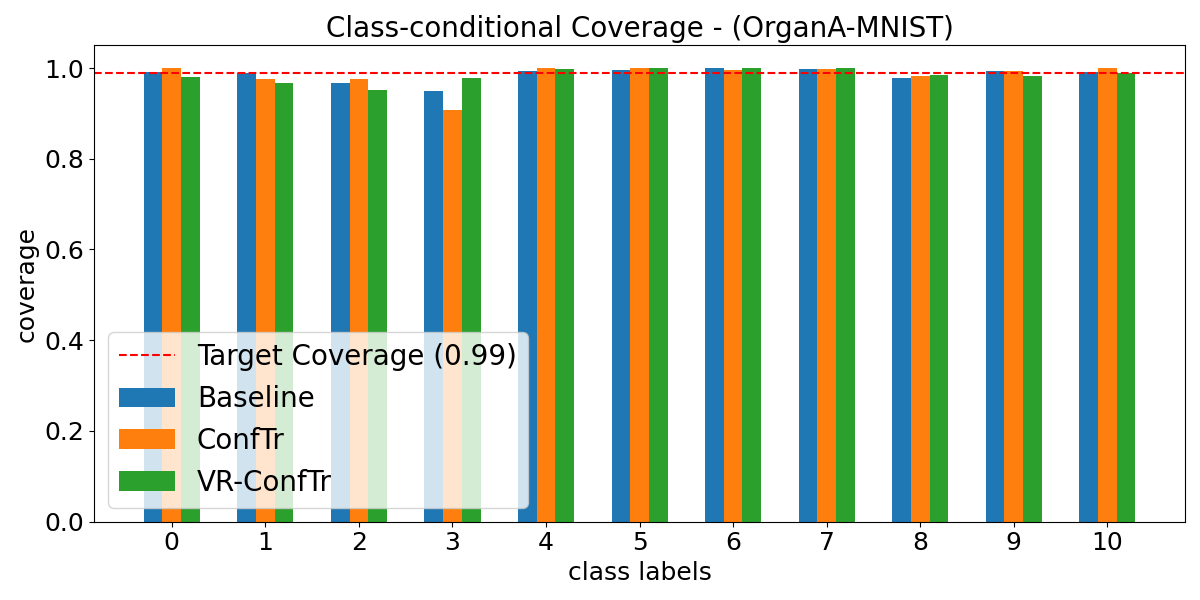}
        \caption{Class-Conditional Coverage (OrganA-MNIST)}
        \label{fig:organamnist_cov}
    \end{minipage}
    \hfill
    \begin{minipage}{0.48\textwidth}
        \centering
        \includegraphics[width=\linewidth]{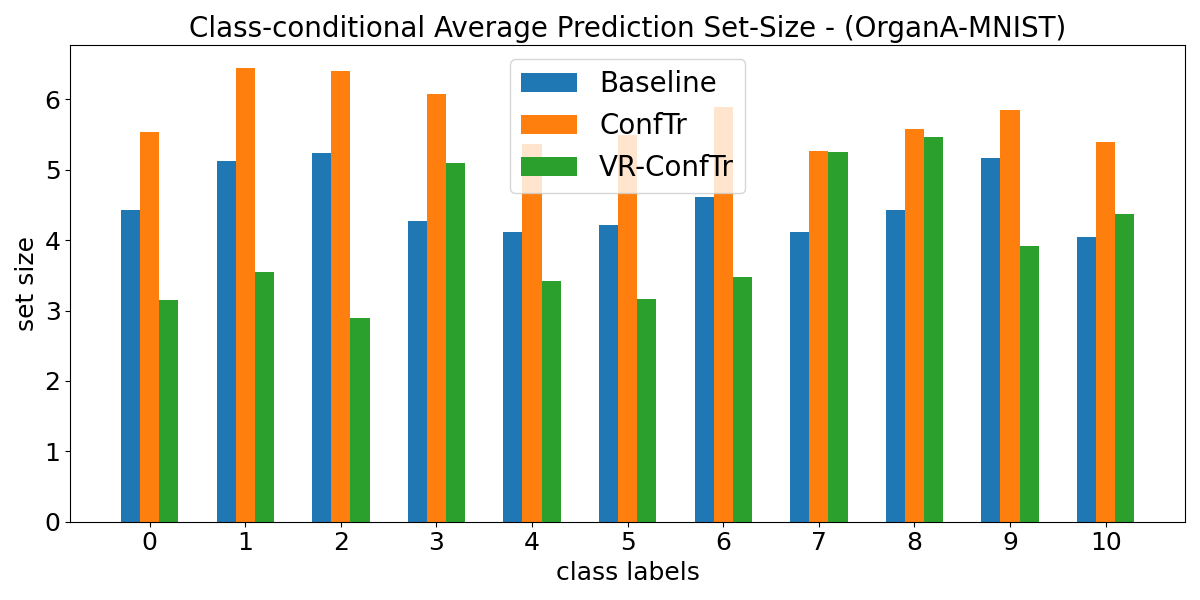}
        \caption{Class-Conditional Set Sizes (OrganA-MNIST)}
        \label{fig:organamnist_setsize}
    \end{minipage}
    
    \caption{Class-conditional coverage rates and average prediction set sizes for each dataset, reported over 10 randomized test trials. For each dataset, the left plot shows the class-conditional coverage rates with the target coverage level of $1 - \alpha = 0.99$ indicated by the horizontal red dashed line. The right plot shows the class-conditional average prediction set sizes.}
    \label{fig:combined_results}
\end{figure*}

\subsection{Tuning \texttt{VR-ConfTr}: Number of Points for Gradient Estimation (\texttt{m})}
In \texttt{VR-ConfTr}, the number of points ($m$) used in the $m$-ranking strategy plays a crucial role in the bias-variance trade-off. Consistent with the theory, increasing $m$ (which translates to increasing the threshold $\varepsilon$) reduces the variance but potentially increases the bias of the gradient estimate. We conduct a grid search over the values $[4,6,8,10,16,20]$ for $m$ and report the results of tuning $m$ for MNIST and Fashion MNIST. We select the value of $m$ that experimentally provides the best trade-off between bias and variance.
\textbf{MNIST Results:} We show in Fig~\ref{fig:mnist_m_tuning}, plots corresponding to the loss on the training loss, test loss, and the test accuracy per epoch. The results illustrate a consistent reduction in the variance of the gradient estimates as $m$ increases. However, once $m$ deviates from its best value, the bias of the gradient estimates increases, which results in higher values of the training loss as well as increase in the size of the prediction sets.

\begin{figure}[htp]
    \centering
    \includegraphics[width=0.33\linewidth]{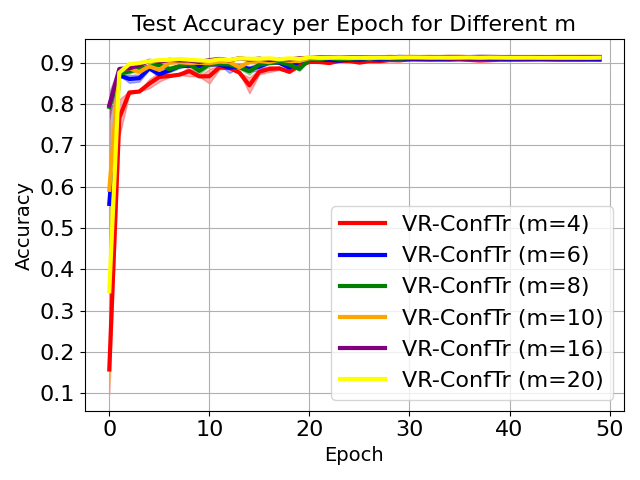}
    \includegraphics[width=0.33\linewidth]{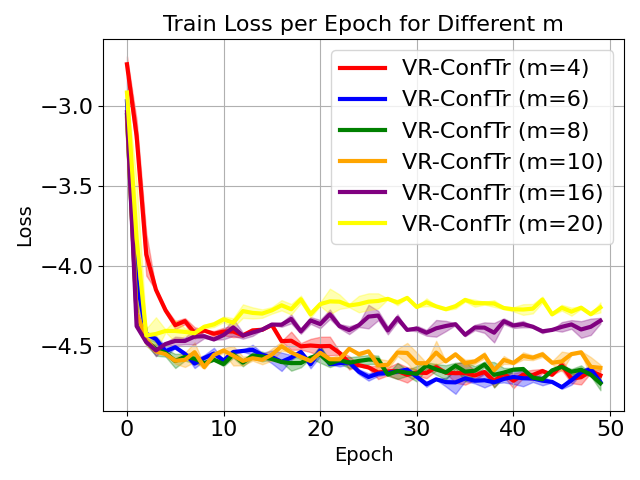}
    \includegraphics[width=0.33\linewidth]{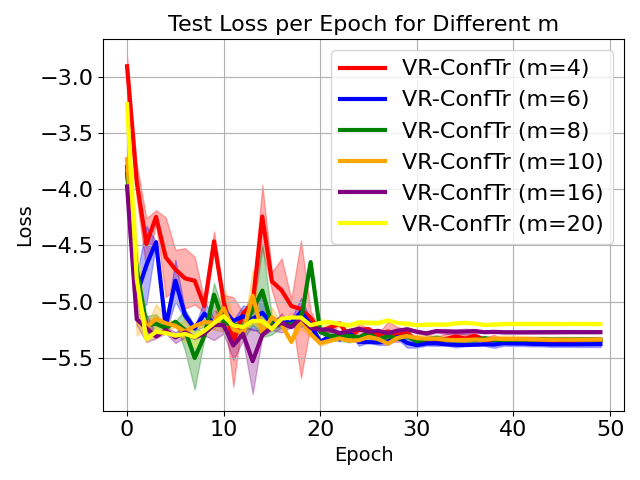}
    \includegraphics[width=0.33\linewidth]{figures/tuning_mnist_test_accuracy.png}
    \includegraphics[width=0.33\linewidth]{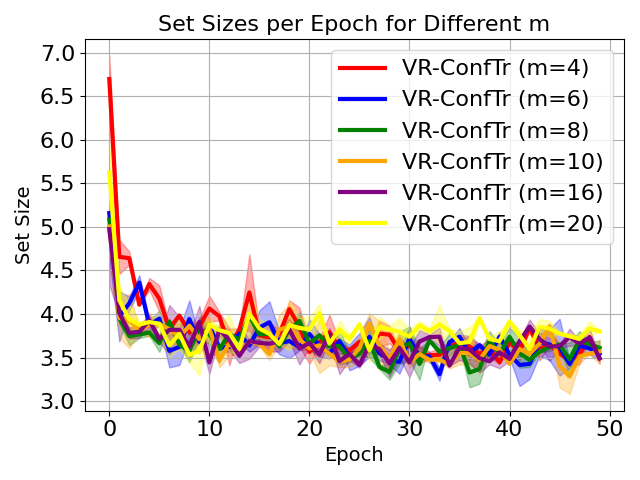}
    \caption{Training curves for different values of $m$ on MNIST}
    \label{fig:mnist_m_tuning}
\end{figure}
\textbf{Fashion-MNIST:} Similarly, tuning $m$ on Fashion-MNIST shows that a value of $m=6$ provides the best results, as depicted in Fig~\ref{fig:fmnist_m_tuning}
\begin{figure}[htp]
    \centering
    \includegraphics[width=0.33\linewidth]{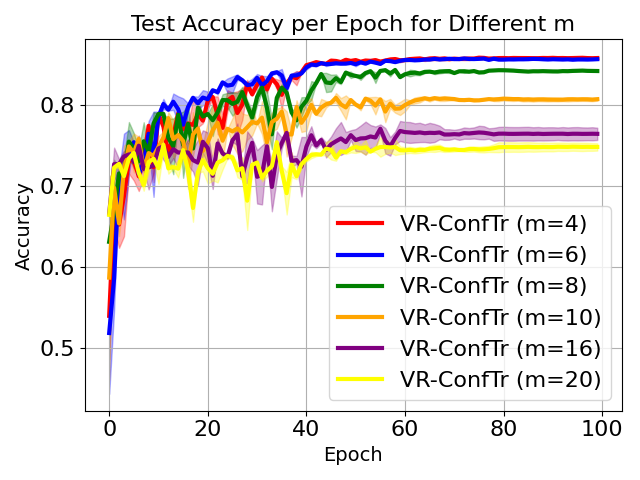}
    \includegraphics[width=0.33\linewidth]{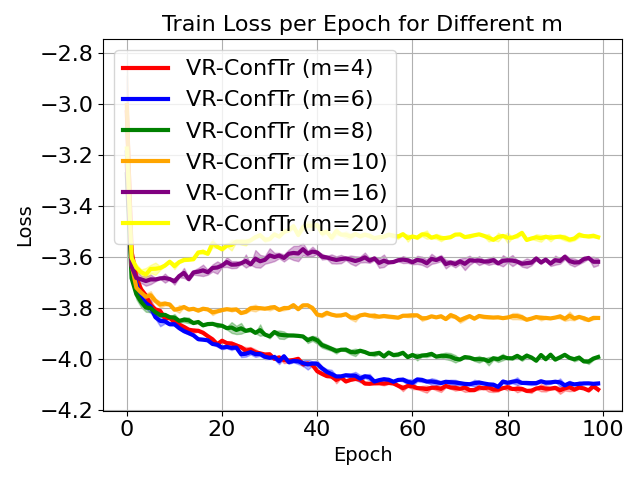}
    \includegraphics[width=0.33\linewidth]{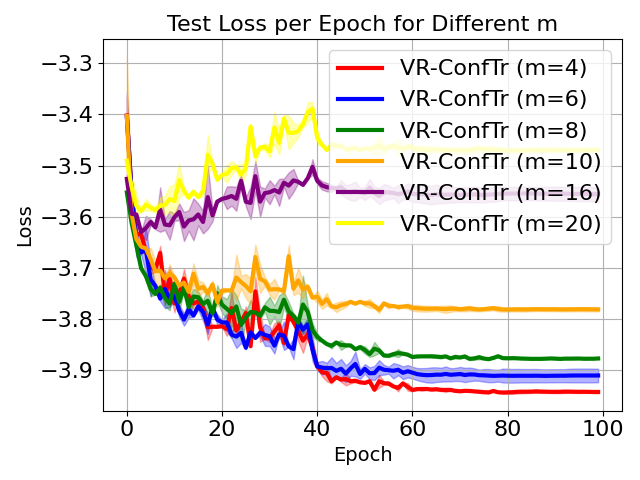}
    \includegraphics[width=0.33\linewidth]{figures/tuning_fmnist_test_accuracy.png}
    \includegraphics[width=0.33\linewidth]{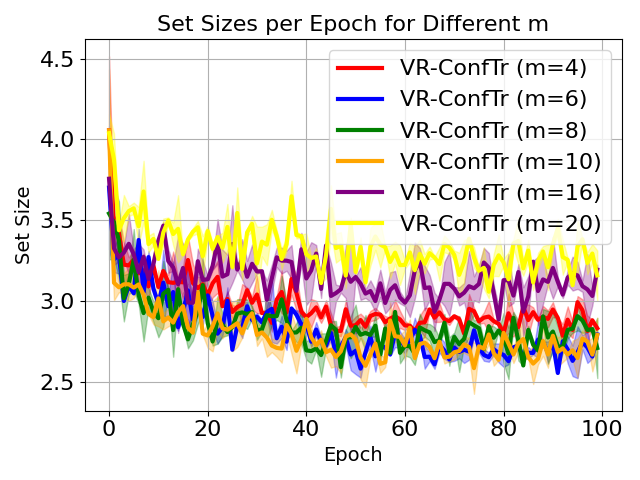}
    \caption{Learning curves for different values of $m$ on Fashion-MNIST}
    \label{fig:fmnist_m_tuning}
\end{figure}
\subsection{Alternative Architecture}
In this section, we compare the performance of \texttt{VR-ConfTr} on Kushuniji-MNIST using a simpler linear model architecture, different than the MLP used in the main paper. The results further reinforce that regardless of the model architecture, the trends observed in terms of convergence speed and prediction set efficiency remain consistent across datasets and architectures. Table \ref{tab:kmnist_eval_linear} shows the average accuracy and set sizes for the two different model architectures trained on K-MNIST.
\begin{table*}[h]
\centering
\begin{tabular}{@{}llccc@{}}
\toprule
\textbf{Dataset} & \textbf{Model Name} & \textbf{Accuracy (Avg ± Std)} & \textbf{Set Size (Avg ± Std)} \\ 
\midrule
\multirow{3}{*}{K-MNIST (Linear)} & Baseline & $0.695 \pm 0.007$ & $6.799 \pm 0.117$ \\ 
& ConfTr   & $0.582 \pm 0.047$ & $6.646 \pm 0.226$ \\ 
& VR-ConfTr  & $0.612 \pm 0.033$ & $6.488 \pm 0.148$ \\ 
\cmidrule{2-4}
\multirow{3}{*}{K-MNIST (MLP)} & Baseline & $0.872 \pm 0.046$ & $4.982 \pm 0.530$ \\ 
& ConfTr   & $0.783 \pm 0.125$ & $4.762 \pm 0.226$ \\ 
& VR-ConfTr  & $0.835 \pm 0.098$ & $4.657 \pm 0.680$ \\ 
\bottomrule
\end{tabular}
\caption{Evaluation results of the KMNIST dataset trained with different model architectures. Columns present average accuracy and set size with their standard deviations (\textbf{Avg ± Std}).}
\label{tab:kmnist_eval_linear}
\end{table*}

\newpage
\section{Experimental Details}\label{appendix:experiments}
In this section we describe the experimental setup, including model architectures, dataset configurations, training protocol, testing procedure, and the corresponding hyper-parameters. The focus of the experiments is on evaluating the CP set sizes during training, convergence speed, and accuracy while ensuring a fair comparison between \texttt{ConfTr} and our proposed \texttt{VR-ConfTr}. 
\subsection{Dataset Configurations}
We consider the benchmark datasets MNIST \cite{lecun1998gradient}, Fashion-MNIST \cite{xiao2017fashion}, Kuzushiji-MNIST \cite{clanuwat2018deep} and OrganAMNIST \cite{medmnistv2}  and CIFAR-10 \cite{krizhevsky2009learning}. MNIST is a dataset of handwritten digits with 10 classes, and Fashion-MNIST consists of 10 fashion product categories. Kuzushiji-MNIST extends the MNIST paradigm by incorporating 10 classes of cursive Japanese characters. OrganAMNIST, derived from medical images, contains 11 classes of abdominal organ slices.  CIFAR-10 is a dataset of natural images with 10 object categories. The training, calibration, and testing splits for each dataset are summarized in Table \ref{tab:dataset_config}. MNIST and Fashion-MNIST are provided by the torchvision library, while Kuzushiji-MNIST and OrganAMNIST, and CIFAR-10 are available from their respective repositories. For MNIST, Fashion-MNIST, Kuzushiji-MNIST, and CIFAR-10 10\% of the training set is reserved as calibration data. For OrganAMNIST, the validation set is used as the calibration data. During evaluation, we combine the calibration and test data and perform evaluations over 10 random splits of the combined dataset into calibration/test partitions. Model parameters are learned exclusively on the training data, while calibration and test data are used to evaluate the model as a black-box at the end of each epoch. The transformations applied to the dataset are as follows: for MNIST, Fashion-MNIST, and Kuzushiji-MNIST, images are normalized to have zero mean and unit variance, using a mean of 0.5 and a standard deviation of 0.5. For OrganAMNIST, images undergo random horizontal flips, random rotations of up to 15 degrees, and are normalized similarly. for CIFAR-10, we use random resizing, horizontal flips, and normalization as data augmentations. 
\begin{table*}[!h]
\centering
\begin{tabularx}{\textwidth}{@{}lXXXXXX@{}}
\toprule
\textbf{Dataset} & \textbf{Classes} & \textbf{Image Size} & \textbf{Training Set} & \textbf{Calibration Set} & \textbf{Test Set}  \\ \midrule
MNIST            & 10               & $28 \times 28$      & 55,000                & 5,000                    & 10,000                 \\
Fashion-MNIST    & 10               & $28 \times 28$      & 55,000                & 5,000                    & 10,000                 \\
OrganMNIST       & 11               & $28 \times 28$      & 34,561                & 6,491                    & 17,778                 \\
Kuzushiji-MNIST  & 10               & $28 \times 28$      & 55,000                & 5,000                    & 10,000                 \\ 
CIFAR-10         & 10               & $32 \times 32$      & 45,000                & 5,000                    & 10,000                 \\ \bottomrule
\end{tabularx}
\caption{Dataset Splits}
\label{tab:dataset_config}
\end{table*}
\subsection{Model Architectures}
In our experiments, we implemented all models using JAX \cite{jax2018github}. We utilize a range of architectures including linear models, multi-layer perceptrons (MLPs), and modified ResNet architectures tailored for specific datasets. For the \textbf{MNIST} dataset, we employ a simple linear model, which consists of a single dense layer. The input images, reshaped from $28 \times 28$ into a flattened vector of size $784$, are passed through a fully connected layer mapping the inputs directly to the $10$ output classes. For \textbf{Fashion-MNIST}, we use a multi-layer perceptron (MLP), with two hidden layers. We use 64 units per hidden layer, with ReLU activations \cite{nair2010relu} , followed by a dense layer for the $10$ output classes. For \textbf{Kuzushiji-MNIST}, we utilize a similar MLP architecture. The model contains two hidden layers with 256 and 128 units, respectively. The input data is flattened and passed through these fully connected layers with ReLU activations. For \textbf{OrganAMNIST}, we used a residual network, inspired by the ResNet architecture from \cite{he2016resnet}
, with modifications. The model consists of an initial convolutional layer followed by four stages of residual blocks, each with two layers. Each residual block uses $ 3 \times 3$ convolutions with ReLU activations. The number of output channels doubles after each state $(64, 128, 256, 512).$ Global average pooling is applied before the final fully connected layer, which maps the pooled feature representations to the 11 output classes. For \textbf{CIFAR-10}, we use a ResNet20 architecture, a lightweight version of ResNet~\cite{he2016resnet} with 20 layers. Particularly, we use a pretrained ResNet20 model trained with cross-entropy loss. The last linear layer of the model is reinitialized and then fine-tuned using the \texttt{ConfTr} and \texttt{VR-ConfTr} algorithms. We do not attempt to optimize the model architectures in order to solve the datasets with high accuracy. Instead, we focus on the conformal prediction results, and ensure that the architecture used across different algorithms are identical for a fair comparison.

\subsection{Training Details}
Similar to \cite{stutz2022learning}, we trained all models using Stochastic Gradient Descent (SGD) with Nesterov momentum \cite{sutskever2013momentum}. The learning rate follows a multi-step schedule where the initial learning rate was decreased by a factor of 0.1 after 2/5, 3/5, and 4/5 of the total number of epochs. The models were trained using cross-entropy-loss for \texttt{Baseline} training, and for \texttt{ConfTr} and \texttt{VR-ConfTr} based on the size-loss as described by \cite{stutz2022learning}. During training, for MNIST, FMNIST, KMNIST, and OrganAMNIST we set the conformal prediction threshold parameter $\alpha = 0.01$. For finetuning CIFAR-10, we use $\alpha = 0.1$, and a weight decay term of 0.0005 for the optimizer. To ensure statistical robustness, we conducted multiple randomized training trials for each dataset, using a different random seed for each trial. Specifically, we performed 10 training trials for MNIST and 5 training trials each for FMNIST, KMNIST, OrganAMNIST, and CIFAR-10. The corresponding learning curves, i.e the training loss, testing loss, accuracy and CP set sizes evaluated on the test data at the end of every epoch, were averaged over these randomized trials to provide a smooth and general view of the model's performance.
The key hyper-parameters used for training are listed in Table~\ref{tab:hyperparams}. These hyper-parameters include \textbf{size weight} which scales the loss term associated with the size of the CP sets during training, \textbf{alpha $\alpha$} corresponding to the miscoverage rate, \textbf{batch size} for SGD, \textbf{learning rate } for the optimizer, and the number of \textbf{epochs} for which the model is trained for. 
\begin{table}[htbp]
    \centering
    \begin{tabularx}{\textwidth}{lXXXXX}
        \toprule
        \textbf{Hyper-parameter} & \textbf{MNIST} & \textbf{Fashion-MNIST} & \textbf{Kuzushiji-MNIST} & \textbf{OrganA-MNIST} & \textbf{CIFAR-10} \\
        \midrule
        Batch Size               & 500            & 500                    & 500                      & 500                   & 500              \\
        Training Epochs          & 50             & 150                    & 100                      & 100                   & 50               \\
        Learning Rate            & 0.05           & 0.01                   & 0.01                     & 0.01                  & 0.01              \\
        Optimizer                & SGD            & SGD                    & SGD                      & SGD                   & SGD \\
        Temperature              & 0.5            & 0.1                    & 0.1                      & 0.5                   & 1              \\
        Target Set Size          & 1              & 0                      & 1                        & 1                     & 0                \\
        Regularizer Weight       & 0.0005         & 0.0005                 & 0.0005                   & 0.0005                & 0.0005           \\
        Size Weight              & 0.01           & 0.01                   & 0.01                     & 0.1                   & 0.05              \\
        $\alpha$                 & 0.01           & 0.01                   & 0.01                     & 0.01                  & 0.1              \\
        $m$-rank                 & 6              & 6                      & 4                        & 4                     & 6                \\
        \bottomrule
    \end{tabularx}
    \caption{Training and evaluation hyper-parameters for each dataset.}
    \label{tab:hyperparams}
\end{table}

\subsection{Evaluation Details}
The evaluation of our models was conducted in two stages: (1) computing the test accuracy for each model after training, and (2) evaluating the conformal prediction (CP) set sizes and coverage over multiple test and calibration splits.
\textbf{Test Accuracy:} For each dataset, the test accuracy of the trained models was evaluated on the test data, and the results were averaged over the randomized training trials.
\textbf{CP set sizes} We first combine the holdout calibration and test data. We then randomly split this combined data into calibration and test portions, repeating the process 10 times. For each split, we apply the CP \texttt{THR} algorithm with $\alpha$ consistent with the value during training, and compute the CP set sizes on the test portion. The results are averaged across the 10 random splits. The cardinality of each split is consistent with the dataset configurations outlined in Table~\ref{tab:dataset_config}. This procedure is performed for each trained model, and the final reported results are averaged across both the training trials and testing splits.

\subsection{Differences from ConfTr reports}
We report the performance of \texttt{Conftr} with a batch size of 100 for Fashion-MNIST, as originally reported by \cite{stutz2022learning}, selected for optimal performance. While a batch size of 500 yields smaller set sizes, it results in a slight (~1\%) decrease in accuracy. For completeness, we include the results for both configurations. 
\begin{table*}[htp]
    \centering
    \begin{tabular}{lrrcc}
        \toprule
        Model & Batch Size & Accuracy (Avg ± Std) & Set Size (Avg ± Std) \\
        \midrule
        ConfTr & 100 & $0.809 \pm 0.051$ & $3.125 \pm 0.197$ \\
        ConfTr & 500 & $0.799 \pm 0.065$ & $3.048 \pm 0.201$ \\
        VR-ConfTr & 500 & $0.839 \pm 0.043$ & $2.795 \pm 0.154$ \\
        \bottomrule
    \end{tabular}
    \caption{Final evaluation results for Fashion-MNIST, showing average accuracy and set size with their standard deviations (\textbf{Avg ± Std}).}
\end{table*}

\textbf{Retrieving exact reported set sizes as \cite{stutz2022learning}}:
Our experimental results and trends align with those reported in \cite{stutz2022learning}. However, the smaller set sizes for \texttt{Conftr} on MNIST and FMNIST in their paper are likely due to their use of different architectures. Despite this, the overall trends— \texttt{Conftr} outperforming \texttt{Baseline}, and \texttt{VR-Conftr} outperforming \texttt{Conftr}—remain consistent regardless of the model. Our focus is on a fair comparison across algorithms by using the same architecture, rather than reproducing the exact figures or architectures from \cite{stutz2022learning}.

\newpage
\section{On the Computational Complexity of~\texttt{VR-ConfTr}.}
We will now discuss the computational complexity of~\texttt{VR-ConfTr} when compared to~\texttt{ConfTr}. We will argue that the computational complexity of the two algorithms is essentially the same. We start by breaking down the computational cost of \texttt{ConfTr} and then illustrate the difference with \texttt{VR-ConfTr}. 
\newline
\textbf{Per-step computational complexity of \texttt{ConfTr}.} Given a batch and partition $B = \{B_{\text{cal}}, B_{\text{pred}}\}$, with $|B_{\text{cal}}| = |B_{\text{pred}}| = n$, the first step of \texttt{ConfTr} is to compute a sample $\alpha$ quantile $\hat{\tau}(\theta)$ based on the calibration batch $B_{\text{cal}} = \{X_i^{\text{cal}}, Y_i^{\text{cal}}\}_{i=1}^n$, which requires the computation of the calibration batch conformity scores $\{E_{\theta}(X_i^{\text{cal}}, Y_i^{\text{cal}})\}_{i = 1}^n$ and of their $\alpha$-quantile. At this point, the computation of the \texttt{ConfTr} gradient is performed computing the gradient of the loss 
\begin{equation}\label{eq:conftrCost}
\frac{1}{|B_{\text{pred}}|}\sum_{(x, y)\in B_{\text{pred}}}\ell(\theta,\hat{\tau}(\theta), x, y).
\end{equation}
Note that for each sample $(x, y)$, computing the \texttt{ConfTr} gradient implies computing the following (equation~\eqref{eq:eq8} in the main paper):
\begin{equation}\label{eq:conftrgrad}
    \begin{aligned}
        \frac{\partial}{\partial \theta} \left[\ell(\theta,\hat{\tau}(\theta), x, y)\right] 
        = \frac{\partial \ell}{\partial \theta}(\theta, \hat{\tau}(\theta), x, y)
        + \frac{\partial \ell}{\partial\hat{\tau}}(\theta, \hat{\tau}(\theta), x, y) \cdot \frac{\partial\hat{\tau}}{\partial \theta}(\theta)
    \end{aligned}
\end{equation}

Note that computing this gradient requires computing (i) the gradients $\frac{\partial \ell}{\partial \theta}(\theta, \hat{\tau}(\theta), x, y)$ and $\frac{\partial \ell}{\partial\tau}(\theta, \hat{\tau}(\theta), x, y)$ for all samples $(x,y)\in B_{\text{cal}}$, and (ii) the gradient $\frac{\partial\hat{\tau}}{\partial \theta}(\theta)$. The difference in terms of computational complexity between \texttt{ConfTr} and our proposed \texttt{VR-ConfTr} lies in the computation of estimates of $\frac{\partial{\tau}}{\partial \theta}(\theta)$, which in \texttt{ConfTr} is done via computing the gradient of $\hat{\tau}(\theta)$, while in our algorithm is done plugging an improved estimate $\widehat{\frac{\partial{\tau}}{\partial \theta}}(\theta)$. We describe the computational difference between these two approaches in the next paragraph. 
\newline
\newline
\textbf{Per-step computational complexity of \texttt{VR-ConfTr}.} Note that in our proposed algorithm \texttt{VR-ConfTr}, given a batch $B$ defined as above, we consider the same per-step loss function of \texttt{ConfTr} of equation~\eqref{eq:conftrCost}. However, instead of computing directly the gradient of~\eqref{eq:conftrCost}, we compute separately an estimate $\widehat{\frac{\partial \tau}{\partial \theta}}(\theta)$ of $\frac{\partial \tau}{\partial \theta}(\theta)$ using our novel estimation technique and then plug this estimate in equation~\eqref{eq:conftrgrad} in place of $\frac{\partial\hat{\tau}}{\partial \theta}(\theta)$. In the proposed estimator, computing $\widehat{\frac{\partial \tau}{\partial \theta}}(\theta)$ equals computing gradients $\{\frac{\partial E}{\partial \theta}(\theta, x,y)\}_{(x,y)\in \Bar{B}}$, where $\Bar{B}$ is the set containing the $m$ samples whose conformity scores fall within $\epsilon$ distance from the sample quantile $\hat{\tau}(\theta)$.
Note that, computationally, our algorithm requires computing $\frac{\partial \ell}{\partial \theta}(\theta, \hat{\tau}(\theta), x, y)$ and $\frac{\partial \ell}{\partial\tau}(\theta, \hat{\tau}(\theta), x, y)$, which is the same as \texttt{ConfTr}, while we do not need to compute the gradient $\frac{\partial \hat{\tau}}{\partial \theta}(\theta)$. Instead, we replace the computation of the gradient of $\hat{\tau}(\theta)$ with the computation of an average of $m$ gradients of conformity scores. 
%
In conclusion, the main computational difference between \texttt{ConfTr} and \texttt{VR-ConfTr} is in the computation of the estimate of $\frac{\partial \tau}{\partial \theta}(\theta)$, which for both of the techniques boils down to computing and averaging a certain set of conformity scores. This is why we can safely conclude that the computational complexity of the two algorithms is essentially the same.

\end{document}